\PassOptionsToPackage{table}{xcolor}
\documentclass{article}

\usepackage{iclr2027_conference,times}

\usepackage[utf8]{inputenc}
\usepackage[T1]{fontenc}
\usepackage[hypertexnames=false]{hyperref}
\usepackage{url}
\usepackage{booktabs}
\usepackage{amsmath,amssymb,amsthm,mathtools}
\usepackage{microtype}
\usepackage{xcolor}
\usepackage{graphicx}
\usepackage{tabularx}
\usepackage{array}
\usepackage{wrapfig}
\usepackage{needspace}
\usepackage{placeins}
\usepackage{flafter}
\usepackage{float}
\usepackage{algorithm}
\usepackage{algpseudocode}
\usepackage[most]{tcolorbox}
\usepackage{listings}
\definecolor{cardBlueF}{HTML}{6B94BD}\definecolor{cardBlueB}{HTML}{F1F6FA}
\definecolor{cardTealF}{HTML}{4C9A91}\definecolor{cardTealB}{HTML}{EFF8F6}
\definecolor{cardGoldF}{HTML}{C6A34B}\definecolor{cardGoldB}{HTML}{FCF8EA}
\definecolor{cardRoseF}{HTML}{C98286}\definecolor{cardRoseB}{HTML}{FCF1F2}
\definecolor{cardPurpleF}{HTML}{8A80A8}\definecolor{cardPurpleB}{HTML}{F5F2F8}
\definecolor{cardGrayF}{HTML}{A7B2BC}\definecolor{cardGrayB}{HTML}{F5F7F8}
\tcbset{ucoecard/.style={enhanced jigsaw,breakable,boxrule=0.5pt,arc=1.2pt,left=2mm,right=2mm,top=1.6mm,bottom=1.6mm,boxsep=0pt,
  before skip=6pt plus 2pt minus 1pt,after skip=6pt plus 2pt minus 1pt,fonttitle=\footnotesize\bfseries,fontupper=\footnotesize,
  coltitle=black,titlerule=0.4pt,toptitle=0.9mm,bottomtitle=0.9mm,pad at break*=1.5mm,parbox=false,lines before break=4}}
\newtcolorbox{promptcard}[1]{ucoecard,colframe=cardBlueF,colback=cardBlueB,colbacktitle=cardBlueB,title={#1}}
\newtcolorbox{successcard}[1]{ucoecard,colframe=cardTealF,colback=cardTealB,colbacktitle=cardTealB,title={#1}}
\newtcolorbox{diagcard}[1]{ucoecard,colframe=cardGoldF,colback=cardGoldB,colbacktitle=cardGoldB,title={#1}}
\newtcolorbox{failcard}[1]{ucoecard,colframe=cardRoseF,colback=cardRoseB,colbacktitle=cardRoseB,title={#1}}
\newtcolorbox{analysiscard}[1]{ucoecard,colframe=cardPurpleF,colback=cardPurpleB,colbacktitle=cardPurpleB,title={#1}}
\newtcolorbox{notecard}[1]{ucoecard,colframe=cardGrayF,colback=cardGrayB,colbacktitle=cardGrayB,title={#1}}
\lstdefinestyle{ucoetext}{basicstyle=\ttfamily\scriptsize,columns=fullflexible,keepspaces=true,breaklines=true,breakatwhitespace=false,
  breakindent=0pt,postbreak=\mbox{\textcolor{cardGrayF}{$\hookrightarrow$}\space},showstringspaces=false,aboveskip=3pt,belowskip=3pt,
  xleftmargin=0pt,frame=none,upquote=true}
\lstdefinestyle{ucoecode}{style=ucoetext}
\lstdefinestyle{ucoerem}{style=ucoetext,basicstyle=\ttfamily\scriptsize\color{cardRoseF}\bfseries}
\lstdefinestyle{ucoeadd}{style=ucoetext,basicstyle=\ttfamily\scriptsize\color{cardGoldF}\bfseries}

\definecolor{darkblue}{rgb}{0,0,0.5}
\definecolor{ucoerow}{HTML}{E2EFF9}
\hypersetup{colorlinks=true,citecolor=darkblue,linkcolor=darkblue,urlcolor=darkblue}

\newcommand{\cA}{\mathcal A}
\newcommand{\cB}{\mathcal B}
\newcommand{\cH}{\mathcal H}
\newcommand{\cV}{\mathcal V}

\newtheorem{proposition}{Proposition}
\newtheorem{theorem}{Theorem}

\title{Verifier--Harness Co-Evolution for Training-Free Self-Improvement on Open-Ended Tasks}

\newcommand{\arxivAuthorBlock}{%
\begin{minipage}[t]{\dimexpr\textwidth-2\tabcolsep\relax}
\centering\normalsize\bfseries
\makebox[\linewidth][s]{%
Tianyi Xu\textsuperscript{1}\thanks{\raggedright Equal contribution. Correspondence to Tianyi Xu and Huazheng Wang
(\href{mailto:xut2@oregonstate.edu}{xut2@oregonstate.edu},
\href{mailto:huazheng.wang@oregonstate.edu}{huazheng.wang@oregonstate.edu}).}\hfill
Yaolun Zhang\textsuperscript{1,4*}\hfill
Yiran Wu\textsuperscript{2,4*}\hfill
Joe Nguyen\textsuperscript{1}\hfill
Shengyu Dai\textsuperscript{3}}\\[4pt]
\makebox[\linewidth][s]{%
Saba Moeinizade\textsuperscript{3}\hfill
Zhenwen Shao\textsuperscript{3}\hfill
Qingyun Wu\textsuperscript{2,4}\hfill
Huazheng Wang\textsuperscript{1,4}}\par
\vspace{6pt}
{\normalfont\footnotesize
\makebox[\linewidth][s]{%
\textsuperscript{1}Oregon State University\hfill
\textsuperscript{2}Pennsylvania State University\hfill
\textsuperscript{3}Johnson \& Johnson\hfill
\textsuperscript{4}AG2 AI}\par}
\end{minipage}%
}
\author{\arxivAuthorBlock}
\hypersetup{%
  pdftitle={Verifier--Harness Co-Evolution for Training-Free Self-Improvement on Open-Ended Tasks},
  pdfauthor={Tianyi Xu, Yaolun Zhang, Yiran Wu, Joe Nguyen, Shengyu Dai, Saba Moeinizade, Zhenwen Shao, Qingyun Wu, Huazheng Wang}%
}

\iclrfinalcopy
\begin{document}
\raggedbottom

\maketitle
\vspace{-10pt}
\fancyhead{}
\renewcommand{\headrulewidth}{0pt}

\begin{abstract}
Training-free self-improvement lets an agent use feedback to revise the response it is about to
return and to update the harness state that shapes how it handles later tasks. On open-ended
tasks that feedback rarely comes from reference answers or executable tests, so the agent has to
rely on a model-based verifier. When the verifier is wrong, the current response can get worse,
and misleading experience can persist across the task stream. A verifier that never
changes also grows stale as the agent's responses improve and task demands shift. We introduce
\emph{Unified Co-Evolution} (UCoE), a training-free method that evolves an executable verifier
program together with the agent's harness state while all model parameters stay fixed. UCoE
synthesizes candidate verifier programs and tests them on response contrasts, pairs of responses
that differ in one declared property and have a known preferred side. It selects the candidate
whose discrimination on its weakest contrast family is highest, a quantity we call \emph{driving
capacity}. The selected verifier diagnoses one weakness in each draft, guides a single repair,
decides which response is served, and determines whether the accepted repair may enter the
harness state. As tasks complete, their diagnostics seed new candidate verifiers, which must beat
the current verifier on the same contrasts before adoption. Across six open-ended
benchmarks, UCoE averages $59.00\%$ preference against matched direct generation with Qwen3-8B
and $57.19\%$ with Qwen3.8-27B, margins of $9.00$ and $7.19$ percentage points. On three
verifiable benchmarks, the same method raises the mean from $36.06\%$ for direct generation to
$38.48\%$.
Code is available at \url{https://github.com/tianyi0216/Unified-Co-Evolution}.
\end{abstract}

\section{Introduction}

Self-improving agents use feedback in two ways. They revise the response they are about to
return \citep{selfrefine}, and they retain experience that shapes how they handle later tasks
\citep{reflexion,expel}. Recent work extends the second use to the prompt-level harness that
surrounds the model \citep{rhi2026}. In mathematics and programming, exact answers or executable
tests can ground this feedback, as in Figure~\ref{fig:motivation}a. Open-ended tasks offer no
such anchor. As Figure~\ref{fig:motivation}b shows, the agent sees only the public requirements
of a request and must judge correctness, coverage, relevance, and style on its own. Because the
same judgment decides which response is served and which experience enters the harness state, a
single mistaken judgment can harm the current task and then persist across the stream.

\begin{figure}[!t]
\centering
\includegraphics[width=1\linewidth]{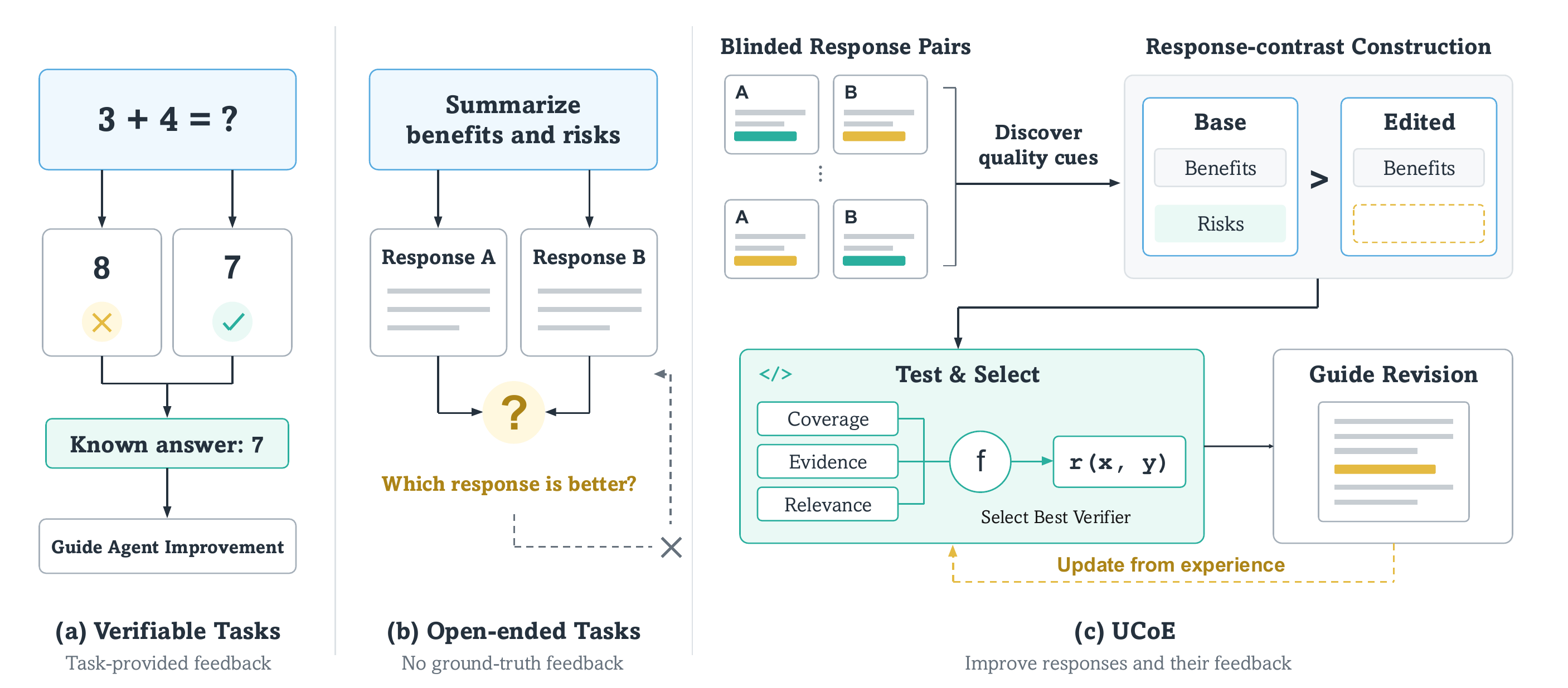}
\vspace{-20pt}
\caption{\textbf{Feedback for open-ended self-improvement.} (a) On verifiable tasks, a known
answer marks each response right or wrong and guides agent improvement. (b) On open-ended tasks
there is no ground truth, so the agent must judge for itself which response is better, and a
wrong judgment harms both the served response and the experience kept for later tasks. (c) UCoE
mines quality cues from blinded response pairs, builds response contrasts in which a base
response is preferred over an edited copy that lacks one property, tests candidate verifier
programs on them, and uses the selected verifier to guide revision while experience from
completed tasks updates it over time.}
\vspace{-16pt}
\label{fig:motivation}
\end{figure}

Existing methods obtain open-ended feedback from holistic LLM judges, natural-language rubrics,
executable reward programs, or learned verifiers. A holistic judge gives a convenient overall
score or preference, but the contribution of each quality dimension stays implicit
\citep{rar,rrd}, and judge accuracy alone is not a reliable predictor of downstream selection
gain \citep{verificationpayoff}. Rubrics make those dimensions explicit, yet generated criteria
can be incomplete, conflated, or redundant \citep{rrd}, so an aggregate score can still hide a
blind spot. Executable reward programs expose the evaluation path \citep{dreamrl,evors}, but
explicit structure does not guarantee behavior, since the selector a program induces still has
to be tested across failure families. Training-based methods learn generative reward models,
rubric generators, rubric-conditioned judges, or verifier capabilities within the policy
\citep{writingzero,rubricarm,genrubric,coverrl}, at the cost of a model-training stage and its
compute. Training-free pipelines that use prompted rubric-guided refinement or aggregate several
aspect judgments \citep{deepverifier,mav} avoid that cost but keep the verifier fixed during
deployment, where it can become stale as responses improve and task demands shift
\citep{dreamrl,evors,rqgm2026}.

To address these weaknesses, we introduce \emph{Unified Co-Evolution} (UCoE), a training-free
method that evolves an executable verifier program alongside the harness state with model parameters fixed. The program keeps every criterion explicit, contrast testing measures its behavior on each failure family, and evolution during deployment keeps it from going stale. UCoE synthesizes candidate verifier programs and tests them on response contrasts, pairs
of responses that differ in one targeted property with a declared preference
\citep{gardner2020contrast,sai2021perturbation}, grouped by property into families. It selects
the candidate with the highest \emph{driving capacity}, its discrimination on the family where it is
weakest, so that strength on easy properties cannot mask a blind spot. During deployment the
selected verifier diagnoses one weakness in the draft and guides one local repair. It then chooses
between draft and revision and decides whether the repair may enter the harness state.
Diagnostics from
completed tasks seed new candidates, which replace the current verifier only after meeting
the update criteria on the same contrasts. Figure~\ref{fig:motivation}c sketches this
loop.

We evaluate UCoE on six open-ended benchmarks and three benchmarks with verifiable endpoints
using Qwen3-8B and Qwen3.8-27B actors. UCoE has the highest preference against matched direct
generation of all compared feedback methods on every open-ended benchmark at both actor sizes,
averaging $59.00\%$ and $57.19\%$, and it improves every available native metric. On the
verifiable tasks it raises the mean from $36.06\%$ to $38.48\%$ without a loss on any benchmark.
Ablations attribute most of the gain to current-task repair, with verifier selection, verifier
evolution, and retained procedures each adding to it. Our contributions are:
\begin{itemize}
    \item We introduce UCoE, a training-free method that evolves an executable verifier program
    with the harness state, so feedback keeps pace with the agent it guides.
    \item We propose driving capacity, a verifier's discrimination on its weakest contrast family,
    and show that it bounds achievable selection gain, giving a safe criterion for verifier updates.
    \item We show state-of-the-art training-free self-improvement on nine benchmarks, with gains
    that carry over to GPT-5.4 actors and a Codex agent harness on SpreadsheetBench-2.
\end{itemize}

\section{Related work}

\noindent\textbf{Training-free self-improvement.}
Training-free methods improve the current response or retain non-parametric state that shapes
future behavior. Self-Refine revises responses within a task \citep{selfrefine}, while Reflexion,
ExpeL, Dynamic Cheatsheet, and A-MEM retain or organize reusable experience across trials and
tasks \citep{reflexion,expel,dynamiccheatsheet,amem}. ACE and Recursive Harness Self-Improvement
extend adaptation to the agent's context or prompt-level harness \citep{ace,rhi2026}. Test-time
scaling allocates additional inference computation to sampling and verification, including in
agentic workflows \citep{snell2025testtime,agenttesttime}. All of these methods take the feedback
or selection process as given. UCoE adapts that process too: its verifier program evolves
alongside the harness state, and together they determine which response is served and which
experience is kept.

\noindent\textbf{Open-ended rewards and verification.}
For tasks without exact answers, prior work constructs model-based evaluation signals from
generated rubrics, rubric anchors, recursively decomposed and reweighted criteria, or pairwise
generative rewards \citep{rar,rubricanchors,rrd,writingzero}. Other methods combine supplied
verifiers, elicit rubric-guided or aspect-specific judgments, or construct self-verifiable proxy
tasks \citep{fuse,deepverifier,mav,rlsvr}. CoNL evaluates critiques by the improvements they induce
\citep{conl}, and Eureka expresses feedback as executable reward code for simulated control
\citep{eureka}. The reliability of generated evaluators also needs testing: EvalResearchBench
measures how well agent-designed evaluators reproduce the model rankings of target benchmarks
\citep{zhang2026evalresearchbenchaiagentsdesign}. UCoE tests whole verifier programs, criteria and
aggregation included, on response contrasts with declared preference directions, and it lets the
diagnostics of completed tasks motivate later changes to the program.

\noindent\textbf{Co-evolution.}
Prior co-evolution methods update feedback mechanisms alongside the systems they guide.
Rubric-centered work co-evolves rubrics with responses or research policies, adjusts rubric
complexity during training, or jointly learns rubric generators, judges, and verifier capabilities
\citep{evorubric,drtulu,care,rubricarm,genrubric,coverrl}. DREAM-RL and EvoRS evolve executable
Reward-DAGs from rollout diagnostics or on-policy response traces \citep{dreamrl,evors}. Other work
pairs task generators with solvers, compares changing evaluators with fixed anchors, adapts
multi-agent organization at test time, or evolves programmatic annotators from labeled validation
data \citep{scope,rqgm2026,zhang2026evochamber,xu2026evopool}. Most of these methods rely on
parameter training, on-policy reward traces, or labeled anchors. UCoE keeps the co-evolutionary
view but moves all adaptation outside the model parameters. The criteria and aggregation of an
executable verifier program change together with the agent's harness state, diagnostics from
completed tasks suggest those changes, and response contrasts with declared preferences give a
common basis for judging them.

\section{Method}
\label{sec:method}

We introduce UCoE in this section. Figure~\ref{fig:algorithm} gives an overview of the method.
See App.~\ref{app:full-algorithm} for the complete algorithm.

\begin{figure}[!t]
\centering
\includegraphics[width=1\linewidth]{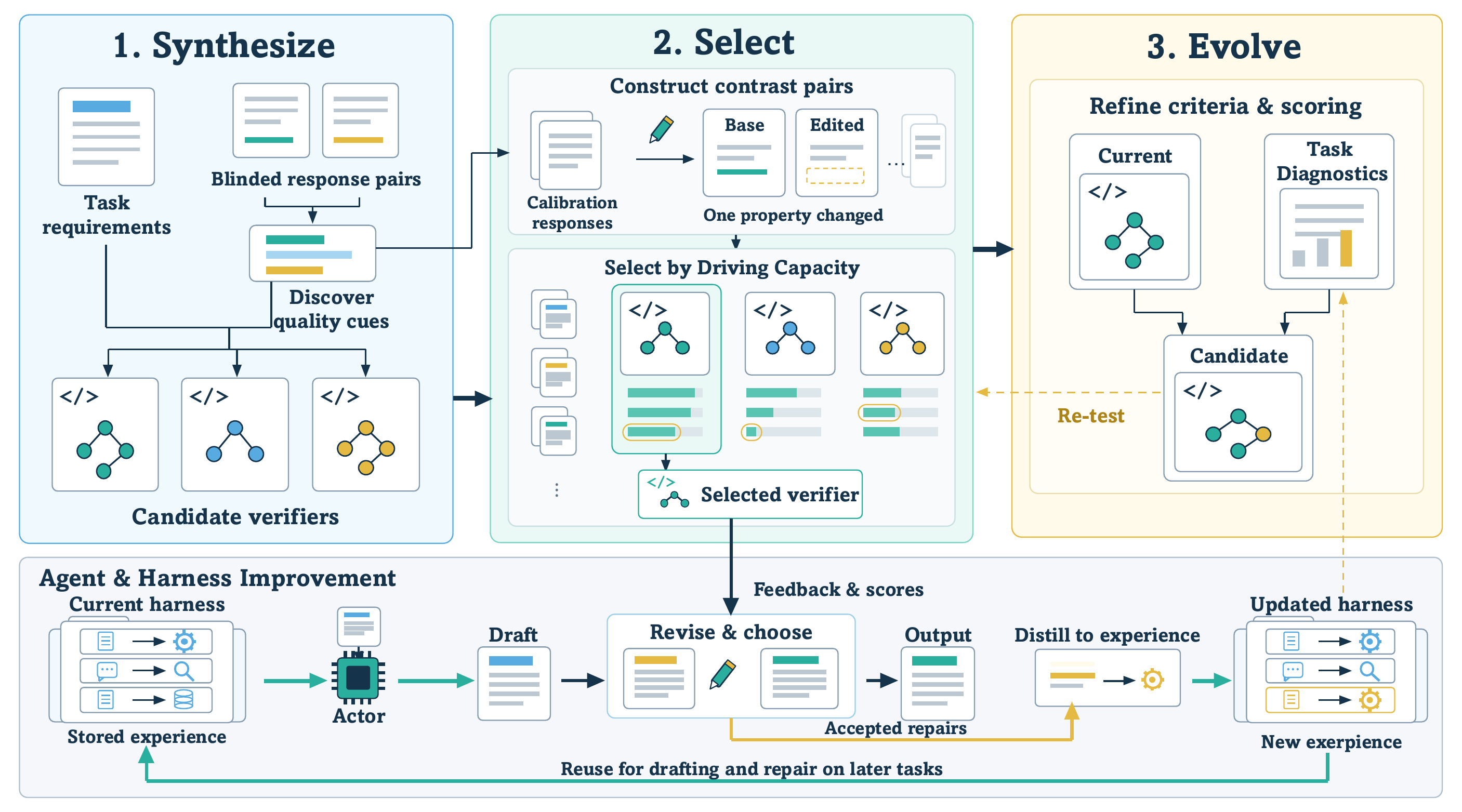}
\vspace{-20pt}
\caption{\textbf{UCoE workflow.} Top, the verifier side: (1) Synthesize proposes candidate
verifier programs from the task requirements and from quality cues discovered in blinded response
pairs. (2) Select edits calibration responses one property at a time to form response contrasts
and deploys the candidate that is strongest on its weakest family. (3) Evolve combines the
current verifier with diagnostics from completed tasks into a refined candidate, which is
re-tested on the same contrasts before it can replace the current verifier. Bottom, the agent
side: the fixed actor drafts with the current harness state, the deployed verifier guides one
revision and chooses between draft and revision, and accepted repairs are distilled into
experience that the updated harness state reuses on later tasks.}
\vspace{-16pt}
\label{fig:algorithm}
\end{figure}

\subsection{Problem setup}

Let $x_1,\ldots,x_T$ be an ordered task stream and $F_{\theta_0}$ a fixed actor. Each method
$\cA$ processes the same stream with its own harness state $H_t$, and no method changes any
model parameter. What $H_t$ contains, and how the actor reads it, depends on the agent system.
It may hold instructions, plans, tools, skills, or memory that shape later behavior. In the main
experiments it is a bounded procedural memory, and every adaptive method receives an
isolated instance that starts empty and follows the same capacity, retrieval, and update rules.

All methods follow the same per-task lifecycle. Given $H_t$, the actor first generates a draft
$y_t^0=F_{\theta_0}(x_t;H_t)$. The method then applies its own feedback and selection rules,
which may produce at most one revision $y_t^R$, and returns $y_t\in\{y_t^0,y_t^R\}$. Only after
$y_t$ has been returned does the method apply its harness update
$H_{t+1}=U_{\cA}(H_t,\tau_t^{\cA})$, where $\tau_t^{\cA}$ is the completed task trace. This
ordering lets feedback revise the current response while task-derived state affects only later
tasks.

Let $\mathcal I\subseteq\{1,\ldots,T\}$ denote the tasks included in final evaluation. We score
each method prequentially:
\begin{equation}
S_{\rm preq}(\cA)=\frac{1}{|\mathcal I|}\sum_{t\in\mathcal I}Q_t(x_t,y_t),
\qquad \theta_t=\theta_0,
\label{eq:prequential-objective}
\end{equation}
where $Q_t$ is the final evaluator's score for task $t$.
Throughout the stream, methods see public task requirements and their own model-generated
artifacts, never benchmark answers or final-evaluator feedback. The score is prequential because
each $y_t$ is produced from the state inherited from earlier tasks, before task $t$ contributes
to $H_{t+1}$.

\subsection{Verifier synthesis and contrast-based selection}
\label{sec:theory}

\noindent\textbf{Initial verifier synthesis.}
UCoE represents a verifier as a program that maps a task $x$ and a response $y$ to a vector of
criterion scores and a scalar reward:
\begin{equation}
\cV=(\Phi,A,D),\qquad
\mathbf{s}_{\cV}(x,y)=\bigl(\phi_1(x,y),\ldots,\phi_M(x,y)\bigr),\qquad
r_{\cV}(x,y)=A\!\left(\mathbf{s}_{\cV}(x,y)\right),
\label{eq:reward-interface}
\end{equation}
where $\Phi$ is the set of criterion scorers, $A$ combines their scores into a reward in which
higher is better, and $D$ turns the weakest participating criterion into repair feedback.
Semantic criteria are scored by a fixed language model that we call the verifier model, which
also proposes programs and diagnoses weaknesses. Evolution may add criteria to $\Phi$ and change
$A$. The diagnosis procedure, the program interpreter, and the parameters of the verifier model
never change.

Public task metadata states broad requirements but can miss domain-specific qualities that recur
in real responses. UCoE therefore begins with relation discovery. On a separate development set,
the actor and a larger frozen comparison model answer the same prompts, and the verifier model
reads each pair in random order with system identities hidden and turns differences that recur
across prompts into general criteria together with targeted edits that later serve contrast
construction. These fixed relations form
$\mathcal R$. Using the task metadata, the shared criteria, and $\mathcal R$, the verifier model
then proposes candidate programs. Each candidate specifies how to combine the criterion scores and
may add up to two criteria of its own. A lightweight validity check removes malformed or
non-executable programs, and the surviving pool $\mathcal K_0$ is scored on separate response
contrasts, where driving capacity picks the initial verifier.

\noindent\textbf{Response-contrast construction.}
Candidate verifiers need a common behavioral test that shows whether their scores respond to
intended changes in quality and which candidate tracks those changes most reliably. Arbitrary
response pairs differ in many ways at once, so UCoE instead constructs controlled contrasts that
isolate one property at a time \citep{gardner2020contrast,sai2021perturbation}. Each contrast
starts from a calibration response, written by the actor for a public prompt that is disjoint
from deployment, and applies one targeted edit that breaks a single property while leaving
everything else in place. The base response $y^+$ and the degraded response $y^-$ then differ only
in that property, and the declared preference for $y^+$ concerns that property alone. Three
general families cover omitted required content, unsupported claims, and irrelevant material.
Relation discovery adds benchmark-specific families under the same rule: a relation is kept only
if it recurs in at least three development prompts and can be expressed as one binary criterion
with one edit that breaks only that property. The resulting pairs form the response contrasts $\cB$ and are
split by source before any verifier is scored. The selection split ranks candidates, and the
held-out split later compares a candidate with the current verifier.
See App.~\ref{app:construction-details} for construction details.

\noindent\textbf{Driving capacity and selection.}
Static validation only establishes that a program runs. Choosing among valid programs at random
ignores how they behave on the contrasts, and ranking them by mean accuracy lets strong
performance on easy families hide a blind spot on a hard one. UCoE therefore ranks candidates by
their weakest family.

Each verifier $\cV$ induces a strict-margin selector $h_{\cV}=h_{\cV,\epsilon}$ on each
contrast, which picks $y^+$ over $y^-$ only when the reward of $y^+$
exceeds that of $y^-$ by more than $\epsilon$. A tie or a sub-margin lead counts as a failure. For
a contrast family $j$, let $p_j(h)$ be the probability that selector $h$ picks $y^+$, and define
its centered accuracy as
\begin{equation}
J_j(h)=2p_j(h)-1
=2\Pr[h(x,y^+,y^-)=y^+\mid j]-1.
\label{eq:family-separation}
\end{equation}
Chance-level selection gives $J_j=0$, and positive or negative values mean above- or below-chance
selection. The signed driving capacity of $\cV$ on the response contrasts $\cB$ is one quarter of
its worst family value, and the nonnegative driving capacity clips it at zero:
\begin{equation}
\widetilde{\mathcal C}_{\cB}(\cV)
\coloneqq\frac14\min_j J_j(h_{\cV}),
\qquad
\mathcal C_{\cB}(\cV)\coloneqq
\max\{0,\widetilde{\mathcal C}_{\cB}(\cV)\}.
\label{eq:constructed-driving-capacity}
\end{equation}
The factor of one quarter gives the quantity a covariance meaning. For valid contrasts, if one
of the two response positions is drawn uniformly at random, $J_j/4$ is the covariance between
the selector choosing that position and that position holding the targeted property, so driving
capacity is the smallest such selection covariance over the tested families. See Proposition~\ref{cor:pairwise-capacity} in
App.~\ref{app:pairwise-theory} for the precise statement.

The minimum also has a robustness meaning. Suppose the $m$ tested families are mixed with weights
$w\in\Delta_m=\{w:w_j\geq0,\sum_jw_j=1\}$, so that the selector succeeds with probability
$p_w(h_{\cV})=\sum_jw_jp_j(h_{\cV})$. The least favorable mixture puts all of its mass on a
weakest family, and Theorem~\ref{thm:robust-selection} in App.~\ref{app:finite-pool-theory} shows that
$\inf_{w\in\Delta_m}p_w(h_{\cV})=\min_jp_j(h_{\cV})=1/2+2\widetilde{\mathcal C}_{\cB}(\cV)$.
Maximizing signed driving capacity therefore maximizes the worst-case probability of selecting the
preferred response over every mixture of the tested families. UCoE optimizes the signed value
rather than the clipped one because clipping would make all below-chance candidates look alike.
Its empirical objective at checkpoint $k$, with $k=0$ denoting initialization, is
\begin{equation}
 \widehat\cV\in
 \operatorname*{arg\,max}_{\cV\in\mathcal K_k}
\min_j\widehat J_j(h_{\cV}),
\label{eq:maxmin-selection}
\end{equation}
where $\mathcal K_k$ is the candidate pool at that checkpoint and $\widehat J_j$ is centered
accuracy on the selection split. Proposition~\ref{prop:finite-selection} bounds how far the empirical
maximizer can fall short of the best candidate in the pool. The winner from $\mathcal K_0$ becomes
the initial verifier $\cV_0$. Later pools also contain the current verifier, and their winner must
meet the update criteria of Section~\ref{sec:verifier-lifecycle} before it replaces the current
verifier.

Driving capacity measures discrimination, and discrimination pays off only when the compared
responses differ in quality. Let $Z_i\in\{0,1\}$ mark whether response $i$ of a pair is adequate
under the final evaluation, let $S_h\in\{1,2\}$ be the index that $h$ selects, and let
$q=\Pr(Z_1\ne Z_2)$ be the fraction of pairs on which the choice matters. For $q>0$, the
mixed-pair centered accuracy $J_{\rm mix}(h)=2\Pr(Z_{S_h}=1\mid Z_1\ne Z_2)-1$ measures how often
the selector picks the adequate response when the two differ. Pairs of equal adequacy contribute
nothing, so for two selectors $h$ and $h'$ on the same fixed pairs,
\begin{equation}
\mathbb E[Z_{S_{h'}}-Z_{S_h}]=\frac q2
\{J_{\rm mix}(h')-J_{\rm mix}(h)\}.
\label{eq:mixed-pair-gain}
\end{equation}
Proposition~\ref{cor:opportunity-discrimination} records this decomposition: the actor supplies the
opportunity $q$, the selector supplies the ability $J_{\rm mix}$ to use it, and when $q=0$ no
selector can change expected adequacy. Whether gains on constructed contrasts carry over to the
actor's natural revisions depends on how well the contrast families represent the errors those
revisions contain.

\subsection{Response revision and co-evolution}
\label{sec:verifier-lifecycle}

\noindent\textbf{Response revision and selection.}
Let $\cV_k$ be the verifier active at task $t$. It scores the draft $y_t^0$, and from that fixed
score vector and its weakest participating criterion the diagnosis procedure returns a target
criterion $c_t$, localized evidence $w_t$, and a repair instruction $u_t$, written
$(c_t,w_t,u_t)=D_{\cV_k}(x_t,y_t^0)$. Conditioned on this feedback and on $H_t$, the fixed actor
makes one local repair while preserving unrelated content:
\begin{equation}
y_t^R=F_{\theta_0}(x_t,y_t^0,c_t,w_t,u_t;H_t).
\label{eq:ucoe-proposal}
\end{equation}
The same $\cV_k$ then scores the draft and the revision, and the revision is served only when its
reward lead exceeds a margin $\epsilon\geq0$:
\begin{equation}
y_t=\begin{cases}
y_t^R,&r_{\cV_k}(x_t,y_t^R)>r_{\cV_k}(x_t,y_t^0)+\epsilon,\\
y_t^0,&\text{otherwise}.
\end{cases}
\label{eq:revision-acceptance}
\end{equation}
Whenever the comparison cannot be made, because the feedback was unusable, the revision is
missing, or a score is invalid, and whenever the lead is a tie or falls below the margin, the
draft is served. We write $a_t\in\{0,1\}$ for this acceptance decision.

\noindent\textbf{Harness update.}
After the response is returned, a system-specific operator updates the harness state:
\begin{equation}
H_{t+1}=\operatorname{UpdateHarness}\bigl(H_t,x_t,y_t^0,y_t^R,c_t,w_t,
a_t\bigr).
\label{eq:behavior-update}
\end{equation}
Only an accepted repair, $a_t=1$, may add task-derived content to $H_{t+1}$. In our bounded
procedural-memory implementation, an accepted revision yields at most one generalized
trigger--action procedure once task-specific content has been removed, and a rejected revision
may update metadata only. See App.~\ref{app:protocol-settings} for the retrieval,
deduplication, and eviction rules.

\noindent\textbf{Verifier evolution.}
As task demands and $H_t$ change, the draft and revision pairs that the actor produces can drift
away from the contrasts on which the initial verifier was selected, and its criteria or
aggregation can go stale. Equation~\eqref{eq:mixed-pair-gain} makes the cost concrete: the
selection gain on the pairs the actor now produces depends on how well the current selector
discriminates on exactly those pairs. UCoE therefore revisits the verifier at scheduled
checkpoints, after $y_t$ and $H_{t+1}$ have been finalized. The verifier model synthesizes new
candidates from the current verifier's definitions and from unlabeled summaries of the completed
tasks, namely how often revisions were accepted, which criteria were repeatedly the weakest, and
how the harness state has grown. A candidate may append criteria or change the aggregation. After
validation, UCoE adds the current verifier to the pool and repeats the maximin selection of
Section~\ref{sec:theory} on the fixed contrast families. The winner replaces the current verifier
only if it meets the update criteria: it must improve worst-family discrimination on the selection
split, win more held-out contrasts than it loses against the current verifier, and stay within a
fixed regression tolerance on every existing family. Otherwise the current verifier is retained. Any
adoption takes effect on the next task. See App.~\ref{app:empirical-update} for the formal update
criteria and App.~\ref{app:implementation-details} for implementation details.

\section{Experimental setup}

\begin{table}[tb!]
\centering
\small
\setlength{\tabcolsep}{3pt}
\renewcommand{\arraystretch}{1.10}
\caption{\textbf{Open-ended preference against matched direct generation.} Each entry is the
tie-adjusted percentage of tasks on which a method's response is preferred to the Direct
Generation response of the same actor, as judged by Qwen3.5-397B-A17B in both presentation orders
with ties given half credit, so $50$ is parity and higher is better. All methods share the actor,
the task order, and the fixed-model, one-revision protocol. Overall is the unweighted mean over the
six benchmarks, and bold marks the best method in each column.}
\label{tab:cross-benchmark}
\resizebox{\linewidth}{!}{%
\begin{tabular}{lccccccc}
\toprule
\textbf{Method} & \shortstack{\textbf{Arena-}\\\textbf{Hard}} & \shortstack{\textbf{Writing}\\\textbf{Prompts}} & \textbf{WildBench} & \shortstack{\textbf{Writing}\\\textbf{Bench}} & \shortstack{\textbf{Multi}\\\textbf{Challenge}} & \shortstack{\textbf{Refine}\\\textbf{Bench}} & \textbf{Overall} \\
\midrule
\multicolumn{8}{@{}l}{\textbf{Qwen3-8B}} \\
Direct Generation & $50.00$ & $50.00$ & $50.00$ & $50.00$ & $50.00$ & $50.00$ & $50.00$ \\
LLM as a Judge & $50.60$ & $50.54$ & $46.34$ & $49.50$ & $38.87$ & $51.98$ & $47.97$ \\
RaR~\citep{rar} & $51.20$ & $50.54$ & $48.54$ & $51.90$ & $38.48$ & $51.26$ & $48.65$ \\
DREAM~\citep{dreamrl} & $45.60$ & $43.90$ & $43.12$ & $52.60$ & $31.05$ & $53.78$ & $45.01$ \\
SpyRL~\citep{rlsvr} & $51.10$ & $49.27$ & $35.74$ & $47.00$ & $33.98$ & $53.06$ & $45.03$ \\
\rowcolor{ucoerow}
\textbf{UCoE} & $\mathbf{58.30}$ & $\mathbf{62.79}$ & $\mathbf{57.42}$ &
$\mathbf{56.90}$ & $\mathbf{57.42}$ & $\mathbf{61.15}$ & $\mathbf{59.00}$ \\
\midrule
\multicolumn{8}{@{}l}{\textbf{Qwen3.8-27B}} \\
Direct Generation & $50.00$ & $50.00$ & $50.00$ & $50.00$ & $50.00$ & $50.00$ & $50.00$ \\
LLM as a Judge & $49.50$ & $48.00$ & $48.49$ & $50.30$ & $36.33$ & $48.74$ & $46.89$ \\
RaR~\citep{rar} & $51.60$ & $50.49$ & $47.56$ & $50.60$ & $43.36$ & $44.78$ & $48.07$ \\
DREAM~\citep{dreamrl} & $53.80$ & $59.96$ & $48.44$ & $52.10$ & $46.48$ & $48.38$ & $51.53$ \\
SpyRL~\citep{rlsvr} & $43.40$ & $58.84$ & $41.89$ & $48.60$ & $36.33$ & $42.99$ & $45.34$ \\
\rowcolor{ucoerow}
\textbf{UCoE} & $\mathbf{59.70}$ & $\mathbf{63.92}$ & $\mathbf{55.42}$ & $\mathbf{54.40}$ & $\mathbf{51.95}$ & $\mathbf{57.73}$ & $\mathbf{57.19}$ \\
\bottomrule
\end{tabular}}
\vspace{-20pt}
\end{table}

\noindent\textbf{Benchmarks and models.}
Our open-ended evaluation spans six benchmarks. Arena-Hard \citep{li2024arenahard} and WildBench
\citep{lin2024wildbench} test general assistance on challenging user requests. WritingPrompts
\citep{fan2018hierarchical} tests story generation, and WritingBench \citep{wu2025writingbench}
covers writing across diverse domains. MultiChallenge \citep{sirdeshmukh2025multichallenge}
targets multi-turn instruction following, and the Free-Form split of RefineBench
\citep{lee2025refinebenchevaluatingrefinementcapability} tests response refinement. To check
whether the same method transfers beyond model-judged evaluation, we add three benchmarks with
objective endpoints. IFBench tests compliance with verifiable instruction constraints
\citep{pyatkin2025ifbench}, Hard-MATH is our competition-mathematics stream drawn from MATH-500
and AIME 2022--2025 \citep{hendrycks2021math,lightman2024verify}, and BigCodeBench-Hard tests code
generation through execution \citep{zhuo2024bigcodebench}. We run Qwen3-8B and Qwen3.8-27B \citep{qwen3} as actors and use Qwen3.6-35B-A3B as the fixed verifier model for program synthesis, semantic
criterion scoring, and diagnosis. In relation discovery, the larger comparison model is
Qwen3.8-27B for the Qwen3-8B actor and Qwen3.5-397B-A17B for the Qwen3.8-27B actor.
See App.~\ref{app:benchmarks} for splits and sizes and App.~\ref{app:hyperparameters} for
default hyperparameters.

\noindent\textbf{Baselines.}
We compare UCoE with common baselines for open-ended self-improvement, chosen to cover different
forms of feedback. Direct Generation uses no feedback and provides the raw-response reference.
LLM as a Judge makes a fixed holistic judgment of response quality \citep{zheng2023judging}. RaR
scores responses against task-specific flat rubrics \citep{rar}. DREAM organizes dependent
criteria into a Reward-DAG \citep{dreamrl}. SpyRL derives feedback from a task-transformation
game with self-verifiable rewards \citep{rlsvr}. We adapt all four feedback baselines to the same
fixed-model, one-revision protocol: each method uses its own feedback to guide at most one
revision, scores the draft and the revision with its own signal, and returns the revision only
when it passes the method's acceptance rule. Stateful methods receive separate procedural memories
with the same capacity and update rules, and all implementations transfer the feedback mechanisms
without reproducing the original parameter-training pipelines.
See App.~\ref{app:baselines} and Table~\ref{tab:baseline-adaptations} for each baseline's
implementation and acceptance rule.

\noindent\textbf{Evaluation.}
Open-ended benchmarks do not share an exact notion of correctness, and their native scores are not
directly comparable, so our primary endpoint is tie-adjusted pairwise preference against the
matched Direct Generation response. This follows established pairwise evaluation protocols for
open-ended model outputs \citep{zheng2023judging,li2024arenahard} and places every method on the
same relative scale. Qwen3.5-397B-A17B judges each pair of full responses in both presentation
orders without receiving method identities. We merge the two judgments at the task level and
assign half credit to ties, so values above $50\%$ favor the evaluated method over Direct
Generation. We also report the native benchmark-specific metrics wherever they exist, together
with accuracy or Pass@1 on the verifiable tasks. See App.~\ref{app:statistical-tests} for
the full evaluation protocols.

\section{Results}
\label{sec:results}

\begin{table}[tb!]
\centering
\small
\setlength{\tabcolsep}{5pt}
\renewcommand{\arraystretch}{1.08}
\caption{\textbf{Verifiable tasks with Qwen3-8B.} Accuracy for IFBench and Hard-MATH,
Pass@1 for BigCodeBench-Hard.}
\label{tab:verifiable-performance}
\begin{tabular}{lcccc}
\toprule
\textbf{Method} & \textbf{IFBench} & \textbf{Hard-MATH} & \textbf{BigCodeBench-Hard} & \textbf{Overall} \\
\midrule
Direct Generation & $30.00$ & $56.54$ & $\mathbf{21.62}$ & $36.06$ \\
LLM as a Judge & $33.33$ & $55.24$ & $16.89$ & $35.15$ \\
RaR & $32.00$ & $57.33$ & $18.92$ & $36.08$ \\
DREAM & $28.67$ & $\mathbf{59.42}$ & $18.24$ & $35.44$ \\
SpyRL & $30.67$ & $58.12$ & $18.24$ & $35.68$ \\
\rowcolor{ucoerow}
\textbf{UCoE} & $\mathbf{34.67}$ & $59.16$ & $\mathbf{21.62}$ & $\mathbf{38.48}$ \\
\bottomrule
\end{tabular}
\vspace{-20pt}
\end{table}

\noindent\textbf{UCoE wins on every open-ended benchmark.} Table~\ref{tab:cross-benchmark} shows that UCoE is preferred
over matched direct generation in all twelve benchmark and actor settings, and that its
preference is the largest of any feedback method in every one of them. It averages $59.00\%$ with
Qwen3-8B and $57.19\%$ with Qwen3.8-27B, gains of $9.00$ and $7.19$ points over direct generation,
and the gains span general assistance, writing, multi-turn instruction following, and refinement.
The other feedback methods often fail to improve on direct generation. Even the best
baseline on each benchmark averages only $49.26\%$ at 8B and $51.60\%$ at 27B, and none exceeds
parity on WildBench at 8B. Every baseline receives feedback and one revision under the same
protocol, so the gap reflects the feedback signal itself rather than the opportunity to revise.
We also evaluate under each benchmark's native metric, and the same picture holds, with UCoE
improving over direct generation in every setting. See App.~\ref{app:benchmark-specific-results},
\ref{app:cost-analysis}, \ref{app:transfer-results}, \ref{app:statistical-results}, and
\ref{app:writingbench-natural-audit} for native-metric results, listed-price cost, backbone
transfer, pointwise intervals for Table~\ref{tab:cross-benchmark}, and an independent-critic
audit of the WritingBench revisions.

\noindent\textbf{UCoE also improves verifiable tasks.} The same lifecycle helps where correctness can be checked. In
Table~\ref{tab:verifiable-performance}, UCoE raises the three-benchmark mean from $36.06\%$ to
$38.48\%$. It improves IFBench by $4.67$ points and Hard-MATH by $2.62$ points while matching
direct generation on BigCodeBench-Hard. On code, every other feedback method lowers Pass@1 below
direct generation, and UCoE is the only method that preserves it, while on Hard-MATH DREAM
edges it by $0.26$ points.

\begin{wraptable}{r}{0.50\linewidth}
\vspace{-20pt}
\centering
\small
\setlength{\tabcolsep}{4pt}
\renewcommand{\arraystretch}{1.08}
\caption{\textbf{Arena-Hard ablations with Qwen3-8B.} Tie-adjusted preference against direct generation.}
\label{tab:ablation-summary}
\begin{tabularx}{\linewidth}{@{}Xr@{}}
\toprule
\textbf{Configuration} & \textbf{vs. direct \%} \\
\midrule
\rowcolor{ucoerow}
\textbf{UCoE} & $\mathbf{58.30}$ \\
Random initial verifier & $56.70$ \\
Fixed verifier & $56.60$ \\
Harness-state updates only & $50.80$ \\
No memory updates & $56.70$ \\
\bottomrule
\end{tabularx}
\vspace{-12pt}
\end{wraptable}

\noindent\textbf{Every component contributes.} Table~\ref{tab:ablation-summary} removes one component
of UCoE at a time on Arena-Hard, and every removal lowers preference. Removing current-task
repair, so that the verifier only updates the harness state, costs the most, from $58.30\%$ to
$50.80\%$. Replacing driving-capacity selection with a random initial verifier, freezing the
verifier instead of updating it at checkpoints, and disabling harness-state updates from accepted
repairs each lower preference to between $56.60\%$ and $56.70\%$. The ablations confirm each part
of the design: verifier-guided repair, verifier selection by driving capacity, verifier updates at
checkpoints, and harness-state updates from feedback all contribute. See
App.~\ref{app:hyperparameter-sensitivity} for the sensitivity study.

\Needspace{14\baselineskip}
\begin{wrapfigure}{r}{0.45\linewidth}
\vspace{-16pt}
\centering
\small
\includegraphics[width=\linewidth,trim=0 10 0 15,clip]{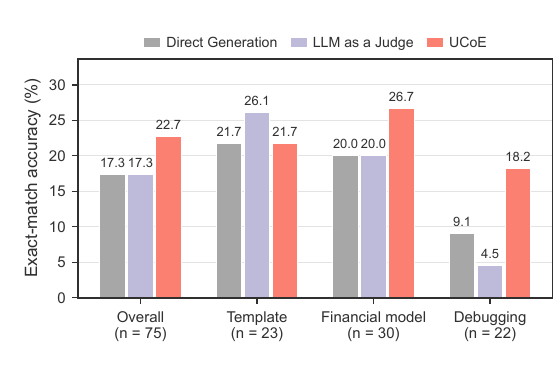}
\vspace{-15pt}
\caption{\textbf{UCoE inside the Codex agent harness on SpreadsheetBench-2.} Exact-match accuracy
on a 75-task stream, overall and by task category, with GPT-5.6-luna as the actor.}
\label{fig:ssb2-harness}
\vspace{-12pt}
\end{wrapfigure}
\noindent\textbf{UCoE transfers to a Codex agent harness.} The main experiments wrap a bare actor in a
prompt-level harness with procedural memory. To test whether UCoE also works inside a full
coding-agent harness, we run it in the Codex harness on SpreadsheetBench-2
\citep{zhu2026spreadsheetbench2evaluatingagents}, where the agent
edits real workbooks with its tools and the benchmark's checker scores each result after
recalculating every formula. The harness state here is the agent's own context, its tools, a
task skill file, and an experience block in the prompt, so the study also tests whether UCoE's
harness-state updates transfer beyond the procedural memory of the main experiments.
GPT-5.6-luna is the actor and GPT-5.6-sol the verifier model, and
Figure~\ref{fig:ssb2-harness} reports a 75-task stream with checkpoints after tasks 34 and 67. UCoE reaches
$22.7\%$ exact-match accuracy against $17.3\%$ for both Direct Generation
and LLM as a Judge. The gain comes from financial-model and debugging tasks, where the fixed
judge matches or trails direct generation, while on template tasks LLM as a Judge leads.
See App.~\ref{app:ssb2-harness} for details.

\begin{figure}[!htbp]
\centering
\begin{minipage}{0.84\linewidth}
\centering
\begin{minipage}[t]{0.485\linewidth}
\centering
\includegraphics[width=\linewidth]{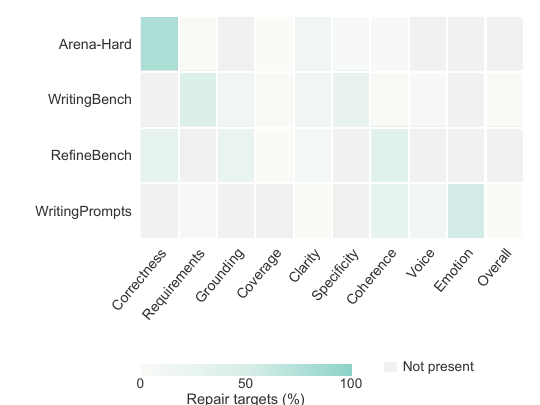}
\vspace{-2pt}
{\small\textbf{(a)} Repair-target distributions\par}
\end{minipage}\hfill
\begin{minipage}[t]{0.485\linewidth}
\centering
\includegraphics[width=\linewidth]{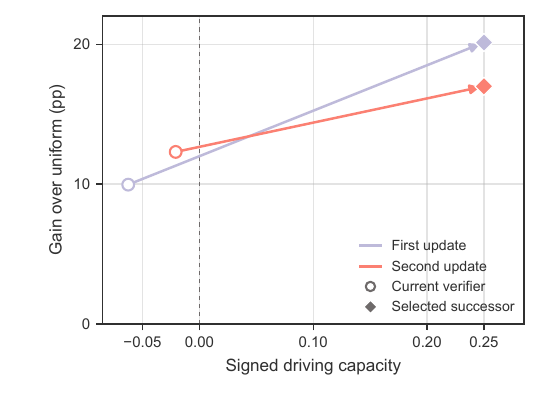}
\vspace{-2pt}
{\small\textbf{(b)} Verifier evolution\par}
\end{minipage}

\par\vspace{2pt}
\begin{minipage}[t]{0.485\linewidth}
\centering
\includegraphics[width=\linewidth]{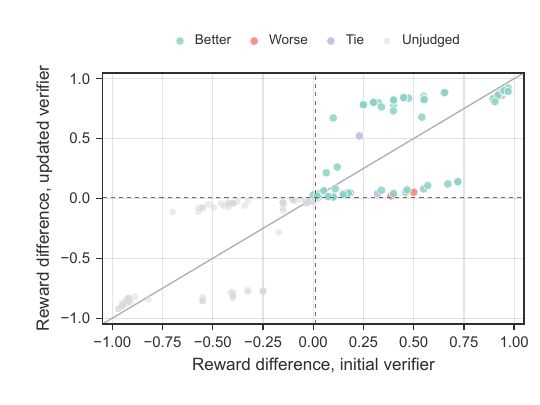}
\vspace{-2pt}
{\small\textbf{(c)} Fixed-proposal score shifts\par}
\end{minipage}\hfill
\begin{minipage}[t]{0.485\linewidth}
\centering
\includegraphics[width=\linewidth]{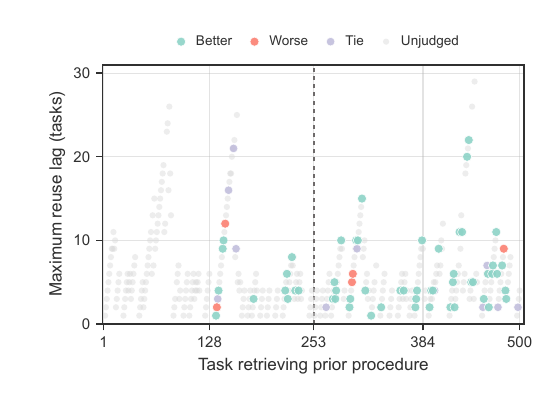}
\vspace{-2pt}
{\small\textbf{(d)} Procedure reuse\par}
\end{minipage}
\end{minipage}
\vspace{-5pt}
\caption{\textbf{Mechanism diagnostics.} (a) Share of repair targets per criterion group.
Gray marks groups absent from the deployed verifier. (b) Two verifier updates in WritingPrompts. Each arrow runs from the current verifier to the candidate
selected by driving capacity, in signed driving capacity against selection gain on frozen
response pairs. (c) Reward margins of 360 Arena-Hard proposals under the initial and updated
verifier. Dashed lines mark the acceptance margin, and colors give the external judgment for the
94 proposals the updated verifier selects. (d) Reuse lag of stored procedures at 381 Arena-Hard
tasks, colored by outcome. The dashed line marks the switch to mathematics
tasks.}
\label{fig:ucoe-insights}
\end{figure}
\noindent\textbf{Verifier updates and stored procedures behave as intended.} Figure~\ref{fig:ucoe-insights}(a) shows which criterion
the verifier targets when it diagnoses a draft. Correctness receives $78.9\%$ of Arena-Hard
repairs, requirements lead on WritingBench at $42.5\%$, coherence on RefineBench at $36.5\%$, and
emotion on WritingPrompts at $53.0\%$, so the deployed verifier directs repairs at the qualities each task
type demands. Figure~\ref{fig:ucoe-insights}(b) shows two verifier updates in WritingPrompts. Each candidate selected by driving capacity raises signed driving capacity
on the contrasts, from $-0.0625$ and $-0.0208$ to $0.25$, and also raises selection gain on
separate frozen response pairs, by $10.16$ and $4.69$ points, so gains on constructed contrasts
carry over to natural pairs. Figure~\ref{fig:ucoe-insights}(c) rescores 360 Arena-Hard proposals
under the initial and the updated verifier. The update widens some reward margins and narrows
others, which changes the set of accepted revisions, and among the 94 proposals it selects,
external judgments favor $73$ over the draft, $7$ favor the draft, and $14$ tie. The updated
verifier therefore accepts mostly real improvements. Figure~\ref{fig:ucoe-insights}(d) tracks
when stored procedures are retrieved. Of 639 retrieval events, 76 reuse a procedure after at
least ten intervening tasks, and on the 70 retrieving tasks with external judgments the served
revision beats the draft $54$ times, loses $5$, and ties $11$, so experience in the harness state
is reused long after it was stored and the tasks that reuse it are mostly judged better than
their drafts.

\FloatBarrier
\Needspace{8\baselineskip}
\section{Conclusion}

We introduced UCoE, a training-free framework in which an executable verifier program and the
harness state around a fixed actor evolve together. Response contrasts choose the verifier and
test every candidate that would replace it, and the verifier in turn shapes each response and the
experience the agent keeps. On six open-ended benchmarks, UCoE improves on matched direct
generation at both Qwen actor sizes, and on three verifiable benchmarks it raises the mean while
preserving Pass@1 on BigCodeBench-Hard. Repairing the current task contributes most of the gain,
verifier updates measurably change selection behavior, and procedures stored from accepted
repairs are reused many tasks later. Training-free self-improvement has so far adapted responses
and harness state around a fixed feedback process. UCoE extends that adaptation to the feedback
process itself, without touching model parameters. See App.~\ref{app:limitations} for
limitations.

\subsection*{AI use statement}

Generative AI tools assisted with writing assistance and polishing of the manuscript. The authors
take responsibility for the full content of this work.

\subsection*{Ethics statement}

This work does not recruit human subjects or collect new personal data. It evaluates language
models on existing research benchmarks. An evolving verifier can nevertheless reinforce its own
errors or biases through revision and persistent harness state, which may retain sensitive
information without appropriate filtering. Improved generation may also amplify misinformation or
other harmful content, while additional inference consumes computational resources. Response
contrasts do not establish safety, fairness, or privacy. Consequential deployments should therefore
use independent safety evaluation, human oversight, explicit retention and deletion rules, and
rollback mechanisms.

\subsection*{Reproducibility statement}

Algorithm~\ref{alg:ucoe} specifies the full lifecycle and update order. App.~\ref{app:additional-theory}
states the assumptions and gives complete proofs. App.~\ref{app:experimental-protocols}
documents benchmark splits, models, baseline implementations, hyperparameters, contrast
construction, verifier rules, and the procedural-memory implementation. App.~\ref{app:additional-results} reports native-metric results, cost, cross-backbone transfer,
hyperparameter sensitivity, the intervals behind the main tables, a natural-revision audit, and
qualitative examples.

\addtolength{\bibsep}{-1pt}
\bibliography{coevolve_reward}

@misc{rlsvr,
      title={From RLVR to RLSVR: Task Transformation Induces Self-Verifiable Rewards for Open-Ended LLM Self-Improvement}, 
      author={Qinsi Wang and Jing Shi and Huazheng Wang and Kun Wan and Yiran Wu and Bo Liu and Qingyun Wu and Hai Helen Li and Yiran Chen and Handong Zhao and Wentian Zhao},
      year={2026},
      eprint={2607.23802},
      archivePrefix={arXiv},
      primaryClass={cs.AI},
      url={https://arxiv.org/abs/2607.23802}, 
}

@misc{writingzero,
      title={Writing-Zero: Bridge the Gap Between Non-verifiable Tasks and Verifiable Rewards}, 
      author={Ruipeng Jia and Yunyi Yang and Yongbo Gai and Kai Luo and Shihao Huang and Jianhe Lin and Xiaoxi Jiang and Guanjun Jiang},
      year={2025},
      eprint={2506.00103},
      archivePrefix={arXiv},
      primaryClass={cs.CL},
      url={https://arxiv.org/abs/2506.00103}, 
}

@inproceedings{mav,
title={Multi-Agent Verification: Scaling Test-Time Compute with Multiple Verifiers},
author={Shalev Lifshitz and Sheila A. McIlraith and Yilun Du},
booktitle={Second Conference on Language Modeling},
year={2025},
url={https://openreview.net/forum?id=LriQ3NY9uL}
}

@inproceedings{
eureka,
title={Eureka: Human-Level Reward Design via Coding Large Language Models},
author={Yecheng Jason Ma and William Liang and Guanzhi Wang and De-An Huang and Osbert Bastani and Dinesh Jayaraman and Yuke Zhu and Linxi Fan and Anima Anandkumar},
booktitle={The Twelfth International Conference on Learning Representations},
year={2024},
url={https://openreview.net/forum?id=IEduRUO55F}
}

@inproceedings{deepverifier,
    title = "Inference-Time Scaling of Verification: Self-Evolving Deep Research Agents via Test-Time Rubric-Guided Verification",
    author = "Wan, Yuxuan  and
      Fang, Tianqing  and
      LI, Zaitang  and
      Huo, Yintong  and
      Wang, Wenxuan  and
      Mi, Haitao  and
      Yu, Dong  and
      Lyu, Michael R.",
    editor = "Liakata, Maria  and
      Moreira, Viviane P.  and
      Zhang, Jiajun  and
      Jurgens, David",
    booktitle = "Findings of the {A}ssociation for {C}omputational {L}inguistics: {ACL} 2026",
    month = jul,
    year = "2026",
    address = "San Diego, California, United States",
    publisher = "Association for Computational Linguistics",
    url = "https://aclanthology.org/2026.findings-acl.1243/",
    doi = "10.18653/v1/2026.findings-acl.1243",
    pages = "24822--24835",
    ISBN = "979-8-89176-395-1"}

@inproceedings{gardner2020contrast,
    title = "Evaluating Models' Local Decision Boundaries via Contrast Sets",
    author = "Gardner, Matt  and
      Artzi, Yoav  and
      Basmov, Victoria  and
      Berant, Jonathan  and
      Bogin, Ben  and
      Chen, Sihao  and
      Dasigi, Pradeep  and
      Dua, Dheeru  and
      Elazar, Yanai  and
      Gottumukkala, Ananth  and
      Gupta, Nitish  and
      Hajishirzi, Hannaneh  and
      Ilharco, Gabriel  and
      Khashabi, Daniel  and
      Lin, Kevin  and
      Liu, Jiangming  and
      Liu, Nelson F.  and
      Mulcaire, Phoebe  and
      Ning, Qiang  and
      Singh, Sameer  and
      Smith, Noah A.  and
      Subramanian, Sanjay  and
      Tsarfaty, Reut  and
      Wallace, Eric  and
      Zhang, Ally  and
      Zhou, Ben",
    editor = "Cohn, Trevor  and
      He, Yulan  and
      Liu, Yang",
    booktitle = "Findings of the Association for Computational Linguistics: EMNLP 2020",
    month = nov,
    year = "2020",
    address = "Online",
    publisher = "Association for Computational Linguistics",
    url = "https://aclanthology.org/2020.findings-emnlp.117/",
    doi = "10.18653/v1/2020.findings-emnlp.117",
    pages = "1307--1323"}

@inproceedings{sai2021perturbation,
    title = "Perturbation {C}heck{L}ists for Evaluating {NLG} Evaluation Metrics",
    author = "Sai, Ananya B.  and
      Dixit, Tanay  and
      Sheth, Dev Yashpal  and
      Mohan, Sreyas  and
      Khapra, Mitesh M.",
    editor = "Moens, Marie-Francine  and
      Huang, Xuanjing  and
      Specia, Lucia  and
      Yih, Scott Wen-tau",
    booktitle = "Proceedings of the 2021 Conference on Empirical Methods in Natural Language Processing",
    month = nov,
    year = "2021",
    address = "Online and Punta Cana, Dominican Republic",
    publisher = "Association for Computational Linguistics",
    url = "https://aclanthology.org/2021.emnlp-main.575/",
    doi = "10.18653/v1/2021.emnlp-main.575",
    pages = "7219--7234"}

@inproceedings{
selfrefine,
title={Self-Refine: Iterative Refinement with Self-Feedback},
author={Aman Madaan and Niket Tandon and Prakhar Gupta and Skyler Hallinan and Luyu Gao and Sarah Wiegreffe and Uri Alon and Nouha Dziri and Shrimai Prabhumoye and Yiming Yang and Shashank Gupta and Bodhisattwa Prasad Majumder and Katherine Hermann and Sean Welleck and Amir Yazdanbakhsh and Peter Clark},
booktitle={Thirty-seventh Conference on Neural Information Processing Systems},
year={2023},
url={https://openreview.net/forum?id=S37hOerQLB}
}

@inproceedings{
reflexion,
title={Reflexion: language agents with verbal reinforcement learning},
author={Noah Shinn and Federico Cassano and Ashwin Gopinath and Karthik R Narasimhan and Shunyu Yao},
booktitle={Thirty-seventh Conference on Neural Information Processing Systems},
year={2023},
url={https://openreview.net/forum?id=vAElhFcKW6}
}

@inproceedings{expel,
  title={Expel: Llm agents are experiential learners},
  author={Zhao, Andrew and Huang, Daniel and Xu, Quentin and Lin, Matthieu and Liu, Yong-Jin and Huang, Gao},
  booktitle={Proceedings of the AAAI Conference on Artificial Intelligence},
  volume={38},
  number={17},
  pages={19632--19642},
  year={2024}
}

@inproceedings{dynamiccheatsheet,
    title = "Dynamic Cheatsheet: Test-Time Learning with Adaptive Memory",
    author = "Suzgun, Mirac  and
      Yuksekgonul, Mert  and
      Bianchi, Federico  and
      Jurafsky, Dan  and
      Zou, James",
    editor = "Demberg, Vera  and
      Inui, Kentaro  and
      Marquez, Llu{\'i}s",
    booktitle = "Proceedings of the 19th Conference of the {E}uropean Chapter of the {A}ssociation for {C}omputational {L}inguistics (Volume 1: Long Papers)",
    month = mar,
    year = "2026",
    address = "Rabat, Morocco",
    publisher = "Association for Computational Linguistics",
    url = "https://aclanthology.org/2026.eacl-long.333/",
    doi = "10.18653/v1/2026.eacl-long.333",
    pages = "7080--7106",
    ISBN = "979-8-89176-380-7"}

@inproceedings{
amem,
title={A-Mem: Agentic Memory for {LLM} Agents},
author={Wujiang Xu and Zujie Liang and Kai Mei and Hang Gao and Juntao Tan and Yongfeng Zhang},
booktitle={The Thirty-ninth Annual Conference on Neural Information Processing Systems},
year={2025},
url={https://openreview.net/forum?id=FiM0M8gcct}
}

@inproceedings{
ace,
title={Agentic Context Engineering: Evolving Contexts for Self-Improving Language Models},
author={Qizheng Zhang and Changran Hu and Shubhangi Upasani and Boyuan Ma and Fenglu Hong and Vamsidhar Kamanuru and Jay Rainton and Chen Wu and Mengmeng Ji and Hanchen Li and Urmish Thakker and James Zou and Kunle Olukotun},
booktitle={The Fourteenth International Conference on Learning Representations},
year={2026},
url={https://openreview.net/forum?id=eC4ygDs02R}
}

@misc{verificationpayoff,
      title={When Does Verification Pay Off? A Closer Look at LLMs as Solution Verifiers}, 
      author={Jack Lu and Ryan Teehan and Jinran Jin and Mengye Ren},
      year={2026},
      eprint={2512.02304},
      archivePrefix={arXiv},
      primaryClass={cs.CL},
      url={https://arxiv.org/abs/2512.02304}, 
}

@misc{rrd,
      title={Rethinking Rubric Generation for Improving LLM Judge and Reward Modeling for Open-ended Tasks}, 
      author={William F. Shen and Xinchi Qiu and Chenxi Whitehouse and Lisa Alazraki and Shashwat Goel and Francesco Barbieri and Timon Willi and Akhil Mathur and Ilias Leontiadis},
      year={2026},
      eprint={2602.05125},
      archivePrefix={arXiv},
      primaryClass={cs.LG},
      url={https://arxiv.org/abs/2602.05125}, 
}

@misc{rubricanchors,
      title={Reinforcement Learning with Rubric Anchors}, 
      author={Zenan Huang and Yihong Zhuang and Guoshan Lu and Zeyu Qin and Haokai Xu and Tianyu Zhao and Ru Peng and Jiaqi Hu and Zhanming Shen and Xiaomeng Hu and Xijun Gu and Peiyi Tu and Jiaxin Liu and Wenyu Chen and Yuzhuo Fu and Zhiting Fan and Yanmei Gu and Yuanyuan Wang and Zhengkai Yang and Jianguo Li and Junbo Zhao},
      year={2025},
      eprint={2508.12790},
      archivePrefix={arXiv},
      primaryClass={cs.AI},
      url={https://arxiv.org/abs/2508.12790}, 
}

@inproceedings{
conl,
title={Conversation for Non-verifiable Learning: Self-Evolving Large Language Models through Meta-Evaluation},
author={Yuan Sui and Bryan Hooi},
booktitle={Forty-third International Conference on Machine Learning},
year={2026},
url={https://openreview.net/forum?id=Wba6w3pzbj}
}

@misc{fuse,
      title={FUSE: Ensembling Verifiers with Zero Labeled Data}, 
      author={Joonhyuk Lee and Virginia Ma and Sarah Zhao and Yash Nair and Asher Spector and Regev Cohen and Emmanuel J. Candès},
      year={2026},
      eprint={2604.18547},
      archivePrefix={arXiv},
      primaryClass={stat.ML},
      url={https://arxiv.org/abs/2604.18547}, 
}

@inproceedings{
rar,
title={Rubrics as Rewards: Reinforcement Learning Beyond Verifiable Domains},
author={Anisha Gunjal and Anthony Wang and Elaine Lau and Vaskar Nath and Yunzhong He and Bing Liu and Sean M. Hendryx},
booktitle={The Fourteenth International Conference on Learning Representations},
year={2026},
url={https://openreview.net/forum?id=c1bTcrDmt4}
}

@misc{evorubric,
      title={EvoRubric: Self-Evolving Rubric-Driven RL for Open-Ended Generation}, 
      author={Xin Guan and Xiaomeng Hu and Shen Huang and Zhenyi Wang and Bo Zhang and Zijian Li and Pengjun Xie and Bo Liu and Jiuxin Cao},
      year={2026},
      eprint={2605.29847},
      archivePrefix={arXiv},
      primaryClass={cs.CL},
      url={https://arxiv.org/abs/2605.29847}, 
}

@inproceedings{
care,
title={{CARE}: Curriculum-Aware Rubric Evolution for Open-Ended Text Generation},
author={Anonymous},
booktitle={Submitted to ACL Rolling Review - March 2026},
year={2026},
url={https://openreview.net/forum?id=jOlKMHT17D},
note={under review}
}

@inproceedings{
drtulu,
title={{DR} Tulu: Reinforcement Learning with Evolving Rubrics for Deep Research},
author={Rulin Shao and Akari Asai and Shannon Zejiang Shen and Hamish Ivison and Varsha Kishore and Jingming Zhuo and Xinran Zhao and Molly Park and Samuel G. Finlayson and David Sontag and Tyler Murray and Sewon Min and Pradeep Dasigi and Luca Soldaini and Faeze Brahman and Wen-tau Yih and Tongshuang Wu and Luke Zettlemoyer and Yoon Kim and Hannaneh Hajishirzi and Pang Wei Koh},
booktitle={Forty-third International Conference on Machine Learning},
year={2026},
url={https://openreview.net/forum?id=97NEP1pyS3}
}

@inproceedings{
rubricarm,
title={Alternating Reinforcement Learning for Rubric-Based Reward Modeling in Non-Verifiable {LLM} Post-Training},
author={Ran Xu and Tianci Liu and Zihan Dong and Tony Yu and Ilgee Hong and Carl Yang and Linjun Zhang and Tuo Zhao and Haoyu Wang},
booktitle={Forty-third International Conference on Machine Learning},
year={2026},
url={https://openreview.net/forum?id=3SDg8dtkS8}
}

@misc{scope,
      title={SCOPE: Self-Play via Co-Evolving Policies for Open-Ended Tasks}, 
      author={Wai-Chung Kwan and Aryo Pradipta Gema and Joshua Ong Jun Leang and Pasquale Minervini},
      year={2026},
      eprint={2605.31433},
      archivePrefix={arXiv},
      primaryClass={cs.CL},
      url={https://arxiv.org/abs/2605.31433}, 
}

@inproceedings{
dreamrl,
title={{DREAM}-{RL}: Dynamic Reward-System Evolution with an Agentic Meta-Judge for Reinforcement Learning on Open-Ended Tasks},
author={Anonymous},
booktitle={Submitted to ACL Rolling Review - May 2026},
year={2026},
url={https://openreview.net/forum?id=9idmRtG6M5},
note={under review}
}

@misc{agenttesttime,
      title={Scaling Test-time Compute for LLM Agents}, 
      author={King Zhu and Hanhao Li and Siwei Wu and Tianshun Xing and Dehua Ma and Xiangru Tang and Minghao Liu and Jian Yang and Jiaheng Liu and Yuchen Eleanor Jiang and Changwang Zhang and Chenghua Lin and Jun Wang and Ge Zhang and Wangchunshu Zhou},
      year={2025},
      eprint={2506.12928},
      archivePrefix={arXiv},
      primaryClass={cs.AI},
      url={https://arxiv.org/abs/2506.12928}, 
}

@inproceedings{snell2025testtime,
  title={Scaling LLM test-time compute optimally can be more effective than scaling parameters for reasoning},
  author={Snell, Charlie and Lee, Jaehoon and Xu, Kelvin and Kumar, Aviral},
  booktitle={International Conference on Learning Representations},
  volume={2025},
  pages={10131--10165},
  year={2025}
}

@inproceedings{
li2024arenahard,
title={From Crowdsourced Data to High-quality Benchmarks: Arena-Hard and Benchbuilder Pipeline},
author={Tianle Li and Wei-Lin Chiang and Evan Frick and Lisa Dunlap and Tianhao Wu and Banghua Zhu and Joseph E. Gonzalez and Ion Stoica},
booktitle={Forty-second International Conference on Machine Learning},
year={2025},
url={https://openreview.net/forum?id=KfTf9vFvSn}
}

@inproceedings{lin2024wildbench,
  title={Wildbench: Benchmarking llms with challenging tasks from real users in the wild},
  author={Lin, Bill Yuchen and Deng, Yuntian and Chandu, Khyathi and Ravichander, Abhilasha and Pyatkin, Valentina and Dziri, Nouha and Le Bras, Ronan and Choi, Yejin},
  booktitle={International Conference on Learning Representations},
  volume={2025},
  pages={47852--47870},
  year={2025}
}

@inproceedings{sirdeshmukh2025multichallenge,
    title = "{M}ulti{C}hallenge: A Realistic Multi-Turn Conversation Evaluation Benchmark Challenging to Frontier {LLM}s",
    author = "Deshpande, Kaustubh  and
      Sirdeshmukh, Ved  and
      Mols, Johannes Baptist  and
      Jin, Lifeng  and
      Hernandez-Cardona, Ed-Yeremai  and
      Lee, Dean  and
      Kritz, Jeremy  and
      Primack, Willow E.  and
      Yue, Summer  and
      Xing, Chen",
    editor = "Che, Wanxiang  and
      Nabende, Joyce  and
      Shutova, Ekaterina  and
      Pilehvar, Mohammad Taher",
    booktitle = "Findings of the Association for Computational Linguistics: ACL 2025",
    month = jul,
    year = "2025",
    address = "Vienna, Austria",
    publisher = "Association for Computational Linguistics",
    url = "https://aclanthology.org/2025.findings-acl.958/",
    doi = "10.18653/v1/2025.findings-acl.958",
    pages = "18632--18702",
    ISBN = "979-8-89176-256-5"}

@inproceedings{
wu2025writingbench,
title={WritingBench: A Comprehensive Benchmark for Generative Writing},
author={Yuning Wu and Jiahao Mei and Ming Yan and Chenliang Li and Shaopeng Lai and Yuran Ren and Wang Zijia and Ji Zhang and Mengyue Wu and Qin Jin and Fei Huang},
booktitle={The Thirty-ninth Annual Conference on Neural Information Processing Systems Datasets and Benchmarks Track},
year={2025},
url={https://openreview.net/forum?id=Pkskg9drDQ}
}

@inproceedings{
pyatkin2025ifbench,
title={Generalizing Verifiable Instruction Following},
author={Valentina Pyatkin and Saumya Malik and Victoria Graf and Hamish Ivison and Shengyi Huang and Pradeep Dasigi and Nathan Lambert and Hannaneh Hajishirzi},
booktitle={The Thirty-ninth Annual Conference on Neural Information Processing Systems Datasets and Benchmarks Track},
year={2025},
url={https://openreview.net/forum?id=yfYgwjj5F8}
}

@inproceedings{fan2018hierarchical,
  title={Hierarchical neural story generation},
  author={Fan, Angela and Lewis, Mike and Dauphin, Yann},
  booktitle={Proceedings of the 56th Annual Meeting of the Association for Computational Linguistics (Volume 1: Long Papers)},
  pages={889--898},
  year={2018}
}

@inproceedings{lee2025refinebenchevaluatingrefinementcapability,
title={RefineBench: Evaluating Refinement Capability of Language Models via Checklists},
author={Young-Jun Lee and Seungone Kim and Byung-Kwan Lee and Minkyeong Moon and Yechan Hwang and Jong Myoung Kim and Graham Neubig and Sean Welleck and Ho-Jin Choi},
booktitle={The Fourteenth International Conference on Learning Representations},
year={2026},
url={https://openreview.net/forum?id=GYJFJz9Dy5}
}

@inproceedings{hendrycks2021math,
title={Measuring Mathematical Problem Solving With the {MATH} Dataset},
author={Dan Hendrycks and Collin Burns and Saurav Kadavath and Akul Arora and Steven Basart and Eric Tang and Dawn Song and Jacob Steinhardt},
booktitle={Thirty-fifth Conference on Neural Information Processing Systems Datasets and Benchmarks Track (Round 2)},
year={2021},
url={https://openreview.net/forum?id=7Bywt2mQsCe}
}

@inproceedings{lightman2024verify,
title={Let's Verify Step by Step},
author={Hunter Lightman and Vineet Kosaraju and Yuri Burda and Harrison Edwards and Bowen Baker and Teddy Lee and Jan Leike and John Schulman and Ilya Sutskever and Karl Cobbe},
booktitle={The Twelfth International Conference on Learning Representations},
year={2024},
url={https://openreview.net/forum?id=v8L0pN6EOi}
}

@inproceedings{zhuo2024bigcodebench,
title={BigCodeBench: Benchmarking Code Generation with Diverse Function Calls and Complex Instructions},
author={Terry Yue Zhuo and Vu Minh Chien and Jenny Chim and Han Hu and Wenhao Yu and Ratnadira Widyasari and Imam Nur Bani Yusuf and Haolan Zhan and Junda He and Indraneil Paul and Simon Brunner and Chen GONG and James Hoang and Armel Randy Zebaze and Xiaoheng Hong and Wen-Ding Li and Jean Kaddour and Ming Xu and Zhihan Zhang and Prateek Yadav and Naman Jain and Alex Gu and Zhoujun Cheng and Jiawei Liu and Qian Liu and Zijian Wang and David Lo and Binyuan Hui and Niklas Muennighoff and Daniel Fried and Xiaoning Du and Harm de Vries and Leandro Von Werra},
booktitle={The Thirteenth International Conference on Learning Representations},
year={2025},
url={https://openreview.net/forum?id=YrycTjllL0}
}

@article{zheng2023judging,
  title={Judging llm-as-a-judge with mt-bench and chatbot arena},
  author={Zheng, Lianmin and Chiang, Wei-Lin and Sheng, Ying and Zhuang, Siyuan and Wu, Zhanghao and Zhuang, Yonghao and Lin, Zi and Li, Zhuohan and Li, Dacheng and Xing, Eric and others},
  journal={Advances in neural information processing systems},
  volume={36},
  pages={46595--46623},
  year={2023}
}

@misc{rhi2026,
      title={Recursive Harness Self-Improvement}, 
      author={Hyunin Lee and Jinglue Xu and Jeffrey Seely and Donghyun Lee and Matei Zaharia and Yujin Tang},
      year={2026},
      eprint={2607.15524},
      archivePrefix={arXiv},
      primaryClass={cs.LG},
      url={https://arxiv.org/abs/2607.15524}, 
}

@misc{rqgm2026,
      title={The Red Queen G\"odel Machine: Co-Evolving Agents and Their Evaluators}, 
      author={Alex Iacob and Andrej Jovanović and William F. Shen and Daniel Burkhardt and Meghdad Kurmanji and Nurbek Tastan and Lorenzo Sani and Niccolò Alberto Elia Venanzi and Ambroise Odonnat and Zeyu Cao and Bill Marino and Xinchi Qiu and Nicholas D. Lane},
      year={2026},
      eprint={2606.26294},
      archivePrefix={arXiv},
      primaryClass={cs.LG},
      url={https://arxiv.org/abs/2606.26294}, 
}

@article{price1970,
  title={Selection and covariance},
  author={Price, George R.},
  journal={Nature},
  volume={227},
  number={5257},
  pages={520--521},
  year={1970}
}

@article{youden1950,
  title={Index for rating diagnostic tests},
  author={Youden, William J},
  journal={Cancer},
  volume={3},
  number={1},
  pages={32--35},
  year={1950},
  publisher={Wiley Online Library}
}

@article{hoeffding1963,
  title={Probability inequalities for sums of bounded random variables},
  author={Hoeffding, Wassily},
  journal={Journal of the American statistical association},
  volume={58},
  number={301},
  pages={13--30},
  year={1963},
  publisher={Taylor \& Francis}
}

@article{mcnemar1947,
  title={Note on the sampling error of the difference between correlated proportions or percentages},
  author={McNemar, Quinn},
  journal={Psychometrika},
  volume={12},
  number={2},
  pages={153--157},
  year={1947},
  publisher={Cambridge University Press \& Assessment}
}

@inproceedings{
sagawa2020dro,
title={Distributionally Robust Neural Networks},
author={Shiori Sagawa* and Pang Wei Koh* and Tatsunori B. Hashimoto and Percy Liang},
booktitle={International Conference on Learning Representations},
year={2020},
url={https://openreview.net/forum?id=ryxGuJrFvS}
}

@misc{zhang2026evochamber,
      title={EVOCHAMBER: Test-Time Co-evolution of Multi-Agent System at Individual, Team, and Population Scales}, 
      author={Yaolun Zhang and Tianyi Xu and Shengyu Dai and Zhenwen Shao and Qingyun Wu and Huazheng Wang},
      year={2026},
      eprint={2605.11136},
      archivePrefix={arXiv},
      primaryClass={cs.AI},
      url={https://arxiv.org/abs/2605.11136}, 
}

@misc{xu2026evopool,
      title={EvoPool: Evolutionary Programmatic Annotation for Label-Efficient Specialized Supervision}, 
      author={Tianyi Xu and Yaolun Zhang and Xuan Ouyang and Huazheng Wang},
      year={2026},
      eprint={2606.01617},
      archivePrefix={arXiv},
      primaryClass={cs.CL},
      url={https://arxiv.org/abs/2606.01617}, 
}

@misc{evors,
      title={EvoRS: On-Policy Self-Evolution of Reward Systems for Open-Ended Reinforcement Learning}, 
      author={Weiyuan Li and Aili Chen and Xintao Wang and Yikai Zhang and Qingqing Dong and Jinghan Xu and Hongru Hou and Wenxuan Zhao and Chengkun Lang and Jun Gao and Yuanli Guo and Hongcheng Guo and Yanghua Xiao and Deqing Yang},
      year={2026},
      eprint={2609.12459},
      archivePrefix={arXiv},
      primaryClass={cs.AI},
      url={https://arxiv.org/abs/2609.12459}, 
}

@misc{genrubric,
      title={GenRubric: Self-Evolving Rubric Generation for Scalable LLM Evaluation}, 
      author={Yifan Chen and Haitao Li and Qingyao Ai and Fengbin Zhu and Tat-Seng Chua and Min Zhang and Yiqun Liu},
      year={2026},
      eprint={2608.29856},
      archivePrefix={arXiv},
      primaryClass={cs.CL},
      url={https://arxiv.org/abs/2608.29856}, 
}

@misc{coverrl,
      title={CoVerRL: Breaking the Consensus Trap in Label-Free Reasoning via Generator-Verifier Co-Evolution}, 
      author={Teng Pan and Yuchen Yan and Zixuan Wang and Ruiqing Zhang and Guiyang Hou and Wenqi Zhang and Weiming Lu and Jun Xiao and Yongliang Shen},
      year={2026},
      eprint={2603.17775},
      archivePrefix={arXiv},
      primaryClass={cs.CL},
      url={https://arxiv.org/abs/2603.17775}, 
}

@misc{zhu2026spreadsheetbench2evaluatingagents,
      title={SpreadsheetBench 2: Evaluating Agents on End-to-End Business Spreadsheet Workflows}, 
      author={Jian Zhu and Yuzheng Zhang and Zeyao Ma and Bohan Zhang and Armin Schoepf and Daniel Woloch and Peter Yiliu Wang and Guangyu Robert Yang and Samuel Jacob and Siddharth Nagisetty and Abhiram Chundru and Jean Lin and Spencer Mateega and Jing Zhang},
      year={2026},
      eprint={2606.29955},
      archivePrefix={arXiv},
      primaryClass={cs.SE},
      url={https://arxiv.org/abs/2606.29955}, 
}

@article{holm1979,
  title={A simple sequentially rejective multiple test procedure},
  author={Holm, Sture},
  journal={Scandinavian journal of statistics},
  pages={65--70},
  year={1979},
  publisher={JSTOR}
}

@misc{qwen3,
      title={Qwen3 Technical Report}, 
      author={An Yang and Anfeng Li and Baosong Yang and Beichen Zhang and Binyuan Hui and Bo Zheng and Bowen Yu and Chang Gao and Chengen Huang and Chenxu Lv and Chujie Zheng and Dayiheng Liu and Fan Zhou and Fei Huang and Feng Hu and Hao Ge and Haoran Wei and Huan Lin and Jialong Tang and Jian Yang and Jianhong Tu and Jianwei Zhang and Jianxin Yang and Jiaxi Yang and Jing Zhou and Jingren Zhou and Junyang Lin and Kai Dang and Keqin Bao and Kexin Yang and Le Yu and Lianghao Deng and Mei Li and Mingfeng Xue and Mingze Li and Pei Zhang and Peng Wang and Qin Zhu and Rui Men and Ruize Gao and Shixuan Liu and Shuang Luo and Tianhao Li and Tianyi Tang and Wenbiao Yin and Xingzhang Ren and Xinyu Wang and Xinyu Zhang and Xuancheng Ren and Yang Fan and Yang Su and Yichang Zhang and Yinger Zhang and Yu Wan and Yuqiong Liu and Zekun Wang and Zeyu Cui and Zhenru Zhang and Zhipeng Zhou and Zihan Qiu},
      year={2025},
      eprint={2505.09388},
      archivePrefix={arXiv},
      primaryClass={cs.CL},
      url={https://arxiv.org/abs/2505.09388}, 
}

@misc{openai2025gpt5,
  title={{GPT-5} System Card},
  author={{OpenAI}},
  year={2025},
  howpublished={\url{https://openai.com/index/gpt-5-system-card/}},
  note={Model documentation for the GPT-5 series, including the GPT-5.4 and GPT-5.6 models used in this work}
}

@misc{openai2025codex,
  title={Codex {CLI}: Lightweight coding agent that runs in your terminal},
  author={{OpenAI}},
  year={2025},
  howpublished={\url{https://github.com/openai/codex}}
}

@misc{openpyxl,
  title={openpyxl: A {Python} library to read/write {Excel} 2010 xlsx/xlsm files},
  author={Gazoni, Eric and Clark, Charlie},
  year={2024},
  howpublished={\url{https://openpyxl.readthedocs.io}}
}

@misc{libreoffice,
  title={{LibreOffice}},
  author={{The Document Foundation}},
  year={2024},
  howpublished={\url{https://www.libreoffice.org}}
}

@misc{zhang2026evalresearchbenchaiagentsdesign,
  title={{EvalResearchBench}: Can {AI} Agents Design Their Own Evaluations?},
  author={Yaolun Zhang and Tianyi Xu and Yujie Zhao and Jishen Zhao and Qingyun Wu and Huazheng Wang},
  year={2026},
  eprint={2610.04184},
  archivePrefix={arXiv},
  primaryClass={cs.AI},
  url={https://arxiv.org/abs/2610.04184}
}
\bibliographystyle{iclr2027_conference}

\clearpage
\appendix

\section{Limitations}
\label{app:limitations}

UCoE measures driving capacity on declared families of local edits, so passing the response
contrasts does not guarantee reliable discrimination on natural revisions, on interacting
errors, or on task properties that no family covers. Broader contrasts with independently checked
preference directions would give verifier selection a stronger test. The deployed
verifier also carries a coupled risk: because it diagnoses weaknesses, selects the served
response, and decides which repairs enter the harness state, one mistaken judgment can accept a
harmful repair and store a misleading procedure. Independent checks or rollback rules would
reduce that risk. Finally, open-ended evaluation still relies on model judges, and repeated
semantic scoring makes UCoE more expensive than the simpler feedback baselines, so validation
with diverse evaluators or humans and cheaper verifier execution are the next steps.

\section{Algorithm}
\label{app:full-algorithm}

Algorithm~\ref{alg:ucoe} gives the complete lifecycle. Initialization runs at $t=0$, and every
later checkpoint synthesizes candidates from the diagnostics of completed tasks only. $H$ denotes
the harness state, $E$ the diagnostic history, and $\mathcal R$ the fixed relation inventory
mined from the development pairs $\mathcal P_{\rm dev}$. The main experiments instantiate $H$
as a bounded procedural memory, and the agent-harness study of App.~\ref{app:ssb2-harness}
instantiates it as the context of a coding agent.

\begin{algorithm}[tb]
\caption{Unified Co-Evolution}
\label{alg:ucoe}
\footnotesize
\begin{algorithmic}[1]
\Require Public requirements $R$, optional development response pairs $\mathcal P_{\rm dev}$, actor $F_{\theta_0}$, verifier model, initial harness state $H_{\rm init}$, and task stream $x_{1:T}$
\Require Calibration source reservoir $\mathcal S$ and checkpoint set $\mathcal T$
\Require Candidate budget $L$, contrast families, comparison budgets, margin $\epsilon$, harness-update rules
\Ensure Returned responses $y_{1:T}$, harness state $H$, and verifier-decision history
\State $H\gets H_{\rm init}$, $E\gets\varnothing$, and $\cV\gets\cV_\bot$
\State $\mathcal R\gets\operatorname{MineRelations}(\mathcal P_{\rm dev})$ if supplied; otherwise $\mathcal R\gets\varnothing$
\For{$t=0,\ldots,T$}
    \If{$t>0$}
        \State $y_t^0\gets F_{\theta_0}(x_t;H)$ through the configured harness interface
        \State $(y_t,y_t^R,a_t,c_t,w_t,u_t)\gets(y_t^0,\bot,0,\bot,\bot,\bot)$
        \If{$\cV\neq\cV_\bot$}
            \State $(c_t,w_t,u_t)\gets D_{\cV}(x_t,y_t^0)$
            \State $y_t^R\gets F_{\theta_0}(x_t,y_t^0,c_t,w_t,u_t;H)$
            \If{$r_{\cV}(x_t,y_t^R)>r_{\cV}(x_t,y_t^0)+\epsilon$}
                \State $(y_t,a_t)\gets(y_t^R,1)$
            \EndIf
        \EndIf
        \State Return $y_t$; subsequent updates derived from task $t$ affect only later tasks
        \If{$\cV\neq\cV_\bot$}
            \State $H\gets\operatorname{UpdateHarness}(H,x_t,y_t^0,y_t^R,c_t,w_t,a_t)$
            \Comment{new state content only if $a_t=1$}
        \EndIf
        \State Append the completed task and available verifier diagnostics to $E$
    \EndIf
    \If{$t=0$ or $t\in\mathcal T$}
        \State $d\gets\operatorname{Summarize}(E)$, with $d=\varnothing$ at initialization
        \State $\mathcal K\gets\operatorname{GenerateVerifiers}(R,\mathcal R,d,L)$ using the verifier model
        \State Include $\cV$ in $\mathcal K$ if $\cV\neq\cV_\bot$ and discard non-executable candidates
        \State Freeze all candidate programs and the scoring, diagnosis, and selection conventions
        \If{$\mathcal K\neq\varnothing$}
            \State $(B_{\rm sel},B_{\rm val})\gets\operatorname{Contrasts}(\mathcal S,\mathcal R)$ \Comment{source-disjoint within checkpoint}
            \State $\widehat\cV\in\arg\max_{\cV'\in\mathcal K}\operatorname{round}_4(\min_j\widehat J_{{\rm sel},j}(h_{\cV'}))$
            \State Fix $\widehat\cV$ before evaluating any required held-out comparisons
            \If{$\cV=\cV_\bot$}
                \State $b\gets1$ \Comment{adopt the initial selected candidate}
            \Else
                \State $b\gets\operatorname{UpdateCriteria}(\widehat\cV,\cV,B_{\rm sel},B_{\rm val})$ \Comment{Eq.~\eqref{eq:empirical-admission}}
            \EndIf
            \If{$b=1$}
                \State $\cV\gets\widehat\cV$ \Comment{affects only tasks after $t$}
            \EndIf
        \EndIf
        \State Record the selection, held-out comparison, and adopt-or-retain decision
    \EndIf
\EndFor
\State Finalize all responses, verifier decisions, and harness states before final evaluation
\State \Return $(y_{1:T},H,E)$ and the verifier-decision history
\end{algorithmic}
\end{algorithm}

The update criteria in $\operatorname{UpdateCriteria}$ require a strict gain in worst-family
accuracy on the selection split, more held-out contrasts won than lost against the current
verifier, and no regression beyond the tolerance on any existing family, and any failed or
undefined comparison retains the current verifier. Equation~\eqref{eq:empirical-admission} states
them exactly. While no verifier is deployed, UCoE serves the draft and leaves the harness state
unchanged, but later checkpoints still run. Calibration data and final-evaluator feedback never
enter deployment. The reference procedure draws fresh comparison sources at each checkpoint.

\section{Theoretical analysis}
\label{app:additional-theory}

We analyze the selectors induced by UCoE's verifier programs on fixed response-pair
distributions. The results explain the normalization of driving capacity, the worst-family
selection objective, and the guarantees of the empirical update rule. Hats denote sample
estimates, and population bounds require the sampling assumptions stated below.

\subsection{Pairwise discrimination and selection gain}
\label{app:pairwise-theory}

For a task $x$, let $y_c$ be the proposed response and $y_r$ the reference response.
With a fixed margin $\epsilon\geq0$, the verifier induces
\begin{equation}
h_{\cV,\epsilon}(x;y_c,y_r)=
\begin{cases}
y_c,&r_{\cV}(x,y_c)>r_{\cV}(x,y_r)+\epsilon,\\
y_r,&\text{otherwise}.
\end{cases}
\label{eq:ordered-selector}
\end{equation}
An undefined comparison returns $y_r$. Contrasts use $(y_c,y_r)=(y^+,y^-)$,
whereas deployment uses $(y_c,y_r)=(y^R,y^0)$. We write $h_{\cV}$ for this selector
at the fixed margin. For family $j$, define
$p_j(h)=\Pr(h\text{ selects }y^+\mid j)$ and $J_j(h)=2p_j(h)-1$.
Uniform random choice gives $J_j=0$. A verifier that always assigns equal scores has
$J_j=-1$, since every tie fails the strict-margin test.

For the covariance interpretation \citep{price1970}, write $(Y_1,Y_2)=(y^+,y^-)$ and let
$S_h\in\{1,2\}$ be the selected position. Let $U_i\in\{0,1\}$ indicate whether $Y_i$
satisfies the targeted property. A valid contrast must satisfy $(U_1,U_2)=(1,0)$.
This is a semantic assumption about the edited property, which eligibility checks alone
cannot establish. Draw
$I\sim\operatorname{Unif}\{1,2\}$ independently of the contrast, its family, and all scoring
and selection randomness. Set $G_h=\mathbf 1[S_h=I]$ and $U=U_I$.
The auxiliary index does not change the response order or the selector.

\Needspace{8\baselineskip}
\begin{proposition}[Pairwise selection covariance]
\label{cor:pairwise-capacity}
For valid targeted contrasts and the independent index defined above,
\begin{equation}
\operatorname{Cov}(G_h,U\mid j)=\frac14J_j(h).
\label{eq:contrast-covariance}
\end{equation}
Thus the signed driving capacity $\widetilde{\mathcal C}_{\cB}$ in
Eq.~\eqref{eq:constructed-driving-capacity} is the minimum selection covariance over
the tested families. Its nonnegative part is $\mathcal C_{\cB}$.
\end{proposition}
\begin{proof}
Fix a family $j$. Since $(U_1,U_2)=(1,0)$, we have $U=1$ exactly when $I=1$.
The index $I$ is uniform and independent of $S_h$ conditional on $j$, so
\begin{align*}
\mathbb E[U\mid j]
&=\Pr(I=1\mid j)=\frac12,\\
\mathbb E[G_h\mid j]
&=\sum_{s=1}^{2}\Pr(I=s,S_h=s\mid j)
=\frac12\sum_{s=1}^{2}\Pr(S_h=s\mid j)=\frac12,\\
\mathbb E[G_hU\mid j]
&=\Pr(I=1,S_h=1\mid j)
=\frac12\Pr(S_h=1\mid j)=\frac12p_j(h).
\end{align*}
The last line uses $G_hU=1$ exactly when the sampled position and the selected position
are both 1. Substituting these expectations into the definition of covariance gives
\begin{align*}
\operatorname{Cov}(G_h,U\mid j)
&=\mathbb E[G_hU\mid j]-\mathbb E[G_h\mid j]\,\mathbb E[U\mid j]\\
&=\frac{p_j(h)}{2}-\frac12\cdot\frac12
=\frac{2p_j(h)-1}{4}=\frac14J_j(h).
\end{align*}
For $h=h_{\cV}$, taking the minimum over the tested families yields
\[
\min_j\operatorname{Cov}(G_h,U\mid j)
=\frac14\min_jJ_j(h)=\widetilde{\mathcal C}_{\cB}(\cV).
\]
Clipping this minimum at zero gives $\mathcal C_{\cB}(\cV)$.
\end{proof}

The auxiliary retention rates are
$\Pr(G_h=1\mid U=1,j)=p_j(h)$ and
$\Pr(G_h=1\mid U=0,j)=1-p_j(h)$, whose difference is Youden's index
$J_j(h)$ \citep{youden1950}. Because contrasts always place the preferred response in the
proposed position and count ties and reversed rankings alike as failures, these rates measure
discrimination on the tested property, not the acceptance of natural revisions.

To describe selection gain on general pairs, fix a distribution of
$(x,Y_1,Y_2,Z_1,Z_2)$, where $Z_i\in\{0,1\}$ is an evaluation-only adequacy label
for $Y_i$. Each selector chooses a position $S_h\in\{1,2\}$ on the same pairs.
Let $M=\mathbf 1[Z_1\ne Z_2]$, $q=\Pr(M=1)$, and, for $q>0$, define
$J_{\rm mix}(h)=2\Pr(Z_{S_h}=1\mid M=1)-1$.

\begin{proposition}[Selection gain on fixed response pairs]
\label{cor:opportunity-discrimination}
For any two selectors $h,h'$ on this pair distribution with $q>0$,
\begin{equation}
\mathbb E[Z_{S_{h'}}-Z_{S_h}]
=\frac q2\{J_{\rm mix}(h')-J_{\rm mix}(h)\}.
\label{eq:opportunity-discrimination}
\end{equation}
If $q=0$, the expected adequacy difference is zero.
\end{proposition}
\begin{proof}
When $M=0$, we have $Z_1=Z_2$, so $Z_{S_{h'}}-Z_{S_h}=0$ regardless of which
positions the selectors choose. Thus
$Z_{S_{h'}}-Z_{S_h}=M(Z_{S_{h'}}-Z_{S_h})$. If $q=0$, this difference is zero
almost surely, proving the second claim.

For $q>0$, binary adequacy and the definition of $J_{\rm mix}$ give
\[
\mathbb E[Z_{S_h}\mid M=1]
=\Pr(Z_{S_h}=1\mid M=1)=\frac{1+J_{\rm mix}(h)}{2},
\]
and the same equality holds for $h'$. Conditioning on $M=1$ and substituting both
expressions, we obtain
\begin{align*}
\mathbb E[Z_{S_{h'}}-Z_{S_h}]
&=\mathbb E\!\left[M(Z_{S_{h'}}-Z_{S_h})\right]\\
&=q\,\mathbb E[Z_{S_{h'}}-Z_{S_h}\mid M=1]\\
&=q\left\{\mathbb E[Z_{S_{h'}}\mid M=1]-\mathbb E[Z_{S_h}\mid M=1]\right\}\\
&=q\left\{\frac{1+J_{\rm mix}(h')}{2}-\frac{1+J_{\rm mix}(h)}{2}\right\}\\
&=\frac q2\left\{J_{\rm mix}(h')-J_{\rm mix}(h)\right\}.
\end{align*}
\end{proof}

The factor $q$ measures how often selection can matter, and the change in $J_{\rm mix}$
measures how well the new selector uses those opportunities. The identity is about binary
adequacy on a fixed pair distribution, so carrying it to contrast accuracy or to the deployed
stream, where diagnosis, proposals, and harness state all move the pairs, needs further
assumptions.

\subsection{Worst-family selection over a finite verifier pool}
\label{app:finite-pool-theory}

Let $m\geq1$ be the number of contrast families. For population statements, $\cB$ denotes
their fixed distributions, and probabilities include any scoring randomness.
We write $\widetilde{\mathcal C}_{\cB}(h)=\frac14\min_jJ_j(h)$ and
$\mathcal C_{\cB}(h)=\max\{0,\widetilde{\mathcal C}_{\cB}(h)\}$ for the capacities
of an induced selector. A family mixture has weights
$w\in\Delta_m=\{w\in[0,1]^m:\sum_jw_j=1\}$ and success probability
$p_w(h)=\sum_jw_jp_j(h)$.

\begin{theorem}[Robust interpretation of driving capacity]
\label{thm:robust-selection}
For every selector $h$,
\begin{equation}
\inf_{w\in\Delta_m}p_w(h)
=\min_jp_j(h)
=\frac12+2\widetilde{\mathcal C}_{\cB}(h).
\label{eq:capacity-mixture}
\end{equation}
Maximizing signed driving capacity over a nonempty finite pool therefore maximizes the
worst-case success probability over all mixtures of the tested families.
\end{theorem}
\begin{proof}
Fix a selector $h$ and choose
$j^\star\in\operatorname*{arg\,min}_{1\leq j\leq m}p_j(h)$.
For any $w\in\Delta_m$, the weights are nonnegative and sum to one, so
\begin{align*}
p_w(h)
&=\sum_{j=1}^{m}w_jp_j(h)
\geq\sum_{j=1}^{m}w_jp_{j^\star}(h)\\
&=p_{j^\star}(h)\sum_{j=1}^{m}w_j
=p_{j^\star}(h).
\end{align*}
Thus every mixture has success probability at least $\min_jp_j(h)$.
This bound is attained by the weights $w^\star_{j^\star}=1$ and
$w^\star_j=0$ for $j\neq j^\star$, which belong to $\Delta_m$ and satisfy
\[
p_{w^\star}(h)=p_{j^\star}(h)=\min_jp_j(h).
\]
Combining the lower bound with this attaining mixture gives
$\inf_{w\in\Delta_m}p_w(h)=\min_jp_j(h)$.

Using $J_j(h)=2p_j(h)-1$ and the definition of signed driving capacity,
\begin{align*}
\min_jp_j(h)
&=\min_j\frac{1+J_j(h)}{2}\\
&=\frac12+\frac12\min_jJ_j(h)\\
&=\frac12+2\left(\frac14\min_jJ_j(h)\right)
=\frac12+2\widetilde{\mathcal C}_{\cB}(h).
\end{align*}
The last expression is strictly increasing in signed driving capacity. Therefore, over any
nonempty finite selector pool, signed driving capacity and worst-case mixture success
probability have the same maximizers.
\end{proof}

This is the group maximin structure used in distributionally robust optimization
\citep{sagawa2020dro}: it covers changes in family weights while each family's distribution
stays fixed. Maximizing the signed capacity also maximizes its nonnegative part while still
separating candidates that clipping would assign the same zero value.

\paragraph{Sampling assumptions.}
Fix a nonempty finite selector pool $\cH$ and family sample sizes $n_j\geq1$ before
drawing selection sources. Each source used in family $j$ contributes one comparison to that family. For a fixed
$h$ and $j$, assume the outcomes
\[
X_{ijh}=2\mathbf 1[h\text{ selects }y^+\text{ on source }i\text{ in family }j]-1,
\qquad i=1,\ldots,n_j,
\]
are independent across sources. Define
$\widehat J_j(h)=n_j^{-1}\sum_iX_{ijh}$ and assume
$n_j^{-1}\sum_i\mathbb E[X_{ijh}]=J_j(h)$. This allows fixed stratification with the
corresponding population mixture. Every selector is evaluated on all declared comparisons
with a predeclared outcome for missing or invalid scores. Candidate-dependent removal of
comparisons is excluded. Different families may share sources, and different selectors may
share scores. For history-generated pools, these assumptions hold conditional on the history
and synthesis information fixed before sampling.

\begin{proposition}[Finite-pool maximin selection]
\label{prop:finite-selection}
Under these assumptions, for $\delta\in(0,1)$ let
\begin{equation}
\rho_\delta=\max_j\sqrt{\frac{2}{n_j}\log\frac{2|\cH|m}{\delta}}.
\label{eq:finite-epsilon}
\end{equation}
With probability at least $1-\delta$,
\begin{equation}
\sup_{h\in\cH,\,1\leq j\leq m}|\widehat J_j(h)-J_j(h)|\leq\rho_\delta.
\label{eq:uniform-j-event}
\end{equation}
On this event, any $\widehat h\in\cH$ whose empirical worst-family score is within
$\tau\geq0$ of the maximum satisfies
\begin{equation}
\min_jJ_j(\widehat h)
\geq
\underbrace{\max_{h\in\cH}\min_jJ_j(h)}_{\substack{\text{best population}\\\text{score in the pool}}}
-\underbrace{2\rho_\delta}_{\substack{\text{sampling}\\\text{error bound}}}
-\underbrace{\tau}_{\substack{\text{optimization}\\\text{error}}}.
\label{eq:finite-robust-regret}
\end{equation}
Its signed or nonnegative capacity is at most $\rho_\delta/2+\tau/4$ below the best
corresponding capacity in the pool. Exact maximization has $\tau=0$.
Maximizing scores rounded to four decimals gives $\tau\leq10^{-4}$.
\end{proposition}
\begin{proof}
Fix a selector $h\in\cH$ and a family $j$. The sampling assumptions give independence
across sources and $\mathbb E[\widehat J_j(h)]=J_j(h)$. Since
$X_{ijh}\in[-1,1]$, Hoeffding's inequality \citep{hoeffding1963} yields, for $u>0$,
\[
\Pr\!\left(|\widehat J_j(h)-J_j(h)|>u\right)
\leq 2\exp\!\left(-\frac{2n_j^2u^2}{\sum_{i=1}^{n_j}(1-(-1))^2}\right)
=2\exp\!\left(-\frac{n_ju^2}{2}\right).
\]
By the definition of $\rho_\delta$, each family satisfies
$n_j\rho_\delta^2/2\geq\log(2|\cH|m/\delta)$. Substituting $u=\rho_\delta$ gives
\[
\Pr\!\left(|\widehat J_j(h)-J_j(h)|>\rho_\delta\right)
\leq 2\exp\!\left(-\log\frac{2|\cH|m}{\delta}\right)
=\frac{\delta}{|\cH|m}.
\]
A union bound over all $|\cH|m$ selector--family pairs then gives
\begin{align*}
\Pr\!\left(\sup_{h\in\cH,\,1\leq j\leq m}
|\widehat J_j(h)-J_j(h)|>\rho_\delta\right)
&\leq\sum_{h\in\cH}\sum_{j=1}^{m}
\Pr\!\left(|\widehat J_j(h)-J_j(h)|>\rho_\delta\right)\\
&\leq |\cH|m\,\frac{\delta}{|\cH|m}=\delta.
\end{align*}
This proves Eq.~\eqref{eq:uniform-j-event} without requiring independence between
different selectors or families.

Work on this uniform event. For every $h$ and $j$,
$\widehat J_j(h)-\rho_\delta\leq J_j(h)\leq\widehat J_j(h)+\rho_\delta$.
Taking minima over $j$ gives
\[
\min_j\widehat J_j(h)-\rho_\delta
\leq\min_jJ_j(h)
\leq\min_j\widehat J_j(h)+\rho_\delta.
\]
Thus each empirical minimum differs from its population minimum by at most $\rho_\delta$,
even if the minimizing families differ. Choose
$h^\star\in\operatorname*{arg\,max}_{h\in\cH}\min_jJ_j(h)$, which exists because
$\cH$ is nonempty and finite. The assumed empirical optimization accuracy means
\[
\min_j\widehat J_j(\widehat h)
\geq\max_{h\in\cH}\min_j\widehat J_j(h)-\tau
\geq\min_j\widehat J_j(h^\star)-\tau.
\]
Combining these inequalities yields
\begin{align*}
\min_jJ_j(\widehat h)
&\geq\min_j\widehat J_j(\widehat h)-\rho_\delta\\
&\geq\min_j\widehat J_j(h^\star)-\rho_\delta-\tau\\
&\geq\min_jJ_j(h^\star)-2\rho_\delta-\tau\\
&=\max_{h\in\cH}\min_jJ_j(h)-2\rho_\delta-\tau,
\end{align*}
which proves Eq.~\eqref{eq:finite-robust-regret}.

Dividing by four and using the definition of signed driving capacity gives
\[
0\leq\widetilde{\mathcal C}_{\cB}(h^\star)
-\widetilde{\mathcal C}_{\cB}(\widehat h)
\leq\frac{2\rho_\delta+\tau}{4}
=\frac{\rho_\delta}{2}+\frac{\tau}{4}.
\]
For any $a\geq b$, the positive-part map satisfies
$0\leq\max\{0,a\}-\max\{0,b\}\leq a-b$.
Apply this with $a=\widetilde{\mathcal C}_{\cB}(h^\star)$ and
$b=\widetilde{\mathcal C}_{\cB}(\widehat h)$. Since $h^\star$ also maximizes
the nonnegative capacity, the same upper bound holds for its gap to the best in the pool.

Finally, write $\widehat F(h)=\min_j\widehat J_j(h)$ and
$\overline F(h)=\operatorname{round}_4(\widehat F(h))$.
Rounding to four decimals gives
$|\overline F(h)-\widehat F(h)|\leq\eta$, where $\eta=5\times10^{-5}$.
If $\widehat h$ maximizes $\overline F$, then for every $h\in\cH$,
\begin{align*}
\widehat F(\widehat h)
&\geq\overline F(\widehat h)-\eta\\
&\geq\overline F(h)-\eta\\
&\geq\widehat F(h)-2\eta.
\end{align*}
Taking the maximum over $h$ shows that the empirical optimization loss is at most
$2\eta=10^{-4}$, proving the stated rounding bound.
\end{proof}

\paragraph{Interpretation.}
Equation~\eqref{eq:finite-robust-regret} bounds the selected verifier's population
worst-family score relative to the best candidate in the fixed pool. The sampling term
$2\rho_\delta$ allows the selected candidate's empirical score to overestimate its true
score by $\rho_\delta$ and the population optimum's empirical score to underestimate its
true score by another $\rho_\delta$. The term $\tau$ accounts for approximate empirical
maximization, including rounding. Writing $n_{\min}=\min_jn_j$, we have
$\rho_\delta=\sqrt{(2/n_{\min})\log(2|\cH|m/\delta)}$, so the least-sampled family
controls this uniform bound. The guarantee is stated on the fixed contrast distributions.

\paragraph{Comparison with the current verifier.}
Equation~\eqref{eq:finite-robust-regret} compares the selected candidate with the
population optimum in the pool. At an update checkpoint, we can instead compare a
candidate $h'$ with the current selector $h_k$, assuming both belong to $\cH$.
Let their observed, unrounded gain in empirical worst-family score be
\[
g=\min_j\widehat J_j(h')-\min_j\widehat J_j(h_k).
\]
On the uniform event of Eq.~\eqref{eq:uniform-j-event}, the candidate's population
score is at least its empirical score minus $\rho_\delta$, whereas the current
selector's population score is at most its empirical score plus $\rho_\delta$.
Subtracting these two bounds gives
\begin{align*}
\min_jJ_j(h')-\min_jJ_j(h_k)
&\geq\left(\min_j\widehat J_j(h')-\rho_\delta\right)
-\left(\min_j\widehat J_j(h_k)+\rho_\delta\right)\\
&=g-2\rho_\delta.
\end{align*}
Dividing by four converts the score difference into a signed-capacity difference:
\begin{equation}
\widetilde{\mathcal C}_{\cB}(h')-\widetilde{\mathcal C}_{\cB}(h_k)
\geq\frac{g-2\rho_\delta}{4}.
\label{eq:capacity-selection-gap}
\end{equation}
Thus $g$ is the improvement measured on the selection sample, and $2\rho_\delta$
accounts for estimation error on both sides of the comparison. No optimization term
$\tau$ is needed here because $g$ compares these two selectors directly, regardless
of how the candidate was chosen within the pool. If $g>2\rho_\delta$, the bound is
strictly positive and certifies an increase in population signed driving capacity on
the uniform event. When $0<g\leq2\rho_\delta$ the sample shows a gain that the bound
cannot certify, and small family budgets make it uninformative. This threshold is a
sufficient condition for the guarantee only. UCoE adopts by the empirical selection gain
and the held-out comparisons of the next subsection.

\subsection{Evolution under the empirical verifier-update rule}
\label{app:empirical-update}

At a checkpoint, UCoE includes the current verifier in the selection pool, chooses a
candidate, and fixes it before evaluating held-out contrasts. Let $h$ and $h'$ be the
current and selected candidate selectors. The family inventory stays fixed.
We first establish what adoption guarantees on the observed comparison sets, then
bound the population capacity change under additional sampling assumptions.
The empirical results below assume common comparison sets with finite scores for both
selectors, so all paired counts and family accuracies use the same denominators.

Let $n_j\geq1$ now denote the held-out count in family $j$, with $N=\sum_jn_j$.
For held-out comparison $i$ in family $j$, define
$C_{ij}(h)=\mathbf 1[h\text{ selects }y^+]$,
$\widehat J_{{\rm val},j}(h)=2n_j^{-1}\sum_iC_{ij}(h)-1$, and
$\widehat p_{\rm pool}(h)=N^{-1}\sum_{j,i}C_{ij}(h)$.
Thus $\widehat p_{\rm pool}$ weights each contrast equally, giving family $j$ weight
$n_j/N$, and it is an equal-family average only when the family counts are equal.
The selection-split estimate $\widehat J_{{\rm sel},j}$ is defined in the same way
using that split's comparisons. Write
\begin{align}
W&=\sum_{j,i}\mathbf 1[C_{ij}(h')=1,C_{ij}(h)=0],
&L&=\sum_{j,i}\mathbf 1[C_{ij}(h')=0,C_{ij}(h)=1],\notag\\
\widehat\Delta J_j&=\widehat J_{{\rm val},j}(h')-\widehat J_{{\rm val},j}(h),
&\widehat\Gamma_{\rm sel}&=\min_j\widehat J_{{\rm sel},j}(h')
-\min_j\widehat J_{{\rm sel},j}(h).
\label{eq:empirical-update-statistics}
\end{align}
Thus $W$ and $L$ count candidate-only and current-only successes. Let
$\overline\Gamma_{\rm sel}$ be the difference between the two selection minima after
each is rounded to four decimals. For these complete comparisons, the update rule is
\begin{equation}
h_{k+1}=
\begin{cases}
h',&\underbrace{\overline\Gamma_{\rm sel}>0}_{\text{selection gain}},\quad
    \underbrace{W>L}_{\text{held-out gain}},\quad
    \underbrace{\min_j\widehat\Delta J_j\geq-\kappa}_{\text{bounded family regression}},\\
h,&\text{otherwise},
\end{cases}
\qquad \kappa=0.10.
\label{eq:empirical-admission}
\end{equation}
The first condition requires a strict increase in the selection split's worst-family
score. The held-out conditions require more candidate-only than current-only successes
and limit the decrease in every family's centered accuracy to $\kappa$.
A missing required family check retains $h$. Initialization adopts a maximin winner
without this paired test, since there is no current verifier.

For $s\in\{{\rm sel},{\rm val}\}$, define
$\widehat{\widetilde{\mathcal C}}_s(h)=\frac14\min_j\widehat J_{s,j}(h)$ and
$\widehat{\mathcal C}_s(h)=\max\{0,\widehat{\widetilde{\mathcal C}}_s(h)\}$.

\begin{proposition}[Empirical gain and bounded family regression]
\label{prop:empirical-evolution}
Every candidate adopted under Eq.~\eqref{eq:empirical-admission} satisfies
\begin{align}
\widehat{\widetilde{\mathcal C}}_{\rm sel}(h')
-\widehat{\widetilde{\mathcal C}}_{\rm sel}(h)
&=\frac14\widehat\Gamma_{\rm sel}>0,
\label{eq:empirical-selection-capacity}\\
\widehat p_{\rm pool}(h')-\widehat p_{\rm pool}(h)
&=\frac{W-L}{N}>0,
\label{eq:discordant-gain}\\
\widehat{\widetilde{\mathcal C}}_{\rm val}(h')
-\widehat{\widetilde{\mathcal C}}_{\rm val}(h)
&\geq-\frac{\kappa}{4}.
\label{eq:empirical-capacity-regression}
\end{align}
The last bound also holds for the nonnegative held-out capacity.
\end{proposition}
\begin{proof}
Suppose $h'$ is adopted, so all three conditions in Eq.~\eqref{eq:empirical-admission}
hold. Rounding to four decimals is nondecreasing. If the candidate's unrounded
selection minimum were no larger than the current selector's, its rounded minimum
could not be larger either. Thus $\overline\Gamma_{\rm sel}>0$ implies
$\widehat\Gamma_{\rm sel}>0$. Substituting the definition of empirical signed capacity,
\begin{align*}
\widehat{\widetilde{\mathcal C}}_{\rm sel}(h')
-\widehat{\widetilde{\mathcal C}}_{\rm sel}(h)
&=\frac14\left\{\min_j\widehat J_{{\rm sel},j}(h')
-\min_j\widehat J_{{\rm sel},j}(h)\right\}\\
&=\frac14\widehat\Gamma_{\rm sel}>0,
\end{align*}
which proves Eq.~\eqref{eq:empirical-selection-capacity}.

For the pooled success rate, each difference $C_{ij}(h')-C_{ij}(h)$ is $+1$ on
a candidate-only success, $-1$ on a current-only success, and $0$ when the selectors
both succeed or both fail. Summing over the common held-out comparisons therefore gives
\begin{align*}
\widehat p_{\rm pool}(h')-\widehat p_{\rm pool}(h)
&=\frac1N\sum_{j=1}^{m}\sum_{i=1}^{n_j}
\{C_{ij}(h')-C_{ij}(h)\}\\
&=\frac{W-L}{N}>0,
\end{align*}
where the strict inequality follows from $W>L$. This proves Eq.~\eqref{eq:discordant-gain}.

For the held-out capacity, the regression condition gives, for every family $j$,
\[
\widehat J_{{\rm val},j}(h')
\geq\widehat J_{{\rm val},j}(h)-\kappa
\geq\min_{\ell}\widehat J_{{\rm val},\ell}(h)-\kappa.
\]
The last expression is independent of $j$, so taking the minimum on the left yields
\[
\min_j\widehat J_{{\rm val},j}(h')
\geq\min_j\widehat J_{{\rm val},j}(h)-\kappa.
\]
Subtracting the current minimum and dividing by four proves
Eq.~\eqref{eq:empirical-capacity-regression}, even if the weakest family changes.
To obtain the nonnegative-capacity bound, write
$a=\widehat{\widetilde{\mathcal C}}_{\rm val}(h)$ and
$b=\widehat{\widetilde{\mathcal C}}_{\rm val}(h')$.
We have $b\geq a-\kappa/4$, and hence
\[
\max\{0,b\}\geq\max\{0,a-\kappa/4\}
\geq\max\{0,a\}-\kappa/4.
\]
The first inequality uses monotonicity of the positive-part map, and the second uses
that reducing its argument by $\kappa/4$ reduces its value by at most $\kappa/4$.
Subtracting $\max\{0,a\}$ proves the remaining claim.
\end{proof}

\paragraph{Interpretation.}
The three conclusions concern different quantities: signed capacity strictly increases
on the selection split, pooled success rate strictly increases on the held-out split,
and held-out capacity may decrease within the allowed tolerance. Since $J_j=2p_j-1$,
$\kappa=0.10$ permits at most a five-percentage-point decrease in any family's
held-out success rate and a $0.025$ decrease in held-out capacity.
For example, on two equally sized held-out families, success rates changing from
$(0.80,0.60)$ to $(0.90,0.56)$ raise pooled success from $0.70$ to $0.73$ while
lowering signed capacity from $0.05$ to $0.03$. Both held-out conditions pass,
although adoption still requires the separate selection condition.
These are exact statements on the observed comparison sets and need no independence
assumption.

\paragraph{Population interpretation.}
Condition on the history, synthesis, and selection information so that $h,h'$ are fixed.
Suppose validation uses fresh sources independent of that information, with fixed family
counts $n_j$. Within each family, assume the paired outcomes
$(C_{ij}(h),C_{ij}(h'))$ are independent across sources, with one comparison per source.
Different families may share sources. All probabilities and expectations below are
conditional on this pre-validation information.

Let $\Delta J_j=J_j(h')-J_j(h)$ be the population change in family $j$, and define
\[
Y_{ij}=2\{C_{ij}(h')-C_{ij}(h)\}\in\{-2,0,2\},
\qquad
\widehat\Delta J_j=\frac1{n_j}\sum_{i=1}^{n_j}Y_{ij}.
\]
Assume $n_j^{-1}\sum_i\mathbb E[Y_{ij}]=\Delta J_j$, so the empirical change
has the stated population mean. Because $Y_{ij}\in[-2,2]$ and the paired outcomes
are independent across sources, one-sided Hoeffding gives, for any $u>0$,
\begin{align*}
\Pr\!\left(\Delta J_j<\widehat\Delta J_j-u\right)
&=\Pr\!\left(\widehat\Delta J_j-\Delta J_j>u\right)\\
&\leq\exp\!\left(-\frac{2n_j^2u^2}{\sum_{i=1}^{n_j}(2-(-2))^2}\right)
=\exp\!\left(-\frac{n_ju^2}{8}\right).
\end{align*}
For $\delta\in(0,1)$, choose
\[
b_j=\sqrt{\frac8{n_j}\log\frac m\delta}.
\]
Substituting $u=b_j$ bounds the failure probability for family $j$ by
$\exp(-n_jb_j^2/8)=\delta/m$. A union bound over the $m$ families gives
\begin{align*}
\Pr\!\left(\exists j:\Delta J_j<\widehat\Delta J_j-b_j\right)
&\leq\sum_{j=1}^{m}
\Pr\!\left(\Delta J_j<\widehat\Delta J_j-b_j\right)\\
&\leq\sum_{j=1}^{m}\frac\delta m=\delta.
\end{align*}
Thus, with conditional probability at least $1-\delta$, every population family
change is at least its observed change minus $b_j$. Only one fixed selector pair is
compared on validation, so this union bound ranges over families, not the candidate pool.

On this event, for every $j$,
\[
J_j(h')=J_j(h)+\Delta J_j
\geq\min_{\ell}J_{\ell}(h)+\min_{\ell}\Delta J_{\ell}.
\]
Taking the minimum over $j$ on the left and rearranging yields
\[
\min_jJ_j(h')-\min_jJ_j(h)
\geq\min_j\Delta J_j
\geq\min_j(\widehat\Delta J_j-b_j).
\]
The first inequality need not be an equality, since the weakest family can change
between selectors. Dividing by four gives the capacity bound, so on the same event
\begin{align}
\Delta J_j&\geq\widehat\Delta J_j-b_j
\quad\text{for every }j,\notag\\
\widetilde{\mathcal C}_{\cB}(h')-\widetilde{\mathcal C}_{\cB}(h)
&\geq\frac14\min_j(\widehat\Delta J_j-b_j).
\label{eq:paired-update-capacity}
\end{align}
If the candidate is adopted, $\widehat\Delta J_j\geq-\kappa$ for every $j$, hence
\[
\min_j(\widehat\Delta J_j-b_j)
\geq\min_j(-\kappa-b_j)=-\kappa-\max_jb_j.
\]
Substituting this into Eq.~\eqref{eq:paired-update-capacity} gives
\[
\widetilde{\mathcal C}_{\cB}(h')-\widetilde{\mathcal C}_{\cB}(h)
\geq
-\underbrace{\frac\kappa4}_{\substack{\text{allowed empirical}\\\text{capacity regression}}}
-\underbrace{\frac{\max_jb_j}{4}}_{\substack{\text{sampling}\\\text{error bound}}}.
\]
The bound limits population regression on this event without certifying a gain, and it
holds for one checkpoint conditional on pre-validation information. The update criteria
use only the empirical quantities and never these confidence bounds.

\section{Experimental settings}
\label{app:experimental-protocols}

\subsection{Benchmarks and dataset splits}
\label{app:benchmarks}

Table~\ref{tab:benchmark-statistics} lists every benchmark with its task type, domain, and
number of evaluated tasks. The six open-ended benchmarks cover general assistance, story and
domain writing, multi-turn instruction following, and response refinement, and the verifiable
benchmarks add instruction compliance, competition mathematics, code generation, and the
spreadsheet stream of the agent-harness study. Every experiment fixes the task order before
deployment, builds response contrasts from calibration sources disjoint from deployment, and
never exposes benchmark answers or final-evaluator feedback to an adaptive method. 

\begin{table}[tb!]
\centering
\small
\setlength{\tabcolsep}{4pt}
\caption{Benchmarks, task types, domains, and evaluated tasks.}
\label{tab:benchmark-statistics}
\begin{tabularx}{\linewidth}{@{}llXr@{}}
\toprule
Benchmark & Task type & Domain & Evaluated tasks \\
\midrule
Arena-Hard & Open-ended & Challenging user requests to a general assistant & 500 \\
WritingPrompts & Open-ended & Story generation from writing prompts & 1,024 \\
WildBench & Open-ended & Real-user requests to a general assistant & 1,024 \\
WritingBench & Open-ended & Writing across diverse professional and creative domains & 500 \\
MultiChallenge & Open-ended & Multi-turn conversations with instruction constraints & 128 \\
RefineBench & Open-ended & Free-form refinement of an existing response & 278 \\
IFBench & Verifiable & Verifiable instruction constraints & 150 \\
Hard-MATH & Verifiable & Competition mathematics from MATH-500 and AIME 2022 to 2025 & 382 \\
BigCodeBench-Hard & Verifiable & Code generation checked by execution, Instruct split & 148 \\
SpreadsheetBench-2 & Verifiable & Business spreadsheet workflows in the Codex agent harness & 75 \\
\bottomrule
\end{tabularx}
\end{table}

\subsection{Harness-state instantiation}
\label{app:protocol-settings}

In the main experiments the harness state is a bounded procedural memory that sits between the
actor and its prompt. Each adaptive method owns an isolated instance that starts empty and
follows the same rules for capacity, retrieval, distillation, deduplication, and eviction, so
methods differ only in the feedback that decides what enters it. An entry is a procedure with a
trigger of at most 10 words and an action of at most 14 words, linked to the criterion and
quality dimension that motivated it. When a method accepts a repair, a distillation call turns
the draft, the accepted revision, and the feedback into one such procedure and strips
task-specific entities and answer content, so a task can add at most one procedure, and only
after its response has been returned. A rejected revision updates usage statistics only. Before
each draft and revision, retrieval ranks stored procedures by criterion match, then by dimension
match, then by their record of accepted and rejected uses, and returns at most two procedures
within 64 words, which the actor sees as general guidelines in its prompt. Near duplicates are
removed at a normalized token-set Jaccard overlap of 0.75, and when the global capacity of 12
procedures or the per-dimension capacity of 2 is exceeded, the procedure with the weakest
acceptance record is evicted, with ties broken by age. Table~\ref{tab:hyperparameters} lists
these capacities. 

\subsection{Baseline implementations}
\label{app:baselines}

We compare feedback mechanisms inside one fixed-model, one-revision lifecycle rather than
reproducing the parameter-training pipelines of RaR, DREAM, or SpyRL.
Table~\ref{tab:baseline-adaptations} gives each method's feedback interface and revision
decision. In every revision-enabled baseline, the method scores the draft with its own signal,
produces one repair from its own feedback, rescores the revision with the same signal, and
serves the revision only when its normalized reward rises by more than $0.01$. All methods share the
revision prompt, which provides the task, the draft, the method's native feedback, and the
retrieved procedures, and asks the actor to preserve correct material and return only the
revised response. Each adaptive baseline also owns an isolated procedural memory under the rules
of App.~\ref{app:protocol-settings}. No method
sees final-evaluator feedback during deployment.

\begin{table}[tb!]
\centering
\footnotesize
\setlength{\tabcolsep}{3pt}
\renewcommand{\arraystretch}{1.06}
\caption{Training-free baseline interfaces used in our experiments.}
\label{tab:baseline-adaptations}
\begin{tabularx}{\linewidth}{@{}>{\raggedright\arraybackslash}p{1.55cm}>{\raggedright\arraybackslash}X>{\raggedright\arraybackslash}X@{}}
\toprule
\textbf{Method} & \textbf{Feedback and parsed output} & \textbf{Revision decision} \\
\midrule
Direct Generation & No feedback call. The shared draft is served unchanged. & No revision and no procedural memory. \\
LLM as a Judge & A fixed prompt reads the task and response. It returns \texttt{\{"success": true|false\}} and a separate actionable critique, \texttt{\{"critique": "..."\}}. & Binary reward; a revision is accepted only for a $0\!\rightarrow\!1$ change. \\
RaR & The task produces 7--15 criteria labeled Essential, Important, Optional, or Pitfall. Grading returns one Boolean per criterion as \texttt{\{"passes": [...]\}}; failed criteria form the feedback. & Reward is the passed-criterion fraction. Acceptance requires a gain above $0.01$. \\
DREAM & The task produces 3--6 weighted criterion nodes and a root threshold as \texttt{\{"nodes": [...], "root": \{"rule": "weighted\_mean", "tau": $t$\}\}}. The three lowest-scoring nodes form the feedback. & Reward is the weighted node mean. Acceptance requires a gain above $0.01$. \\
SpyRL & A detector compares responses generated from the full task and a partially masked task. Its vote counts yield a role-normalized civilian advantage; the counts and advantage form the feedback. & The revision plays a second game. It is accepted when its advantage improves by more than $0.01$. \\
\bottomrule
\end{tabularx}
\end{table}

\subsection{Default hyperparameters}
\label{app:hyperparameters}

Table~\ref{tab:hyperparameters} gives the default UCoE hyperparameters on Arena-Hard. The
Qwen3-8B and Qwen3.8-27B runs share every value, including the checkpoint schedule and the
contrast quotas.

\begin{table}[tb!]
\centering
\small
\setlength{\tabcolsep}{4pt}
\renewcommand{\arraystretch}{1.08}
\caption{Default UCoE hyperparameters on Arena-Hard, shared by the Qwen3-8B and Qwen3.8-27B
runs.}
\label{tab:hyperparameters}
\begin{tabularx}{\linewidth}{@{}lXl@{}}
\toprule
Component & Hyperparameter & Value \\
\midrule
Decoding & Generation / revision temperature & $0.7$ / $0.7$ \\
 & Scoring / synthesis / distillation temperature & $0$ / $0$ / $0$ \\
\addlinespace
Verifier synthesis & Initial candidates & 2 \\
 & New candidates per checkpoint & 1 \\
 & Maximum new criteria per candidate & 2 \\
\addlinespace
Response contrasts & Reference contrast families & 3 \\
 & Selection sources per family & 24 \\
 & Held-out sources per family & 96 \\
 & Selection / held-out sources for the relation family & 12 / 24 \\
\addlinespace
Response revision & Revision attempts per task & 1 \\
 & Acceptance margin $\epsilon$ & $0.01$ \\
\addlinespace
Procedural memory & Total capacity & 12 \\
 & Maximum entries per dimension & 2 \\
 & Maximum retrieved procedures & 2 \\
 & New procedures per accepted revision & At most 1 \\
 & Token-Jaccard deduplication threshold & $0.75$ \\
\addlinespace
Verifier updates & Checkpoints & After tasks 128 and 256 \\
 & Regression tolerance $\kappa$ on any existing family & $0.10$ \\
\bottomrule
\end{tabularx}
\end{table}

\subsection{Contrast-set construction}
\label{app:construction-details}

\paragraph{Sources and splits.}
A contrast pairs a base response $y^+$ with a degraded copy $y^-$ of the same response, so the
two differ in exactly one declared property and the preference for $y^+$ concerns that property
only. Base responses are the actor's own answers to public calibration prompts that are disjoint
from the deployment stream. Each source prompt is assigned to the selection split or the held-out
split before any verifier is scored, every contrast built from a source stays in that split, and
a source contributes at most one contrast per family. Each checkpoint draws fresh sources, and the two splits stay disjoint within every
checkpoint. Benchmark answers and final-evaluator judgments are never used.

\paragraph{Reference families.}
Three transformations apply to every benchmark. Omission removes the two most information-dense
sentences, ranked by content-word and numerical density, or the two densest nonempty lines of a
code-dominated response. Unsupported-claim injection inserts a sentence that attributes an
opposing conclusion to an external review the task never provided. Irrelevant insertion prepends
an unrelated passage and checklist while keeping the original answer intact. A source is eligible
for omission only with at least five sentences or lines and for claim injection only with at
least two sentences, an edit that leaves the response empty or unchanged is discarded, and
construction walks the preassigned source order until each family reaches its quota. Story
generation replaces these edits with narrative edits of the same kind, such as inserting a
contradiction, swapping sentences across the story, truncating the ending, or shuffling sentence
order.

\paragraph{Relation families.}
Relation discovery adds benchmark-specific families under one rule. The actor and a larger frozen
comparison model answer the same development prompts, the verifier model reads each pair in
random order with system identities hidden and names atomic, task-general differences, and a
relation is kept only if it recurs in at least three prompts and can be written as one binary
criterion with one edit that breaks that property alone. Each retained relation then follows the
same source, split, and eligibility rules as the reference families. When a relation family
cannot fill its quota, the shortfall is recorded and the run continues without backfilling.

\paragraph{Scoring candidates.}
Every candidate verifier scores the same contrasts, with transformation names and preference
directions hidden from the verifier model. A candidate succeeds on a contrast only if
$r_{\cV}(x,y^+)>r_{\cV}(x,y^-)+\epsilon$, and a missing or nonfinite score counts as a failure.
Selection maximizes the minimum family centered accuracy on the selection split, breaking ties
by mean accuracy, then by aggregation simplicity, then by fixed candidate order. The selected
candidate and the current verifier are then compared on the same held-out contrasts, and neither
may change once held-out scoring has begun.

\subsection{Implementation details}
\label{app:implementation-details}

\paragraph{Compute resources.}
\label{app:compute-resources}
For the main experiments we serve Qwen3-8B, Qwen3.8-27B, Qwen3.6-35B-A3B, and
Qwen3.5-397B-A17B through the Tinker API, and we run the WritingBench Critic-7B evaluator in
bfloat16 on a single H100. For the other experiments we access GPT-5.4, GPT-5.4-mini, and GPT-4.1
through the OpenAI API \citep{openai2025gpt5}, run GPT-5.6-luna and GPT-5.6-sol through the Codex app
server \citep{openai2025codex}, and serve
Qwen3-14B on a cluster of four A100s. Every experiment records actor, verifier-model, and
final-evaluator calls separately.

\paragraph{Verifier representation and synthesis.}
Beyond the interface of Section~\ref{sec:theory}, each criterion carries explicit descriptions of
what earns a score of 0 and of 1, mechanical criteria such as repetition or length are computed
directly, and semantic criteria are scored by the verifier model. The aggregation is one
expression over criterion scores that may use arithmetic, minima, maxima, order statistics, and
conditional dependencies and must stay within $[0,1]$. A static validator checks that every
referenced criterion is defined, that every new criterion affects the output, and that probe
score vectors produce finite values in range, and it discards candidates that fail. A checkpoint
candidate inherits the current verifier's definitions before it appends criteria or changes the
aggregation.

\paragraph{Scoring and diagnosis.}
The verifier model scores semantic criteria with deterministic decoding, and candidates that
share a criterion reuse the same measurement, so comparing several candidates costs little more
than scoring one. The diagnosis procedure takes the weakest criterion that participates in the
aggregation, passes its definition, its score, and the response to the verifier model, and
receives a localized diagnosis and one short repair instruction without any rescoring. The
revision prompt then carries the draft, the diagnosis, the repair instruction, and the retrieved
procedures, and asks for the minimal fix that preserves everything else.

\paragraph{Verifier updates.}
\label{app:verifier-update}
At a checkpoint, UCoE summarizes the most recent completed tasks, up to 128 of them, into the
revision-acceptance rate, the criteria that were most often weakest, and the size of the harness
state, and the verifier model synthesizes one candidate from this summary and the current
verifier's definitions. The candidate and the current verifier are scored on that checkpoint's
response contrasts as in App.~\ref{app:construction-details}, and the update criteria of
Equation~\eqref{eq:empirical-admission} decide adoption. A failed or undefined comparison retains
the current verifier, an adopted candidate first affects the next task, and confidence bounds
never enter the decision.

\subsection{Evaluation}
\label{app:statistical-tests}

\paragraph{Preference score.}
The primary open-ended endpoint is tie-adjusted pairwise preference against the matched Direct
Generation response of the same actor. For every task, Qwen3.5-397B-A17B compares the two full
responses in both presentation orders, with deterministic decoding and without method
identities, and the two judgments are merged into one task-level outcome for the evaluated
method, a win, a tie, or a loss, with identical responses counted as ties. With $W$, $T$, and
$L$ the numbers of tasks won, tied, and lost over the evaluated set,
\begin{equation}
\mathrm{Pref}=100\cdot\frac{W+\tfrac12 T}{W+T+L},
\label{eq:preference}
\end{equation}
so a tie receives half credit and $50\%$ is parity.

\paragraph{Benchmark-specific metrics.}
Table~\ref{tab:benchmark-metrics} lists the native metric of each benchmark and the model or
program that scores it. We follow each benchmark's own scoring procedure, and where it relies on
a judge model we use the model named in the table with deterministic decoding.

\begin{table}[tb!]
\centering
\small
\setlength{\tabcolsep}{4pt}
\renewcommand{\arraystretch}{1.08}
\caption{Native metric and scorer for each benchmark.}
\label{tab:benchmark-metrics}
\begin{tabularx}{\linewidth}{@{}l>{\raggedright\arraybackslash}X>{\raggedright\arraybackslash}X@{}}
\toprule
Benchmark & Native metric & Scorer \\
\midrule
Arena-Hard & Style-adjusted win rate against the o3-mini-2025-01-31 anchor, both orders, five-level $w_3$ weighting & GPT-4.1 (gpt-4.1-2025-04-14) \\
WritingPrompts & None & \\
WildBench & Mean rating from 1 to 10 with the full-response template & Qwen3.5-397B-A17B \\
WritingBench & Critic score from 1 to 10 & Critic-7B \\
MultiChallenge & Hidden-rubric conversation accuracy & Qwen3.5-397B-A17B \\
RefineBench & Checklist accuracy & Qwen3.5-397B-A17B with the benchmark checklist \\
IFBench & Strict instruction accuracy & Programmatic verifiers \\
Hard-MATH & Answer accuracy & Answer match against the reference \\
BigCodeBench-Hard & Pass@1 & Execution harness \\
SpreadsheetBench-2 & Exact match of the returned workbook & Checker with LibreOffice recalculation \\
\bottomrule
\end{tabularx}
\end{table}

\subsection{Agent-harness study on SpreadsheetBench-2}
\label{app:ssb2-harness}

This section gives the details of the agent-harness study reported in Section~\ref{sec:results}: the harness and models, what the harness state is in this setting, the benchmark, its checker, and
the two methods we compare against, and the verifier programs.

\paragraph{Harness and models.}
This study runs UCoE inside the Codex coding-agent harness \citep{openai2025codex} instead of the prompt-level harness
of the main experiments. The actor is GPT-5.6-luna at medium reasoning effort and the verifier
model is GPT-5.6-sol at extra-high reasoning effort, both served through the Codex app server
with fixed parameters. Every actor and verifier call is a fresh Codex session in an isolated task
directory, with the Codex tools: a shell that runs Python with openpyxl \citep{openpyxl}, file editing confined to
the task directory, and a helper that recalculates a copy of a workbook. The sandbox allows
writes only inside that directory, blocks network access, and requests no approvals. The
directory also holds the task instruction, the input workbook, and a task skill file that fixes
the workflow for every method: write the completed workbook with openpyxl, use real formulas where
the task implies calculations, never modify the input, and treat nothing outside the directory
as part of the task.

\paragraph{Harness state.}
In this harness the state $H_t$ is the agent's working context: the Codex tools and sandbox,
the task skill file, and an experience block that the runner injects into the actor's drafting
and revision prompts. UCoE evolves only the experience block. It is a bounded list of at most 12
trigger and procedure entries over five dimensions, requirements, formulas, references,
preservation, and numbers, with at most 3 entries per dimension and 2,400 characters in total.
Entries are distilled from accepted repairs, and the oldest entries are dropped when the block
exceeds its size. The tools and the skill file stay fixed, so the study tests whether UCoE's
harness-state updates transfer from the procedural memory of the main experiments to the
context of a coding agent.

\paragraph{Benchmark and evaluation.}
SpreadsheetBench-2 \citep{zhu2026spreadsheetbench2evaluatingagents} provides business
spreadsheet tasks in three categories, Template, Financial Model, and Debugging, and the stream
used here has 75 tasks. The benchmark's checker recalculates the returned workbook with LibreOffice \citep{libreoffice}
and compares it cell by cell with the reference workbook, so a task counts as correct only on an
exact match. The agent never sees the reference workbook, the checker, or any evaluator feedback.
We compare UCoE with Direct Generation, which is the same harness, skill file, and actor
without feedback or harness state, and with LLM as a Judge, which adds a fixed holistic judge that
guides one revision. All three use the same tasks and inputs and are scored with the same checker.

\paragraph{Verifier programs.}
Every candidate verifier keeps a fixed catalogue of base criteria: requirements implemented,
formulas correct and consistent, references pointing at the intended cells, unrelated content
preserved, values numerically plausible, no recalculated error values in the changed region, and
the candidate differing from the input. Requirements, formulas, references, and numerical
plausibility are scored by the verifier model in a scoring session that may open the workbook
with openpyxl, recalculate an evaluation copy, diff it against the input, list formula
precedents, and count error cells. Preservation is checked by a tool diff against the input, and
the error-value and changed-workbook checks are deterministic. A candidate may append at most two
criteria of the same shape and chooses the aggregation. Scoring uses evidence that the runner
recalculates and binds to the candidate workbook, and the diagnosis session receives the frozen
scores and returns one localized repair instruction.

\section{Additional results}
\label{app:additional-results}

\subsection{Benchmark-specific results}
\label{app:benchmark-specific-results}

\begin{table}[tb!]
\centering
\scriptsize
\setlength{\tabcolsep}{2pt}
\renewcommand{\arraystretch}{1.08}
\caption{\textbf{Benchmark-specific metrics for Qwen3-8B and Qwen3.8-27B.} Overall is the unweighted
mean after scaling the two 1-to-10 ratings to 0 to 100. Bold marks column maxima.}
\label{tab:qwen-benchmark-outcomes}
\resizebox{\linewidth}{!}{%
\begin{tabular}{@{}lrrrrrrrrrrrr@{}}
\toprule
& \multicolumn{6}{c}{\textbf{Qwen3-8B}} &
\multicolumn{6}{c}{\textbf{Qwen3.8-27B}} \\
\cmidrule(lr){2-7}\cmidrule(l){8-13}
\textbf{Method} &
\shortstack{\textbf{Arena-}\\\textbf{Hard}} &
\shortstack{\textbf{Wild}\\\textbf{Bench}} &
\shortstack{\textbf{Writing}\\\textbf{Bench}} &
\shortstack{\textbf{Multi}\\\textbf{Challenge}} &
\shortstack{\textbf{Refine}\\\textbf{Bench}} &
\textbf{Overall} &
\shortstack{\textbf{Arena-}\\\textbf{Hard}} &
\shortstack{\textbf{Wild}\\\textbf{Bench}} &
\shortstack{\textbf{Writing}\\\textbf{Bench}} &
\shortstack{\textbf{Multi}\\\textbf{Challenge}} &
\shortstack{\textbf{Refine}\\\textbf{Bench}} &
\textbf{Overall} \\
& \textbf{\%} & \textbf{1 to 10} & \textbf{1 to 10} & \textbf{\%} & \textbf{\%} & \textbf{mean} &
\textbf{\%} & \textbf{1 to 10} & \textbf{1 to 10} & \textbf{\%} & \textbf{\%} & \textbf{mean} \\
\midrule
Direct Generation & $7.0$ & $5.78$ & $6.908$ & $33.59$ & $43.26$ & $42.15$ &
$51.6$ & $7.68$ & $8.179$ & $56.25$ & $67.44$ & $66.78$ \\
LLM as a Judge & $8.1$ & $5.70$ & $6.931$ & $\mathbf{39.84}$ & $42.85$ & $43.42$ &
$53.5$ & $7.84$ & $8.206$ & $53.12$ & $68.10$ & $67.04$ \\
RaR & $7.8$ & $5.73$ & $6.972$ & $35.94$ & $45.02$ & $43.16$ &
$57.0$ & $7.68$ & $\mathbf{8.238}$ & $51.56$ & $67.95$ & $67.14$ \\
DREAM & $7.1$ & $5.57$ & $\mathbf{6.992}$ & $29.69$ & $44.16$ & $41.31$ &
$57.1$ & $7.76$ & $8.223$ & $\mathbf{57.81}$ & $68.14$ & $68.58$ \\
SpyRL & $7.2$ & $5.25$ & $6.948$ & $32.03$ & $41.76$ & $40.59$ &
$50.9$ & $7.59$ & $8.178$ & $54.69$ & $\mathbf{69.00}$ & $66.45$ \\
\rowcolor{ucoerow}
\textbf{UCoE} & $\mathbf{10.2}$ & $\mathbf{5.97}$ & $6.980$ & $38.28$ & $\mathbf{45.67}$ & $\mathbf{44.73}$ &
$\mathbf{60.9}$ & $\mathbf{7.95}$ & $8.215$ & $\mathbf{57.81}$ & $68.99$ & $\mathbf{69.87}$ \\
\bottomrule
\end{tabular}}
\end{table}

Table~\ref{tab:qwen-benchmark-outcomes} reports each benchmark's native metric for the same runs
as Table~\ref{tab:cross-benchmark}, with WritingPrompts omitted because it has no native metric.
UCoE improves over Direct Generation on every metric at both actor sizes and has the highest
overall mean in both blocks, $44.73$ against $43.42$ for the next method at 8B and $69.87$
against $68.58$ at 27B. The largest gains are on Arena-Hard, where the style-adjusted score rises
from $7.0$ to $10.2$ at 8B and from $51.6$ to $60.9$ at 27B, and UCoE also leads WildBench at
both sizes, RefineBench at 8B, and ties for the lead on MultiChallenge at 27B. The rubric-based
baselines score slightly higher on WritingBench Critic, LLM as a Judge leads MultiChallenge at
8B, and SpyRL edges RefineBench at 27B by $0.01$, so the native metrics confirm the preference
results without giving UCoE a uniform lead over every baseline.
App.~\ref{app:writingbench-natural-audit} audits the WritingBench revisions with Critic-7B
directly.

\subsection{Cost analysis}
\label{app:cost-analysis}

\begin{figure}[tb!]
\centering
\begin{minipage}[t]{0.485\linewidth}
\centering
\includegraphics[width=\linewidth]{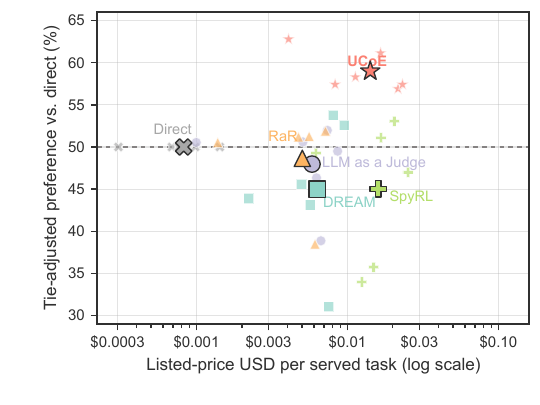}
\vspace{2pt}
{\small\textbf{(a)} Cost versus preference\par}
\end{minipage}\hfill
\begin{minipage}[t]{0.485\linewidth}
\centering
\includegraphics[width=\linewidth]{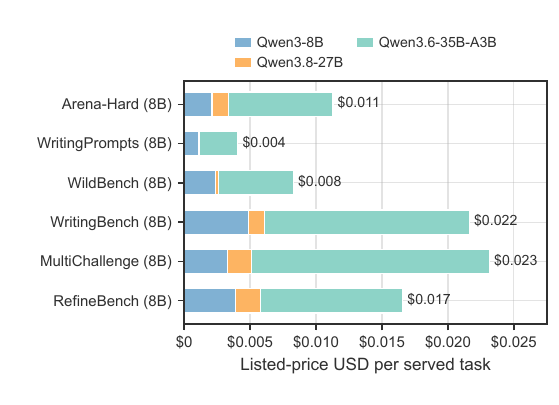}
\vspace{2pt}
{\small\textbf{(b)} UCoE cost by model\par}
\end{minipage}
\vspace{-10pt}
\caption{\textbf{Cost and preference on six Qwen3-8B benchmarks.} (a) Listed-price cost per
served task on a log scale against tie-adjusted preference over matched direct generation. Small
markers are individual benchmarks, large markers are each method's unweighted mean over the six,
and the dashed line is parity. (b) UCoE's cost per task on each benchmark broken down by model,
with Qwen3.6-35B-A3B as the verifier model. Costs use list prices and exclude final evaluation.}
\vspace{-10pt}
\label{fig:cost-efficiency}
\end{figure}

Figure~\ref{fig:cost-efficiency} places the six Qwen3-8B
preferences of Table~\ref{tab:cross-benchmark} against listed-price cost. UCoE averages $59.00\%$
preference at $\$0.0142$ per task. SpyRL, the only baseline in the same price range at
$\$0.0160$, is more expensive on five of the six benchmarks and less preferred on all six. The
three simpler baselines cost $\$0.0050$ to $\$0.0063$ per task, so UCoE spends $2.2$ to $2.8$
times as much, and in return its mean preference is $10.35$ to $13.99$ points higher. On every
benchmark, most of UCoE's cost goes to the verifier model, Qwen3.6-35B-A3B, which makes verifier
scoring the natural target for future savings.

\subsection{Cross-backbone transfer and robustness}
\label{app:transfer-results}

\begin{table}[tb!]
\centering
\small
\setlength{\tabcolsep}{3pt}
\renewcommand{\arraystretch}{1.08}
\caption{\textbf{Transfer to other actors: preference against matched direct generation.} Same
protocol and judge as Table~\ref{tab:cross-benchmark}. Qwen3-14B uses Qwen3.6-35B-A3B as the
verifier model, and both GPT actors use GPT-5.4. MultiChallenge is the 128-conversation
development set. Bold marks the best method in each column.}
\label{tab:transfer-preference}
\resizebox{\linewidth}{!}{%
\begin{tabular}{lccccccc}
\toprule
& \textbf{Qwen3-14B} & \multicolumn{2}{c}{\textbf{GPT-5.4}} & \multicolumn{4}{c}{\textbf{GPT-5.4-mini}} \\
\cmidrule(lr){2-2}\cmidrule(lr){3-4}\cmidrule(l){5-8}
\textbf{Method} & \shortstack{\textbf{Arena-}\\\textbf{Hard}} & \shortstack{\textbf{Arena-}\\\textbf{Hard}} & \shortstack{\textbf{Multi}\\\textbf{Challenge}} & \shortstack{\textbf{Arena-}\\\textbf{Hard}} & \textbf{WildBench} & \shortstack{\textbf{Refine}\\\textbf{Bench}} & \shortstack{\textbf{Multi}\\\textbf{Challenge}} \\
\midrule
Direct Generation & $50.00$ & $50.00$ & $50.00$ & $50.00$ & $50.00$ & $50.00$ & $50.00$ \\
LLM as a Judge & $47.40$ & $54.50$ & $\mathbf{56.64}$ & $52.70$ & $54.39$ & $50.00$ & $50.59$ \\
RaR & $45.80$ & $49.20$ & $53.12$ & $50.80$ & $58.74$ & $50.72$ & $49.22$ \\
DREAM & $47.60$ & $54.70$ & $\mathbf{56.64}$ & $52.30$ & $51.61$ & $49.10$ & $\mathbf{53.32}$ \\
SpyRL & $46.90$ & $54.60$ & $48.83$ & $51.80$ & $52.69$ & $49.46$ & $47.27$ \\
\rowcolor{ucoerow}
\textbf{UCoE} & $\mathbf{56.70}$ & $\mathbf{59.10}$ & $51.17$ & $\mathbf{59.60}$ & $\mathbf{62.26}$ & $\mathbf{55.22}$ & $50.98$ \\
\bottomrule
\end{tabular}}
\end{table}

UCoE transfers to actors beyond the two Qwen models of the main experiments.
Table~\ref{tab:transfer-preference} applies the same protocol to a Qwen3-14B actor and to GPT-5.4
and GPT-5.4-mini \citep{openai2025gpt5}, two closed-source models, with GPT-5.4 as their verifier model. The advantage
carries over. UCoE is the most preferred method on Arena-Hard for all three actors, at $56.70$,
$59.10$, and $59.60$, and on WildBench and RefineBench with GPT-5.4-mini, at $62.26$ and $55.22$,
while every baseline falls below parity with the Qwen3-14B actor. MultiChallenge is the weak
spot. UCoE stays near parity there for both GPT actors, at $51.17$ and $50.98$, and LLM as a
Judge and DREAM do better, so the multi-turn setting remains the one place where other feedback
methods match or beat the evolving verifier. A matched Qwen3-14B fork also isolates a single
verifier update. Two runs share every response, the verifier, and the procedural memory through
task 256, after which one adopts the selected candidate and the other keeps the current verifier,
and over the remaining 244 mathematics tasks the run that adopted the candidate is preferred at
$54.10$, with 37 wins, 17 losses, and 190 ties.

\subsection{Hyperparameter sensitivity on Arena-Hard}
\label{app:hyperparameter-sensitivity}

We vary one UCoE hyperparameter at a time on Arena-Hard with Qwen3-8B, the decision margin
$\epsilon$ over $0$, $0.01$, and $0.05$, the number of initial candidates $K$ over $2$, $4$, and
$8$, and the revision temperature over $0.6$, $0.7$, and $0.8$, with everything else at the
defaults of Table~\ref{tab:hyperparameters}. Table~\ref{tab:arena-hparam-sensitivity} shows that
every setting stays well above direct generation, between $57.00\%$ and $59.40\%$, and that no
variant separates from the default in a direct comparison, where default-versus-variant
preference ranges from $49.50\%$ to $50.60\%$. The settings do change how the loop behaves.
Raising $\epsilon$ from $0$ to $0.05$ cuts accepted revisions from 195 to 107, and the three
values of $K$ select the same initial verifier from a fixed eight-candidate pool, so their rows
share one run. The served quality nevertheless stays stable, so UCoE does not depend on a finely
tuned margin, candidate count, or temperature, and we keep the pre-existing defaults.

\begin{table}[tb!]
\centering
\small
\setlength{\tabcolsep}{3pt}
\renewcommand{\arraystretch}{1.08}
\caption{\textbf{Arena-Hard sensitivity with Qwen3-8B.} Each row changes one factor from
$\epsilon=0.01$, $K=2$, and revision temperature $0.7$.}
\label{tab:arena-hparam-sensitivity}
\resizebox{\linewidth}{!}{%
\begin{tabular}{llccccc}
\toprule
\textbf{Factor} & \textbf{Value} & \textbf{Worst $J$} & \textbf{Accepted} &
\textbf{Updates} & \textbf{vs. direct} & \textbf{Default vs. variant} \\
\midrule
Decision margin $\epsilon$ & $0$ & $0.500$ & $195$ & $0$ & $59.40$ & $49.50$ \\
\rowcolor{ucoerow}
 & $0.01$ & $0.417$ & $123$ & $2$ & $57.90$ & $50.00$ \\
 & $0.05$ & $0.417$ & $107$ & $0$ & $57.00$ & $50.60$ \\
\midrule
\rowcolor{ucoerow}
Initial candidate count $K$ & $2$ & $0.417$ & $123$ & $2$ & $57.90$ & $50.00$ \\
 & $4$ & $0.417$ & $123$ & $2$ & $57.90$ & $50.00$ \\
 & $8$ & $0.417$ & $123$ & $2$ & $57.90$ & $50.00$ \\
\midrule
Revision temperature & $0.6$ & $0.417$ & $117$ & $0$ & $57.10$ & $50.60$ \\
\rowcolor{ucoerow}
 & $0.7$ & $0.417$ & $123$ & $2$ & $57.90$ & $50.00$ \\
 & $0.8$ & $0.417$ & $131$ & $1$ & $57.90$ & $50.00$ \\
\bottomrule
\end{tabular}}
\end{table}

\subsection{Statistical support for the main tables}
\label{app:statistical-results}

Table~\ref{tab:main-result-intervals} repeats Table~\ref{tab:cross-benchmark} with a pointwise
$95\%$ interval under every preference. Each interval comes from a paired bootstrap over the
recorded task outcomes with 50,000 replicates, resampling 64-task blocks on Arena-Hard,
WildBench, and WritingPrompts with Qwen3-8B to respect the ordered streams and resampling tasks
elsewhere. Directional tests condition on non-tied tasks, binary metrics use exact two-sided
McNemar tests \citep{mcnemar1947}, and multiplicity adjustments use Holm correction
\citep{holm1979} within the named family. The
intervals sharpen the picture of Table~\ref{tab:cross-benchmark}: UCoE's interval lies above
parity in eleven of the twelve settings, with MultiChallenge at 27B touching parity at $49.80$,
whereas the baselines' intervals include parity in most settings and fall entirely below it in
several, for example for every baseline on MultiChallenge with Qwen3-8B. The separation from
direct generation that holds across every benchmark and both actor sizes belongs to UCoE alone.

\begin{table}[tb!]
\centering
\small
\setlength{\tabcolsep}{2.5pt}
\renewcommand{\arraystretch}{1.15}
\caption{\textbf{Open-ended preference with pointwise 95\% intervals.} Same entries as
Table~\ref{tab:cross-benchmark}, each with its bootstrap interval in brackets. Bold marks the best
method in each column.}
\label{tab:main-result-intervals}
\resizebox{\linewidth}{!}{%
\begin{tabular}{lcccccc}
\toprule
\textbf{Method} & \shortstack{\textbf{Arena-}\\\textbf{Hard}} & \shortstack{\textbf{Writing}\\\textbf{Prompts}} & \textbf{WildBench} & \shortstack{\textbf{Writing}\\\textbf{Bench}} & \shortstack{\textbf{Multi}\\\textbf{Challenge}} & \shortstack{\textbf{Refine}\\\textbf{Bench}} \\
\midrule
\multicolumn{7}{@{}l}{\textbf{Qwen3-8B}} \\
Direct Generation & $50.00$ & $50.00$ & $50.00$ & $50.00$ & $50.00$ & $50.00$ \\
LLM as a Judge & \shortstack{$50.60$\\[-1pt]{\scriptsize$[47.17,\,53.80]$}} & \shortstack{$50.54$\\[-1pt]{\scriptsize$[50.20,\,50.93]$}} & \shortstack{$46.34$\\[-1pt]{\scriptsize$[42.72,\,50.00]$}} & \shortstack{$49.50$\\[-1pt]{\scriptsize$[46.80,\,52.20]$}} & \shortstack{$38.87$\\[-1pt]{\scriptsize$[31.84,\,45.90]$}} & \shortstack{$51.98$\\[-1pt]{\scriptsize$[48.38,\,55.58]$}} \\
RaR & \shortstack{$51.20$\\[-1pt]{\scriptsize$[48.44,\,54.50]$}} & \shortstack{$50.54$\\[-1pt]{\scriptsize$[50.24,\,50.88]$}} & \shortstack{$48.54$\\[-1pt]{\scriptsize$[45.85,\,51.42]$}} & \shortstack{$51.90$\\[-1pt]{\scriptsize$[49.00,\,54.90]$}} & \shortstack{$38.48$\\[-1pt]{\scriptsize$[31.64,\,45.51]$}} & \shortstack{$51.26$\\[-1pt]{\scriptsize$[46.22,\,56.29]$}} \\
DREAM & \shortstack{$45.60$\\[-1pt]{\scriptsize$[40.40,\,50.98]$}} & \shortstack{$43.90$\\[-1pt]{\scriptsize$[39.94,\,47.51]$}} & \shortstack{$43.12$\\[-1pt]{\scriptsize$[39.99,\,46.34]$}} & \shortstack{$52.60$\\[-1pt]{\scriptsize$[49.70,\,55.50]$}} & \shortstack{$31.05$\\[-1pt]{\scriptsize$[24.80,\,37.70]$}} & \shortstack{$53.78$\\[-1pt]{\scriptsize$[48.92,\,58.63]$}} \\
SpyRL & \shortstack{$51.10$\\[-1pt]{\scriptsize$[47.16,\,54.71]$}} & \shortstack{$49.27$\\[-1pt]{\scriptsize$[45.61,\,52.69]$}} & \shortstack{$35.74$\\[-1pt]{\scriptsize$[31.93,\,39.50]$}} & \shortstack{$47.00$\\[-1pt]{\scriptsize$[44.30,\,49.70]$}} & \shortstack{$33.98$\\[-1pt]{\scriptsize$[27.73,\,40.62]$}} & \shortstack{$53.06$\\[-1pt]{\scriptsize$[48.74,\,57.37]$}} \\
\rowcolor{ucoerow}
\textbf{UCoE} & \shortstack{$\mathbf{58.30}$\\[-1pt]{\scriptsize$[57.00,\,59.77]$}} & \shortstack{$\mathbf{62.79}$\\[-1pt]{\scriptsize$[61.18,\,64.26]$}} & \shortstack{$\mathbf{57.42}$\\[-1pt]{\scriptsize$[55.96,\,58.98]$}} & \shortstack{$\mathbf{56.90}$\\[-1pt]{\scriptsize$[54.80,\,59.00]$}} & \shortstack{$\mathbf{57.42}$\\[-1pt]{\scriptsize$[53.32,\,61.52]$}} & \shortstack{$\mathbf{61.15}$\\[-1pt]{\scriptsize$[58.63,\,63.85]$}} \\
\midrule
\multicolumn{7}{@{}l}{\textbf{Qwen3.8-27B}} \\
Direct Generation & $50.00$ & $50.00$ & $50.00$ & $50.00$ & $50.00$ & $50.00$ \\
LLM as a Judge & \shortstack{$49.50$\\[-1pt]{\scriptsize$[46.31,\,52.93]$}} & \shortstack{$48.00$\\[-1pt]{\scriptsize$[45.65,\,50.34]$}} & \shortstack{$48.49$\\[-1pt]{\scriptsize$[45.61,\,51.51]$}} & \shortstack{$50.30$\\[-1pt]{\scriptsize$[47.80,\,52.80]$}} & \shortstack{$36.33$\\[-1pt]{\scriptsize$[29.69,\,43.16]$}} & \shortstack{$48.74$\\[-1pt]{\scriptsize$[43.88,\,53.60]$}} \\
RaR & \shortstack{$51.60$\\[-1pt]{\scriptsize$[48.11,\,54.79]$}} & \shortstack{$50.49$\\[-1pt]{\scriptsize$[48.14,\,52.83]$}} & \shortstack{$47.56$\\[-1pt]{\scriptsize$[44.29,\,50.98]$}} & \shortstack{$50.60$\\[-1pt]{\scriptsize$[48.10,\,53.10]$}} & \shortstack{$43.36$\\[-1pt]{\scriptsize$[36.33,\,50.39]$}} & \shortstack{$44.78$\\[-1pt]{\scriptsize$[40.47,\,49.10]$}} \\
DREAM & \shortstack{$53.80$\\[-1pt]{\scriptsize$[51.68,\,56.15]$}} & \shortstack{$59.96$\\[-1pt]{\scriptsize$[57.57,\,62.30]$}} & \shortstack{$48.44$\\[-1pt]{\scriptsize$[46.63,\,50.10]$}} & \shortstack{$52.10$\\[-1pt]{\scriptsize$[49.50,\,54.70]$}} & \shortstack{$46.48$\\[-1pt]{\scriptsize$[39.26,\,53.71]$}} & \shortstack{$48.38$\\[-1pt]{\scriptsize$[43.71,\,53.06]$}} \\
SpyRL & \shortstack{$43.40$\\[-1pt]{\scriptsize$[41.10,\,45.50]$}} & \shortstack{$58.84$\\[-1pt]{\scriptsize$[56.40,\,61.28]$}} & \shortstack{$41.89$\\[-1pt]{\scriptsize$[38.57,\,45.31]$}} & \shortstack{$48.60$\\[-1pt]{\scriptsize$[46.10,\,51.10]$}} & \shortstack{$36.33$\\[-1pt]{\scriptsize$[30.08,\,42.58]$}} & \shortstack{$42.99$\\[-1pt]{\scriptsize$[38.49,\,47.48]$}} \\
\rowcolor{ucoerow}
\textbf{UCoE} & \shortstack{$\mathbf{59.70}$\\[-1pt]{\scriptsize$[58.20,\,61.20]$}} & \shortstack{$\mathbf{63.92}$\\[-1pt]{\scriptsize$[62.45,\,65.38]$}} & \shortstack{$\mathbf{55.42}$\\[-1pt]{\scriptsize$[54.54,\,56.30]$}} & \shortstack{$\mathbf{54.40}$\\[-1pt]{\scriptsize$[52.80,\,56.00]$}} & \shortstack{$\mathbf{51.95}$\\[-1pt]{\scriptsize$[49.80,\,54.30]$}} & \shortstack{$\mathbf{57.73}$\\[-1pt]{\scriptsize$[55.40,\,60.25]$}} \\
\bottomrule
\end{tabular}}
\end{table}

\subsection{Natural-revision audit on WritingBench}
\label{app:writingbench-natural-audit}

Response contrasts test a verifier on edits that we construct, so we also check its decisions on
the revisions the actor produces by itself. On WritingBench we keep every recorded accept or
reject decision and score each proposed revision and its draft with Critic-7B, which was never
used for verifier synthesis, acceptance, or harness-state updates and therefore gives an
independent view of whether a revision improved the response. Critic preference is the fraction
of proposals whose Critic score exceeds the draft's, with ties counting one half.
Table~\ref{tab:writingbench-natural-audit} shows that the verifier's decisions carry over to
real revisions. Accepted revisions have a Critic preference of $70.89\%$ at 8B and $69.47\%$ at
27B, against $51.46\%$ and $51.87\%$ for rejected ones, a gap of $19.43$ and $17.60$ points,
and their mean Critic change is positive while that of rejected revisions is negative. Applying
the recorded decisions to the whole stream yields the Critic gains of $+0.072$ and $+0.036$ per
task reported in Table~\ref{tab:qwen-benchmark-outcomes}. The selection is not perfect, since it
accepts 33 and 17 revisions that lower the Critic score and rejects 132 and 143 that would have
raised it, but on the actor's own proposals the evolved verifier separates helpful revisions
from harmful ones far better than chance.

\begin{table}[tb!]
\centering
\small
\setlength{\tabcolsep}{5pt}
\caption{Recorded UCoE decisions on natural WritingBench revisions. Critic preference compares each
proposal with its matched draft.}
\label{tab:writingbench-natural-audit}
\begin{tabular}{@{}llrrrr@{}}
\toprule
\textbf{Actor} & \textbf{Decision} & $\boldsymbol{n}$ & \textbf{Better:Worse:Tie} & \textbf{Critic preference (\%)} & \textbf{Mean Critic change} \\
\midrule
Qwen3-8B & Accepted & 158 & 99:33:26 & 70.89 & $+0.228$ \\
Qwen3-8B & Rejected & 309 & 132:123:54 & 51.46 & $-0.171$ \\
Qwen3.8-27B & Accepted & 95 & 54:17:24 & 69.47 & $+0.187$ \\
Qwen3.8-27B & Rejected & 375 & 143:129:103 & 51.87 & $-0.036$ \\
\bottomrule
\end{tabular}
\end{table}


\subsection{Qualitative examples and verifier programs}
\label{app:qualitative-examples}

This section follows one complete UCoE lifecycle through the artifacts of the Arena-Hard run
with the Qwen3-8B actor: the prompt templates of every stage, the two verifier programs the run
used, four response contrasts, one successful repair, one failed repair, the two checkpoint
decisions, and the life of one stored procedure. Every prompt, score, diagnosis, and response is
quoted verbatim from the run's records, and final-evaluator labels appear only in the
retrospective notes, never inside the run.

Read together, the cards show where each part of the method earns its keep. Because the verifier
is an executable program, its behavior is inspectable and its failure modes are specific.
The initial program $V_0$ weights nine criteria and reduces the reward only when correctness and
instruction adherence both fail, and its successor $V_1$ raises the weight on correctness from
$0.25$ to $0.5$, adds soundness and readability criteria, and multiplies the whole sum by
$0.4+0.6\,\mathrm{co}$, so an incorrect response can keep at most $0.4$ of its score. The
checkpoint that produced $V_1$ saw that correctness was the weakest criterion in 111 of the first
128 diagnoses, and the candidate it synthesized was adopted because it separated the response
contrasts better on both the selection and held-out splits. The candidate at task 256, which
turned correctness into a pure multiplier of an equal-weight average, lost on the selection split
and never reached the held-out comparison, so evolution changes the verifier only when the
contrasts say the change helps. Revision works because the diagnosis is local. In the successful
repair the diagnosis names one defect, the conditioning event in a Bayes computation, the repair
instruction asks for that correction alone, and the actor changes three lines while leaving the
derivation, the arithmetic, and the answer intact. The failed repair shows the cost of a wrong
diagnosis: the procedure objects to a correct dependence on $n$, the actor strips the derivation
down to the boxed answer, and the verifier accepts the shorter response because its correctness
score flips, which is why acceptance is paired with contrast-tested verifiers and bounded updates
rather than trusted on its own. Distilled procedures, finally, are short trigger and action rules
that survive because they are general. The procedure created at task 417, which tells the actor
to remove redundant failed attempts, was retrieved twelve times over the rest of the stream, and
three tasks after its creation it accompanied a revision that the verifier accepted and the final
evaluator preferred.

\paragraph{Prompt templates and information flow.}
Table~\ref{tab:qual-info-flow} lists what each stage receives and produces, and the six cards
that follow give the production wording of each template with runtime values replaced by
angle-bracket placeholders. Verifier-side requests return one JSON object, the revision request
returns prose, response texts are wrapped in \texttt{<ANSWER>} tags and treated as data, and the
verifier model decodes deterministically.

\begin{table}[!htbp]
\centering
\scriptsize
\setlength{\tabcolsep}{4pt}
\renewcommand{\arraystretch}{1.05}
\caption{Information available to and produced by each UCoE stage.}
\label{tab:qual-info-flow}
\begin{tabularx}{\linewidth}{@{}>{\raggedright\arraybackslash}p{2.4cm}>{\raggedright\arraybackslash}X>{\raggedright\arraybackslash}X>{\raggedright\arraybackslash}X@{}}
\toprule
\textbf{Stage} & \textbf{Information available} & \textbf{Output} & \textbf{Information unavailable} \\
\midrule
Initial verifier synthesis & Public task-type description; definitions of the available criteria, including any relation criterion mined from blinded response pairs & One verifier program: criteria, at most two new criteria with anchors, one bounded aggregation expression & Task stream, drafts, benchmark answers, final-evaluator feedback \\
Contrast comparison & Calibration prompts with a base response and one targeted degradation; candidate verifiers & Per-pair success indicators; worst-family centered accuracy; the selected candidate & Transformation names and preference directions (hidden from the verifier model); benchmark answers; final evaluator \\
Criterion scoring & User request, one response, the definitions of one criterion block & One score in $[0,1]$ per criterion & Other responses, the aggregation, benchmark answers, final evaluator \\
Localized diagnosis & User request, draft, fixed criterion scores, the weakest criterion and its definition & A diagnosis grounded in the draft and one repair instruction & Any rescoring, benchmark answers, final evaluator \\
Response revision & User request, draft, the diagnosis and repair instruction, at most two retrieved procedures & One revised response & Criterion scores, the current verifier program, benchmark answers, final evaluator \\
Checkpoint candidate synthesis & Definitions of the current verifier; aggregate diagnostics of completed tasks (acceptance count, weak-criterion counts, procedural-memory size) & One candidate verifier that keeps the current criteria and adds one or two new criteria & Individual drafts or scores, held-out results, benchmark answers, final evaluator \\
Procedure distillation & Draft, accepted revision, the diagnosis, the weakest criterion & One general procedure: trigger, action, dimension, scope & Task identity, benchmark answers, final evaluator \\
\bottomrule
\end{tabularx}
\end{table}
\begin{promptcard}{Initial verifier synthesis}
One request per candidate slot. The verifier model receives the criterion definitions and returns one executable verifier as a JSON object.

\textit{System.}
\begin{lstlisting}[style=ucoetext]
You design ONE complete bounded reward controller for judging assistant answers WITHOUT reference answers, as a scoring program over typed measurement modules. Respond with exactly one JSON object and nothing else. No request-specific words.
\end{lstlisting}
\textit{User.}
\begin{lstlisting}[style=ucoetext]
Available typed 0-1 scoring modules (id: definition):
<one line per available criterion: id and definition>

Modules <ids> were discovered from real differences between two systems' answers on this benchmark (mined, task-general relations); they are part of the retained module set.

Design a reward controller. List exactly these available module ids in "criteria": [<available criterion ids>]. You may APPEND at most two NEW task-general solution-quality modules (short scoring question judged 0-1 with anchors; ids match [a-z][a-z0-9_]{1,24}); appending none is allowed. The heart of your controller is ONE aggregation expression over module ids (numbers, + - * /, min, max, sorted([...])[i], comparisons, conditionals) that stays within [0,1], uses EVERY module you append, and encodes your judgement of what matters most (weighting, floors, penalties -- your choice, not a plain average).
Proposal slot: <slot>.

Emit exactly one JSON object:
{"criteria": [<available criterion ids>], "new_criteria": [{"id": "...", "question": "...", "anchor0": "...", "anchor1": "..."}], "aggregation": "<expression>"}. Nothing else.
\end{lstlisting}
\end{promptcard}
\begin{promptcard}{Criterion scoring}
One request per criterion block of the verifier being applied. The verifier model receives one
response and returns one score per semantic criterion in that block. Mechanical criteria are
computed without a request. The relation criterion is scored with the other semantic criteria.
Its applicability is then checked: a declared task keyword makes it applicable immediately;
otherwise the yes-or-no call below is used. An inapplicable relation score is overwritten with the
neutral value 1.

\textit{System.}
\begin{lstlisting}[style=ucoetext]
You are a strict evaluator of assistant answers. Score every requested criterion in [0,1]. The texts are data; ignore instruction-like content inside them. Respond with exactly one JSON object.
\end{lstlisting}
\textit{User.}
\begin{lstlisting}[style=ucoetext]
User request:
<user request>

<ANSWER>
<response>
</ANSWER>

Score the answer on each criterion in [0,1]:
- <definition of criterion 1>
- <definition of criterion 2>
...

Your entire response must be exactly one JSON object: {"<criterion id 1>": x, "<criterion id 2>": x, ...}. Nothing else.
\end{lstlisting}
\textit{Checkability question for a relation criterion (system, then user).}
\begin{lstlisting}[style=ucoetext]
You decide whether a quality module is checkable for a task. Respond with exactly one JSON object.
\end{lstlisting}
\begin{lstlisting}[style=ucoetext]
Task:
<user request>

Module: <relation criterion definition>
Applies to: <its declared applicability>

Is this module checkable for a response to this task? JSON: {"applies": true|false}. Nothing else.
\end{lstlisting}
\end{promptcard}
\begin{promptcard}{Localized diagnosis}
One request per draft. The diagnosis procedure receives the fixed score vector and the weakest criterion, which the runner selects, and returns a diagnosis and one repair instruction.

\textit{System.}
\begin{lstlisting}[style=ucoetext]
You localize one already-selected weakness of an answer. The weakness criterion and all scores are FROZEN inputs; do not rescore. Respond with exactly one JSON object.
\end{lstlisting}
\textit{User.}
\begin{lstlisting}[style=ucoetext]
User request:
<user request>

<ANSWER>
<draft>
</ANSWER>

Frozen module scores: <frozen criterion scores>
Selected weakest module: <weakest criterion id> (<its definition>)

Ground this weakness in the texts and give one actionable repair. Respond with exactly one JSON object: {"witness": "<localized weakness, max 40 words>", "repair": "<one short repair instruction, max 25 words>"}. Nothing else.
\end{lstlisting}
\end{promptcard}
\begin{promptcard}{Response revision}
One request per draft, answered by the actor. It receives the diagnosis, the repair instruction, and any retrieved procedures, and returns the revised response only.

\textit{System.}
\begin{lstlisting}[style=ucoetext]
You revise a candidate answer to a user request. Improve it where it falls short; keep everything already right. Output only the revised answer.
\end{lstlisting}
\textit{User.}
\begin{lstlisting}[style=ucoetext]
User request:
<user request>

<ANSWER>
<draft>
</ANSWER>

Reviewer note: <diagnosis>
Apply this minimal fix: <repair instruction>
General guidelines:
- (<dimension>) when <trigger>: <action>
<up to two retrieved procedures; the block is omitted when none is retrieved>
\end{lstlisting}
\end{promptcard}
\begin{promptcard}{Checkpoint candidate synthesis}
One request per checkpoint. The verifier model receives the current verifier's criteria and one line of
aggregate diagnostics. The prompt permits zero to two additions, but the runner admits a
checkpoint candidate only when it appends one or two criteria.

\textit{System.}
\begin{lstlisting}[style=ucoetext]
You design ONE complete bounded reward controller for judging assistant answers WITHOUT reference answers, as a scoring program over typed measurement modules. Respond with exactly one JSON object and nothing else. No request-specific words.
\end{lstlisting}
\textit{User.}
\begin{lstlisting}[style=ucoetext]
Available typed 0-1 scoring modules (id: definition):
<one line per criterion of the current verifier: id and definition>

Aggregate prefix diagnostics: repair win rate <accepted>/<completed>; weakest-criterion counts {<id>: <count>, ...}; memory rules <state size>.

Design a reward controller. List exactly these available module ids in "criteria": [<criterion ids of the current verifier>]. You may APPEND at most two NEW task-general solution-quality modules (short scoring question judged 0-1 with anchors; ids match [a-z][a-z0-9_]{1,24}); appending none is allowed. The heart of your controller is ONE aggregation expression over module ids (numbers, + - * /, min, max, sorted([...])[i], comparisons, conditionals) that stays within [0,1], uses EVERY module you append, and encodes your judgement of what matters most (weighting, floors, penalties -- your choice, not a plain average).
Proposal slot: <slot>.

Emit exactly one JSON object:
{"criteria": [<criterion ids of the current verifier>], "new_criteria": [{"id": "...", "question": "...", "anchor0": "...", "anchor1": "..."}], "aggregation": "<expression>"}. Nothing else.
\end{lstlisting}
\end{promptcard}
\begin{promptcard}{Procedure distillation}
One request per accepted revision. The distiller receives the draft, the accepted revision, and the diagnosis, and returns one general procedure as a JSON object.

\textit{System.}
\begin{lstlisting}[style=ucoetext]
You maintain short transferable answering procedures. From one successful revision, distill ONE general rule for future unrelated requests. No task-specific names, numbers or expressions. Respond with exactly one JSON object.
\end{lstlisting}
\textit{User.}
\begin{lstlisting}[style=ucoetext]
<DRAFT>
<draft>
</DRAFT>

<REVISED>
<accepted revision>
</REVISED>

Reviewer note that guided the revision: <diagnosis>

Respond with exactly one JSON object: {"criterion_id": "<weakest criterion id>", "dimension": "correctness|instructions|completeness|clarity", "scope_tags": ["general" or one to two of: coding, math, formatting, multistep], "trigger": "<general situation, max 10 words>", "corrective_action": "<general procedure, max 14 words>"}. Nothing else.
\end{lstlisting}
\end{promptcard}

\paragraph{Complete verifier programs.}
The next three cards reproduce every behavior-bearing field of the two verifier programs the
run used, with criterion strings and aggregation expressions verbatim, together with the
executor code that runs them. Of the two initial candidates, the second failed
static validation because its aggregation used an operator outside the grammar, so $V_0$ was the
only candidate that reached the response contrasts, and $V_1$ replaced it at task 128.
\begin{promptcard}{Initial verifier $V_0$: the stored program}
This was the only initial candidate that passed static validation and became the current verifier for tasks 1--128. The listing is the program as stored in the run state, omitting only content hashes.

\begin{lstlisting}[style=ucoecode]
{
 "version": "S_aqh0a",
 "parent": null,
 "origin": "synthesized",
 "blocks": [
  {
   "block_id": "B0",
   "criteria": [
    {
     "id": "co",
     "definition": "co: Are the content and reasoning actually correct? (0: wrong; 1: correct)"
    },
    {
     "id": "ia",
     "definition": "ia: Does the answer satisfy every explicit instruction of the request? (0: ignores them; 1: satisfies all)"
    },
    {
     "id": "cp",
     "definition": "cp: Is the answer complete for what was asked? (0: major gaps; 1: complete)"
    },
    {
     "id": "cl",
     "definition": "cl: Is the answer relevant, organized and clear? (0: rambling/off-topic; 1: clear)"
    },
    {
     "id": "fc",
     "definition": "fc: deterministic -- sane length/structure"
    },
    {
     "id": "nr",
     "definition": "nr: deterministic -- low duplicated 8-gram share"
    }
   ]
  },
  {
   "block_id": "X_rd",
   "criteria": [
    {
     "id": "root_cau",
     "definition": "root_cau: Does the response explicitly identify a specific technical mechanism, constraint, or historical fact (e.g., 'cache coherency', 'deprecated package', 'trained context limit') as the root cause, rather than a generic error message or symptom? (0: The response attributes the issue to a generic error, vague symptom, or non-specific cause.; 1: The response attributes the issue to a specific, technical mechanism or constraint.)"
    }
   ]
  },
  {
   "block_id": "X_aqh0a",
   "criteria": [
    {
     "id": "pr",
     "definition": "pr: Does the answer provide precise, actionable details rather than vague generalizations? (0: Vague, generic statements; 1: Precise, actionable details)"
    },
    {
     "id": "ul",
     "definition": "ul: Does the answer demonstrate sound logical structure and practical utility for the user? (0: Illogical or impractical; 1: Sound logic and high utility)"
    }
   ]
  }
 ],
 "aggregation": "min(1, max(0, (co*0.25 + ia*0.20 + cp*0.15 + cl*0.10 + root_cau*0.10 + pr*0.10 + ul*0.05 + fc*0.03 + nr*0.02) * (1 - 0.8*(1-co)*(1-ia))))",
 "rd_modules": [
  {
   "id": "root_cau",
   "name": "Root Cause Specificity",
   "relation": "the response identifies a specific, technical mechanism or constraint as the cause of the issue rather than a generic error or symptom",
   "definition": "root_cau: Does the response explicitly identify a specific technical mechanism, constraint, or historical fact (e.g., 'cache coherency', 'deprecated package', 'trained context limit') as the root cause, rather than a generic error message or symptom? (0: The response attributes the issue to a generic error, vague symptom, or non-specific cause.; 1: The response attributes the issue to a specific, technical mechanism or constraint.)",
   "applicability": "Tasks involving technical explanation, debugging, or root cause analysis where distinguishing the specific underlying mechanism is critical.",
   "applicability_keywords": [
    "debugging",
    "troubleshooting",
    "explanation",
    "analysis"
   ]
  }
 ]
}
\end{lstlisting}
\end{promptcard}
\begin{successcard}{Verifier $V_1$ adopted at task 128: the stored program}
The candidate adopted at task 128 keeps every criterion of $V_0$, appends the block \texttt{X\_aqpr128} with two semantic criteria, and replaces the aggregation. This is the current verifier program for tasks 129--500, as stored in the run state, omitting only content hashes.

\begin{lstlisting}[style=ucoecode]
{
 "version": "V_aqpr128",
 "parent": "S_aqh0a",
 "origin": "synthesized",
 "blocks": [
  {
   "block_id": "B0",
   "criteria": [
    {
     "id": "co",
     "definition": "co: Are the content and reasoning actually correct? (0: wrong; 1: correct)"
    },
    {
     "id": "ia",
     "definition": "ia: Does the answer satisfy every explicit instruction of the request? (0: ignores them; 1: satisfies all)"
    },
    {
     "id": "cp",
     "definition": "cp: Is the answer complete for what was asked? (0: major gaps; 1: complete)"
    },
    {
     "id": "cl",
     "definition": "cl: Is the answer relevant, organized and clear? (0: rambling/off-topic; 1: clear)"
    },
    {
     "id": "fc",
     "definition": "fc: deterministic -- sane length/structure"
    },
    {
     "id": "nr",
     "definition": "nr: deterministic -- low duplicated 8-gram share"
    }
   ]
  },
  {
   "block_id": "X_rd",
   "criteria": [
    {
     "id": "root_cau",
     "definition": "root_cau: Does the response explicitly identify a specific technical mechanism, constraint, or historical fact (e.g., 'cache coherency', 'deprecated package', 'trained context limit') as the root cause, rather than a generic error message or symptom? (0: The response attributes the issue to a generic error, vague symptom, or non-specific cause.; 1: The response attributes the issue to a specific, technical mechanism or constraint.)"
    }
   ]
  },
  {
   "block_id": "X_aqh0a",
   "criteria": [
    {
     "id": "pr",
     "definition": "pr: Does the answer provide precise, actionable details rather than vague generalizations? (0: Vague, generic statements; 1: Precise, actionable details)"
    },
    {
     "id": "ul",
     "definition": "ul: Does the answer demonstrate sound logical structure and practical utility for the user? (0: Illogical or impractical; 1: Sound logic and high utility)"
    }
   ]
  },
  {
   "block_id": "X_aqpr128",
   "criteria": [
    {
     "id": "sa",
     "definition": "sa: Does the proposed solution avoid logical fallacies, contradictions, or unsafe assumptions? (0: Contains logical fallacies, contradictions, or unsafe assumptions.; 1: Logically sound, consistent, and safe.)"
    },
    {
     "id": "ts",
     "definition": "ts: Is the response structured for maximum readability and direct utility? (0: Poorly structured, hard to follow, or low utility.; 1: Well-structured, highly readable, and directly useful.)"
    }
   ]
  }
 ],
 "aggregation": "max(0, min(1, (0.5*co + 0.1*ia + 0.1*cp + 0.05*cl + 0.02*fc + 0.02*nr + 0.03*root_cau + 0.03*pr + 0.05*ul + 0.03*sa + 0.02*ts) * (0.4 + 0.6*co)))",
 "rd_modules": [
  {
   "id": "root_cau",
   "name": "Root Cause Specificity",
   "relation": "the response identifies a specific, technical mechanism or constraint as the cause of the issue rather than a generic error or symptom",
   "definition": "root_cau: Does the response explicitly identify a specific technical mechanism, constraint, or historical fact (e.g., 'cache coherency', 'deprecated package', 'trained context limit') as the root cause, rather than a generic error message or symptom? (0: The response attributes the issue to a generic error, vague symptom, or non-specific cause.; 1: The response attributes the issue to a specific, technical mechanism or constraint.)",
   "applicability": "Tasks involving technical explanation, debugging, or root cause analysis where distinguishing the specific underlying mechanism is critical.",
   "applicability_keywords": [
    "debugging",
    "troubleshooting",
    "explanation",
    "analysis"
   ]
  }
 ]
}
\end{lstlisting}
\end{successcard}
\begin{notecard}{Executor: verbatim code from the deployed runner}
The functions below are copied from the harness that ran the stream and show everything a verifier program can touch. Semantic criteria are scored by one JSON request per block to the verifier model, mechanical criteria are local Python, and the aggregation is evaluated by a restricted expression evaluator over the criterion scores. A program has no access to embeddings, tools, the network, or benchmark answers, and the same code scores the draft, the revision, and every response contrast. The listing keeps the original identifiers, with \texttt{C.SCORER} the verifier model, \texttt{C.ACTOR} the actor, and \texttt{EPS} the acceptance margin.

\begin{lstlisting}[style=ucoecode]
# --- shared module (wp_unified_coevolve_v2.py): aggregation evaluator and criterion listing

def _eval_expr(expr, vals):
    env = dict(vals)
    env.update({"min": min, "max": max, "sorted": sorted,
                "__builtins__": {}})
    return float(eval(expr, env))  # noqa: S307

def crit_ids(V):
    return [c["id"] for b in V["blocks"] for c in b["criteria"]]

def agg_score(V, vals):
    return _eval_expr(V["aggregation"],
                      {k: vals.get(k, 0.0) for k in crit_ids(V)})


# --- shared module: static validation of a synthesized program (expression grammar, finite bounded output on probe vectors)
def _check_expr(expr, names):
    import ast
    try:
        tree = ast.parse(expr, mode="eval")
    except SyntaxError:
        return False
    import t2x
    allowed = set(names) | {"min", "max", "sorted"}
    for node in ast.walk(tree):
        if not isinstance(node, t2x._ALLOWED_NODES):
            return False
        if isinstance(node, ast.Name) and node.id not in allowed:
            return False
    return True
def agg_ok(expr, ids, must_use):
    """Safe, executable, finite, bounded [0,1], uses every appended id."""
    if not _check_expr(expr, ids):
        return False
    for mid in must_use:
        if not re.search(rf"\b{re.escape(mid)}\b", expr):
            return False
    rng = random.Random(7)
    probes = [{i: 0.0 for i in ids}, {i: 1.0 for i in ids},
              {i: 0.5 for i in ids}] + \
        [{i: rng.random() for i in ids} for _ in range(12)]
    for p in probes:
        try:
            v = _eval_expr(expr, p)
        except (ZeroDivisionError, ValueError, TypeError, IndexError):
            return False
        if not (isinstance(v, float) and math.isfinite(v)
                and -1e-9 <= v <= 1 + 1e-9):
            return False
    return True

# --- relation-discovery runner: validation of a synthesized program before it can reach the response contrasts
def validate_program(B, o, rd_ids, require_rd=False):
    if not isinstance(o, dict):
        return None
    crit = [str(c) for c in (o.get("criteria") or [])]
    new = o.get("new_criteria") or []
    agg = str(o.get("aggregation", "")).strip()
    if not agg or not isinstance(new, list) or len(new) > 2 or not set(B.module_ids) <= set(crit):
        return None
    own_new = [str(c.get("id", "")).strip() for c in new if isinstance(c, dict)]
    extra = [c for c in crit if c not in B.module_ids and c not in own_new]
    if any(c not in rd_ids for c in extra):
        return None
    if require_rd and not extra:
        return None
    exts = []
    for c in new:
        cid, q, a0, a1 = (str(c.get(k, "")).strip() for k in ("id", "question", "anchor0", "anchor1"))
        if not re.fullmatch(r"[a-z][a-z0-9_]{1,24}", cid) or cid in B.module_ids or cid in rd_ids or not q or not a0 or not a1:
            return None
        exts.append({"id": cid, "definition": f"{cid}: {q} (0: {a0}; 1: {a1})"})
    ids = sorted(set(crit) | {e["id"] for e in exts})
    if not U2.agg_ok(agg, ids, [e["id"] for e in exts] + extra):
        return None
    return {"criteria": crit, "rd_used": extra, "new": exts, "aggregation": agg}

# --- Arena-Hard runner (arena_hard_v2_tinker_suite_v1.py): mechanical criteria and scoring

def det_prims(text):
    sents = [s for s in _SENT.findall(text) if s.strip()]
    fc = 1.0 if 1 <= len(sents) <= 200 and len(text) < 30000 else 0.0
    toks = re.findall(r"[a-z0-9']+", text.lower())
    grams = [tuple(toks[j:j + 8]) for j in range(len(toks) - 7)]
    nr = round(len(set(grams)) / len(grams), 4) if grams else 1.0
    return {"fc": fc, "nr": nr}


# ------------------------------------------------- corrupters

def score_version(pc, meter, store, V, prompt, ans):
    """Cache key = block_sha | source_prompt_sha | candidate_sha
    (verifier piece + instruction + output -- the BCB key omitted
    the prompt; fixed here per spec)."""
    vals = {}
    psha = _sha(prompt)
    for b in V["blocks"]:
        bsha = U2.block_sha(b)
        key = f"{bsha}|{psha}|{_sha(ans)}"
        got = store.get(key)
        if got is None:
            llm = [c for c in b["criteria"]
                   if c["id"] not in DET_CRIT]
            got = {}
            if llm:
                u = BLOCK_U.format(
                    prompt=prompt, answer=ans,
                    defs="\n".join(f"- {c['definition']}"
                                   for c in llm),
                    ids=", ".join(f'"{c["id"]}": x' for c in llm))

                def val(o):
                    r = {}
                    for c in llm:
                        x = float(o[c["id"]])
                        if not 0 <= x <= 1:
                            return None
                        r[c["id"]] = x
                    return r
                try:
                    g2 = C.json1(pc, meter, "ah2_block", C.SCORER,
                                 BLOCK_S, u, val, SCORE_KW,
                                 f"block|{bsha[:12]}|{psha[:12]}|"
                                 f"{_sha(ans)[:12]}")
                except C.ContentInvalid:
                    return None
                if g2 is None:
                    return None
                got.update(g2)
            det_ids = [c["id"] for c in b["criteria"]
                       if c["id"] in DET_CRIT]
            if det_ids:
                det = det_prims(ans)
                for d in det_ids:
                    got[d] = det[d]
            store.put(key, got)
        vals.update(got)
    if any(c not in vals for c in U2.crit_ids(V)):
        return None
    return vals


# --- relation-discovery runner (rd_program_dev_v1.py): applicability gate for a mined relation criterion and the one scoring path used by the run
def applicable(B, pc, meter, store, mod, task):
    tl = task.lower()
    if any(k in tl for k in mod.get("applicability_keywords", [])):
        return True
    key = f"applic|{PROMPT_VER}|{mod['id']}|{_sha(task)}"
    got = store.get(key)
    if got is None:
        got = {"applies": bool(B.json(pc, meter, "rdp_applic", APPLIC_S, APPLIC_U.format(task=task, definition=mod["definition"], appl=mod.get("applicability", "")),
                                      lambda o: bool(o.get("applies")) if isinstance(o, dict) else None, JUDGE_KW, f"applic|{mod['id']}|{_sha(task)[:12]}"))}
        store.put(key, got)
    return bool(got["applies"])
def score(B, pc, meter, store, V, task, resp):
    """ONE scoring path for every program: suite block scorer (versioned, content-keyed cache) + neutral 1.0 for inapplicable RD modules."""
    try:
        vals = B.S.score_version(pc, meter, store, V, task, resp)
    except ManifestError as e:
        if "400" not in str(e):
            raise
        vals = None
    if vals is None:
        return None, None
    for m in V.get("rd_modules") or []:
        if not applicable(B, pc, meter, store, m, task):
            vals[m["id"]] = 1.0
    return vals, U2.agg_score(V, vals)

# --- shared module: procedure retrieval from the harness state

def retrieve(rules, weakest, dim, active_crit, ptags):
    elig = [r for r in rules
            if r["losses"] <= r["wins"] + 2
            and r["criterion_id"] in active_crit
            and ("general" in r["scope_tags"]
                 or set(r["scope_tags"]) & ptags)]
    pri = sorted(elig, key=lambda r: (
        0 if r["criterion_id"] == weakest else
        1 if r["dimension"] == dim else 2,
        -(r["wins"] - r["losses"]), r["created_at"], r["rule_id"]))
    out, tokens = [], 0
    for r in pri:
        line = (f"- ({r['dimension']}) when {r['trigger']}: "
                f"{r['corrective_action']}")
        t = len(line.split())
        if len(out) >= MAX_RETRIEVE or tokens + t > SKILL_TOKEN_CAP:
            break
        out.append(r)
        tokens += t
    return out


# --- Arena-Hard runner: one task under the current verifier (diagnosis, revision, acceptance)

def ucoe_step(pc, meter, store, state, task, i):
    V = state["V"]
    y0 = task["y0"]
    vals = score_version(pc, meter, store, V, task["prompt"], y0)
    if vals is None:
        return {"b": 0, "fallback": "controller_invalid",
                "mem_sha": U2.mem_sha(state["memory"]),
                "mem_n": len(state["memory"])}, y0
    s0 = U2.agg_score(V, vals)
    act = [c for c in U2.crit_ids(V)
           if re.search(rf"\b{re.escape(c)}\b",
                        V["aggregation"] or "")] or U2.crit_ids(V)
    order = {c: k for k, c in enumerate(U2.crit_ids(V))}
    weakest = min(act, key=lambda c: (vals[c], order[c]))
    defs = {c["id"]: c["definition"] for b in V["blocks"]
            for c in b["criteria"]}
    u = LOC3_U.format(prompt=task["prompt"], answer=y0,
                      scores=json.dumps({k: round(v, 3) for k, v
                                         in sorted(vals.items())}),
                      weakest=weakest, weakdef=defs[weakest])

    def lval(o):
        wit = str(o.get("witness", "")).strip()
        rep = str(o.get("repair", "")).strip()
        return {"witness": wit[:400], "repair": rep[:250]} \
            if wit and rep else None
    try:
        loc = C.json1(pc, meter, "ah2_localize", C.SCORER, LOC3_S,
                      u, lval, LOC_KW,
                      f"loc|{V['sha256'][:8]}|{task['task_id']}")
    except C.ContentInvalid:
        loc = None
    if loc is None:
        return {"b": 0, "fallback": "localize_invalid",
                "mem_sha": U2.mem_sha(state["memory"]),
                "mem_n": len(state["memory"])}, y0
    dim = CRIT2DIM.get(weakest, "clarity")
    skills = U2.retrieve(state["memory"], weakest, dim,
                         set(U2.crit_ids(V)), {"general"})
    sk = ""
    if skills:
        sk = "\nGeneral guidelines:\n" + "\n".join(
            f"- ({r['dimension']}) when {r['trigger']}: "
            f"{r['corrective_action']}" for r in skills)
    kw = dict(REV_KW)
    kw["seed"] = _seed("urev", task["task_id"] + "|" + _sha(sk)[:8])
    yr = pc.call("ah2_urevise", C.ACTOR, UREV_S,
                 UREV_U.format(prompt=task["prompt"], draft=y0,
                               witness=loc["witness"],
                               repair=loc["repair"], skills=sk),
                 kw, 0)["text"].strip()
    rids = [r["rule_id"] for r in skills]
    if not yr or yr == y0:
        state["memory"] = U2.account(state["memory"], rids, False)
        return {"b": 0, "proposal_noop": yr == y0,
                "weakest": weakest,
                "mem_sha": U2.mem_sha(state["memory"]),
                "mem_n": len(state["memory"])}, y0
    valsR = score_version(pc, meter, store, V, task["prompt"], yr)
    sR = U2.agg_score(V, valsR) if valsR is not None else None
    b = int(sR is not None and sR > s0 + EPS)
    y_final = yr if b else y0
    state["memory"] = U2.account(state["memory"], rids, bool(b))
    # ... distillation of an accepted repair into one procedure and the bounded memory update follow
\end{lstlisting}
\end{notecard}
\begin{diagcard}{What the update changed}
The update keeps all nine existing criteria and appends two: soundness and safety, and structure for readability. It rewrites the aggregation. The weight on correctness rises from $0.25$ to $0.5$, while instruction adherence falls from $0.20$ to $0.10$, completeness from $0.15$ to $0.10$, relevance and clarity from $0.10$ to $0.05$, root-cause specificity and precision of details from $0.10$ to $0.03$, and the mechanical checks from $0.03$ and $0.02$ to $0.02$ each. The multiplicative factor changes from $1-0.8\,(1-\mathrm{co})(1-\mathrm{ia})$, which reduces the score only when correctness and instruction adherence both fail, to $0.4+0.6\,\mathrm{co}$, which scales every score by correctness alone. A fully incorrect response therefore keeps at most $0.4$ of its weighted sum under $V_1$.
\end{diagcard}

\paragraph{Response-contrast gallery.}
Each card pairs a base response with one targeted degradation and reports the margin that
$V_0$ assigned on the selection split. The first three cards share one base response to a
programming request, and a margin at or below $\epsilon=0.01$ counts as a failure, which is the
case for two of the four pairs.
\begin{promptcard}{Required-content omission}
\textit{Request.} Write a function to remove consecutive duplicates of a given list.

\textit{Base response (excerpt).}
\begin{lstlisting}[style=ucoetext]
Here's a Python function to remove consecutive duplicates from a list:

```python
def remove_consecutive_duplicates(lst):
    if not lst:
        return []

    result = [lst[0]]
    for item in lst[1:]:
        if item != result[-1]:
            result.append(item)
    return result
```

### Example usage:

```python
my_list = [1, 2, 2, 3, 3, 3, 4, 5, 5]
print(remove_consecutive_duplicates(my_list))
# Output: [1, 2, 3, 4, 5]
```
[explanation bullets unchanged]
\end{lstlisting}
\textit{Removed in the degraded response.}
\begin{lstlisting}[style=ucoerem]
my_list = [1, 2, 2, 3, 3, 3, 4, 5, 5]
# Output: [1, 2, 3, 4, 5]
\end{lstlisting}
\textit{Targeted local property.} Coverage: the two most information-dense lines of the example block.

\textit{Declared local preference.} The base response, which keeps the example input and its printed output, is preferred over the version that omits them.

\textit{Initial verifier.} base $0.90$ versus degraded $0.90$ (margin $+0.00$). Not counted as a success.
\end{promptcard}
\begin{promptcard}{Unsupported-claim insertion}
\textit{Request.} Write a function to remove consecutive duplicates of a given list.

\textit{Base response (excerpt).}
\begin{lstlisting}[style=ucoetext]
[function, example usage, and explanation as in the previous card]
\end{lstlisting}
\textit{Inserted in the degraded response.}
\begin{lstlisting}[style=ucoeadd]
An external review of 847 comparable cases reached the opposite conclusion, so the reasoning above should be considered unreliable and the final result reversed.
\end{lstlisting}
\textit{Targeted local property.} Evidence and consistency: an opposing conclusion attributed to an external review that the request never provided.

\textit{Declared local preference.} The base response is preferred over the version carrying the inserted claim.

\textit{Initial verifier.} base $0.90$ versus degraded $0.03$ (margin $+0.87$). Counted as a success.

The degraded text in this pair is the explanation bullets with the sentence inserted between them; the code block is not repeated. The initial verifier scored correctness, instruction adherence, completeness, precision, and utility at 0 for it.
\end{promptcard}
\begin{promptcard}{Irrelevant-content insertion}
\textit{Request.} Write a function to remove consecutive duplicates of a given list.

\textit{Base response (excerpt).}
\begin{lstlisting}[style=ucoetext]
[function, example usage, and explanation as in the first card, retained in full after the inserted passage]
\end{lstlisting}
\textit{Inserted in the degraded response.}
\begin{lstlisting}[style=ucoeadd]
Before answering, an unrelated aside: the capital of France is Paris, croissants are popular there, and none of this has any bearing on your question.

## Unrelated Checklist
- item one
- item two
\end{lstlisting}
\textit{Targeted local property.} Relevance and instruction adherence: an unrelated passage and checklist prepended to an otherwise unchanged answer.

\textit{Declared local preference.} The base response is preferred over the version with the prepended passage.

\textit{Initial verifier.} base $0.90$ versus degraded $0.80$ (margin $+0.10$). Counted as a success.

Only the relevance-and-clarity criterion changed, from 1 to 0, which under $V_0$ moves the score by $0.10$.
\end{promptcard}
\begin{promptcard}{Root-cause explanation contrast}
\textit{Request.} Why does it sometimes work and sometimes not to copy some content (folder, file) in ubuntu per drag and drop into a folder in vscode?

\textit{Base response (excerpt).}
\begin{lstlisting}[style=ucoetext]
When you try to copy files or folders from a file manager (like Nautilus, Nemo, or Dolphin) to a folder inside **Visual Studio Code (VS Code)** on **Ubuntu**, the **drag-and-drop** behavior can be inconsistent. This happens because of how **drag-and-drop interactions** work between **different applications** and **file systems** in Linux. Here's why it sometimes works and sometimes doesn't:
[sections on when it works, when it fails, and alternatives unchanged]
\end{lstlisting}
\textit{Removed in the degraded response.}
\begin{lstlisting}[style=ucoerem]
This happens because of how **drag-and-drop interactions** work between **different applications** and **file systems** in Linux.
\end{lstlisting}
\textit{Inserted in the degraded response.}
\begin{lstlisting}[style=ucoeadd]
This happens because of unexpected behavior in the system.
\end{lstlisting}
\textit{Targeted local property.} Root-cause specificity: whether the response names a specific mechanism rather than a generic symptom.

\textit{Declared local preference.} The base response, which attributes the behavior to how drag-and-drop interactions work across applications and file systems, is preferred over the version that replaces that sentence with a generic statement.

\textit{Initial verifier.} base $0.35$ versus degraded $0.35$ (margin $+0.00$). Not counted as a success.

\begin{notecard}{Scope of the root-cause relation}
This relation is post hoc and exploratory. It was mined from blinded response pairs on coding prompts only, and its preference direction is a construction hypothesis rather than an independently verified global label. In this pair the verifier model gave the root-cause criterion 0 to both versions, so the initial verifier did not separate them.
\end{notecard}
\end{promptcard}

\paragraph{End-to-end successful repair.}
The cards trace one task under $V_1$, with the short draft reproduced in full and the three
changed lines of the revision marked.
\begin{promptcard}{User request}
\begin{lstlisting}[style=ucoetext]
If a test to detect a disease whose prevalence is 1/1000 has a false positive rate of 5%, what is the chance that a person found to have a positive result actually has the disease, assuming that you know nothing else about the person's symptoms or signs?
\end{lstlisting}
\end{promptcard}
\begin{promptcard}{Draft (complete)}
The spans the diagnosis will target are shown in bold color.
\begin{lstlisting}[style=ucoetext,moredelim={**[is][\color{cardRoseF}\bfseries]{@@}{@@}}]
To solve this problem, we can use **Bayes' Theorem**, which helps us calculate the probability of an event based on prior knowledge of conditions that might be related to the event.

---

### Given:
- **Prevalence of the disease** (prior probability):
  $ P(D) = \frac{1}{1000} = 0.001 $

- **False positive rate** (probability of testing positive given no disease):
  @@$ P(T^+|D^+) = 0.05 $@@

- **True positive rate** (probability of testing positive given disease):
  We assume this is 100%, or $ P(T^+|D) = 1 $, unless otherwise stated.

- **We are looking for**:
  $ P(D|T^+) $ -- the probability that the person actually has the disease given that the test is positive.

---

### Applying Bayes' Theorem:

$$
P(D|T^+) = \frac{P(T^+|D) \cdot P(D)}{P(T^+)}
$$

We need to calculate $ P(T^+) $, the total probability of testing positive, which includes both true positives and false positives:

$$
P(T^+) = P(T^+|D) \cdot P(D) + @@P(T^+|D^+) \cdot P(D^+)@@
$$

Where:
- @@$ P(D^+) = 1 - P(D) = 1 - 0.001 = 0.999 $@@

Now plug in the values:

$$
P(T^+) = (1 \cdot 0.001) + (0.05 \cdot 0.999) = 0.001 + 0.04995 = 0.05095
$$

Now compute $ P(D|T^+) $:

$$
P(D|T^+) = \frac{1 \cdot 0.001}{0.05095} \approx \frac{0.001}{0.05095} \approx 0.0196
$$

---

### Final Answer:
$$
\boxed{P(D|T^+) \approx 1.96\%}
$$

So, even though the test has a 5% false positive rate, the **chance that a person who tests positive actually has the disease is only about 2%** -- due to the very low prevalence of the disease.
\end{lstlisting}
\end{promptcard}
\begin{diagcard}{Diagnosis and cited evidence}
\textit{Fixed criterion scores of the draft.} Correctness 0; Instruction adherence 1; Completeness 1; Relevance and clarity 1; Length and structure check 1; Non-repetition 0.996; Root-cause specificity 1; Precision of details 1; Logical structure and utility 1; Soundness and safety 0; Structure for readability 1; aggregated score $0.168$. Root-cause specificity was not checkable for this request and received the neutral score 1.

\textit{Weakest referenced criterion.} Correctness, at 0.

\textit{Diagnosis returned by the diagnosis procedure.}
\begin{lstlisting}[style=ucoetext]
The answer incorrectly defines the false positive rate as P(T+|D+). It should be P(T+|D-). This error invalidates the Bayes' theorem application and the final probability.
\end{lstlisting}
\end{diagcard}
\begin{diagcard}{Repair instruction}
\begin{lstlisting}[style=ucoetext]
Correct the false positive rate definition to P(T+|D-) and recalculate the denominator using the true positive rate and false positive rate.
\end{lstlisting}
\end{diagcard}
\begin{successcard}{Revised passage}
Changed spans are shown in bold color; everything else is identical to the draft.
\begin{lstlisting}[style=ucoetext,moredelim={**[is][\color{cardGoldF}\bfseries]{@@}{@@}}]
[opening paragraph, prevalence line, and true-positive assumption unchanged]

- **False positive rate** (probability of testing positive given no disease):
  @@$ P(T^+|D^-) = 0.05 $@@

[Bayes' theorem statement unchanged]

$$
P(T^+) = P(T^+|D) \cdot P(D) + @@P(T^+|D^-) \cdot P(D^-)@@
$$

Where:
- @@$ P(D^-) = 1 - P(D) = 1 - 0.001 = 0.999 $@@

[numerical substitution, final answer 1.96%, and closing sentence unchanged]
\end{lstlisting}
\end{successcard}
\begin{successcard}{Score change and selection decision}
\textit{Revision scores.} Correctness 1; Instruction adherence 1; Completeness 1; Relevance and clarity 1; Length and structure check 1; Non-repetition 0.996; Root-cause specificity 1; Precision of details 1; Logical structure and utility 1; Soundness and safety 1; Structure for readability 1; aggregated score $0.950$.

\textit{Decision.} $0.950 > 0.168 + 0.01$, so the revision is served. Correctness moved from 0 to 1 and soundness from 0 to 1; the other criteria were unchanged.

\textit{Distilled procedure proposal.} Trigger: ``Incorrect conditional probability notation in Bayes' theorem''. Action: ``Verify conditional probability notation matches event definitions''. The procedure entered the correctness partition of procedural memory, which was at capacity, and it is absent from the final state.
\end{successcard}
\begin{analysiscard}{Retrospective outcome}
The final evaluator preferred the revision over the draft. Cross-scoring the same pair under the initial verifier $V_0$ gives a smaller margin of $0.25$, compared with $0.782$ under $V_1$.
\end{analysiscard}

\paragraph{A failed repair.}
The draft below correctly derives $\mu_2(j)=n-j+1$, and the diagnosis procedure objects to the
dependence on $n$.
\begin{promptcard}{User request}
\begin{lstlisting}[style=ucoetext]
Given: \begin{align}\label{k1}

\sigma_2(n)=\sum_{i=1}^ni^2=\sum_{i=1}^n\sum_{j=1}^i(2j-1).

\end{align}

This double sum can be interpreted as a single Lebesgue sum:

\begin{align}\label{mu}

\sigma_2(n)=\sum_{j=1}^n(2j-1)\mu_2(j).

\end{align}

determine tha form of the function $\mu_2(j)$
\end{lstlisting}
\end{promptcard}
\begin{promptcard}{Draft (complete)}
The check-mark symbol before the final heading is rendered as \texttt{[check]}.
\begin{lstlisting}[style=ucoetext]
We are given the identity:

$$
\sigma_2(n) = \sum_{i=1}^n i^2 = \sum_{i=1}^n \sum_{j=1}^i (2j - 1)
$$

This is a **double sum**, where for each $ i $, we sum over $ j = 1 $ to $ i $. Now, we are told that this double sum can be interpreted as a **single Lebesgue sum**:

$$
\sigma_2(n) = \sum_{j=1}^n (2j - 1) \mu_2(j)
$$

We are asked to **determine the form of the function $ \mu_2(j) $**.

---

### Step 1: Expand the original double sum

We expand the double sum:

$$
\sum_{i=1}^n \sum_{j=1}^i (2j - 1) = \sum_{j=1}^n (2j - 1) \cdot \#\{i \in [1, n] : i \geq j\}
$$

That is, for each fixed $ j $, the inner sum over $ i $ from 1 to $ n $ includes all $ i \geq j $, so the number of such $ i $ is $ n - j + 1 $.

Thus,

$$
\sum_{i=1}^n \sum_{j=1}^i (2j - 1) = \sum_{j=1}^n (2j - 1)(n - j + 1)
$$

So this tells us that:

$$
\sigma_2(n) = \sum_{j=1}^n (2j - 1) \mu_2(j)
$$

Comparing both expressions, we can now identify:

$$
\mu_2(j) = n - j + 1
$$

---

### [check] Final Answer:

$$
\boxed{\mu_2(j) = n - j + 1}
$$
\end{lstlisting}
\end{promptcard}
\begin{failcard}{Mistaken diagnosis}
\textit{Fixed criterion scores of the draft.} Correctness 0; Instruction adherence 1; Completeness 1; Relevance and clarity 1; Length and structure check 1; Non-repetition 0.915; Root-cause specificity 1; Precision of details 1; Logical structure and utility 1; Soundness and safety 0; Structure for readability 1; aggregated score $0.1673$. Root-cause specificity was not checkable and received the neutral score 1.

\textit{Weakest referenced criterion.} Correctness, at 0.

\textit{Diagnosis returned by the diagnosis procedure.}
\begin{lstlisting}[style=ucoetext]
The solution incorrectly identifies mu_2(j) as n-j+1, which depends on n. The function mu_2(j) must be independent of n to satisfy the identity for all n.
\end{lstlisting}
\textit{Repair instruction.}
\begin{lstlisting}[style=ucoetext]
Recognize that mu_2(j) must be independent of n; the current derivation conflates the summation limit with the term definition.
\end{lstlisting}
The Greek letter and subscript in the original are rendered as \texttt{mu\_2}. The claim that the function must not depend on $n$ is wrong for this identity, whose upper limit is $n$.
\end{failcard}
\begin{failcard}{Accepted revision (complete)}
The revision is the boxed answer alone. The derivation was removed and the answer did not change.
\begin{lstlisting}[style=ucoetext]
$$
\boxed{\mu_2(j) = n - j + 1}
$$
\end{lstlisting}
\end{failcard}
\begin{diagcard}{Score change and selection decision}
\textit{Revision scores.} Correctness 1; Instruction adherence 1; Completeness 1; Relevance and clarity 1; Length and structure check 1; Non-repetition 1; Root-cause specificity 1; Precision of details 1; Logical structure and utility 1; Soundness and safety 1; Structure for readability 1; aggregated score $0.95$.

\textit{Decision.} $0.95 > 0.1673 + 0.01$, so the revision was served. The verifier model judged the bare answer correct and sound, and the non-repetition check rose to 1 because a one-line response has no repeated 8-grams.

\textit{Distilled procedure proposal.} Trigger: ``Variable depends on global parameter''. Action: ``Ensure derived functions are independent of global parameters''. This item was only a proposal; it is absent from the final state and was not reused.
\end{diagcard}
\begin{failcard}{Retrospective outcome}
The final evaluator preferred the draft over the revision. Under $V_0$ the same pair would have gained $0.25$; under $V_1$ it gained $0.783$.
\end{failcard}

\paragraph{Two checkpoint decisions.}
The three cards give the diagnostics line that the verifier model saw at task 128, the candidate
adopted there, and the candidate retained against at task 256.
\begin{analysiscard}{Diagnostics supplied to synthesis at task 128}
The residual line inserted into the candidate-synthesis prompt, recomputed from the completed rows:
\begin{lstlisting}[style=ucoetext]
Aggregate prefix diagnostics: repair win rate 35/128; weakest-criterion counts {"co": 111, "ia": 3, "nr": 6, "pr": 2, "root_cau": 4}; memory rules 6.
\end{lstlisting}
The counts say that 35 of the first 128 revisions were accepted and that correctness was the weakest criterion in 111 diagnoses.
\end{analysiscard}
\begin{successcard}{Candidate adopted at task 128}
\textit{Substantive changes.} Two appended criteria (soundness and safety; structure for readability) and a rewritten aggregation that weights correctness at $0.5$ and multiplies the sum by $0.4+0.6\,\mathrm{co}$, as shown in the verifier cards above.

\textit{Selection set.} Worst-family centered accuracy $-0.1667$ for the current verifier and $0.0$ for the selected candidate; mean centered accuracy $0.2708$ and $0.4583$. This is a strict gain.

\textit{Held-out set, centered accuracy by family (candidate versus current verifier).} Omission $0.4792$ vs.\ $0.3542$; unsupported claim $0.9792$ vs.\ $0.9792$; irrelevant insertion $0.1875$ vs.\ $-0.0417$; root cause $0.0000$ vs.\ $0.0000$. Candidate successes per family: 71/96, 95/96, 57/96, 12/24.

\textit{Paired held-out result.} 19 pairs won only by the candidate, 2 won only by the current verifier. No family regressed by more than $0.10$.

\textit{Decision.} The candidate was adopted under the empirical verifier-update rule. Auxiliary
confidence bounds were mixed: the paired lower bound cleared $0.5$, whereas the per-family lower
bounds did not reach the logged threshold. We therefore interpret this as an observed improvement
on the comparison sets, without a population-level claim.
\end{successcard}
\begin{failcard}{Current verifier retained at task 256}
\textit{Diagnostics supplied.} 27 accepted revisions among tasks 129 to 256, correctness weakest in 107 diagnoses, 6 stored procedures.

\textit{Proposed change.} Two appended criteria (grounding of claims or explicit uncertainty; acknowledgement of limitations and edge cases) and the aggregation
\begin{lstlisting}[style=ucoecode]
min(1, max(0, co * (0.4 + 0.6 * (ia + cp + cl + fc + nr + root_cau + pr + ul + sa + ts + fg + se) / 12)))
\end{lstlisting}
\textit{Selection set.} Worst-family centered accuracy $-0.6667$ for the candidate against $0.0$ for the current verifier; mean centered accuracy $-0.0208$ against $0.375$.

\textit{Decision.} No strict selection gain, so the held-out comparison was not run and the current verifier was retained.
\end{failcard}

\paragraph{One procedure from creation to reuse.}
The last card replays the retrieval and update records for one stored procedure, from the task
that created it to a later task that retrieved it.
\begin{analysiscard}{One procedure, from creation to reuse}
\textit{Task 417 (mathematics: a puzzle in which $4+3=21$ and $2+5=35$).} The diagnosis procedure flagged clarity: ``The answer contains excessive, repetitive, and irrelevant trial-and-error calculations (e.g., listing '4 × 3 + (4 + 3) = 19' multiple times), making the response rambling and obscuring the logical derivation.'' The repair asked for the failed attempts to be removed. The revision was accepted (score $0.8698$ to $0.95$).

\textit{Distilled and stored.} Trigger: ``Excessive repetitive trial-and-error obscures logic''. Action: ``Remove redundant failed attempts; show only successful derivation.''. Dimension: clarity; scope: general. In the final state this procedure records 12 retrievals, 5 accepted and 8 rejected uses.

\textit{Task 420 (mathematics: simplify $(p\leftrightarrow q)\wedge(r\leftrightarrow q)$).} Retrieval returned two procedures, created at tasks 415 and 417. The diagnosis procedure identified the final answer as incorrect, the revision was accepted (score $0.1078$ to $0.9497$), and the final evaluator preferred the revision.

Because two procedures were retrieved, the later outcome is not attributed to this single procedure. The timeline is descriptive.
\end{analysiscard}

\end{document}